\documentclass{article}

\usepackage{microtype}
\usepackage{graphicx}
\usepackage{booktabs}
\usepackage{enumitem}
\usepackage{algorithm2e}
\setlist[itemize]{noitemsep, nolistsep}
\usepackage[table]{xcolor}
\usepackage{color,soul}
\usepackage{lscape}
\usepackage{caption}
\usepackage{subcaption}
\usepackage{comment}
\usepackage{tabularx}
\usepackage{arydshln}

\usepackage{hyperref}

\usepackage[accepted]{icml2024}

\usepackage{amsmath}
\usepackage{amssymb}
\usepackage{mathtools}
\usepackage{multirow}
\usepackage{amsthm}
\usepackage[capitalize,noabbrev]{cleveref}
\usepackage{mymacros}
\pgfplotsset{compat=1.18}
\usepgfplotslibrary{groupplots}

\icmltitlerunning{Learning Optimal Deterministic Policies with Stochastic Policy Gradients} 

\begin{document}

\setlength{\abovedisplayskip}{3pt}
\setlength{\belowdisplayskip}{3pt}
\setlength{\textfloatsep}{8pt}

\twocolumn[
\icmltitle{Learning Optimal Deterministic Policies with Stochastic Policy Gradients}



\icmlsetsymbol{equal}{*}

\begin{icmlauthorlist}
\icmlauthor{Alessandro Montenegro}{poli}
\icmlauthor{Marco Mussi}{poli}
\icmlauthor{Alberto Maria Metelli}{poli}
\icmlauthor{Matteo Papini}{poli}
\end{icmlauthorlist}

\icmlaffiliation{poli}{Politecnico di Milano, Piazza Leonardo da Vinci 32, 20133, Milan, Italy}

\icmlcorrespondingauthor{Alessandro Montenegro}{alessandro.montenegro@polimi.it}

\icmlkeywords{Policy Gradient, Convergence, Deterministic Policies}

\vskip 0.3in
]

\printAffiliationsAndNotice{}

\thinmuskip=2mu
\medmuskip=2mu
\thickmuskip=2mu

\allowdisplaybreaks[4]

\begin{abstract}
\emph{Policy gradient} (PG) methods are successful approaches to deal with continuous \emph{reinforcement learning} (RL) problems. They learn stochastic parametric (hyper)policies by either exploring in the space of actions or in the space of parameters. Stochastic controllers, however, are often undesirable from a practical perspective because of their lack of robustness, safety, and traceability. In common practice, stochastic (hyper)policies are learned only to deploy their deterministic version. 
In this paper, we make a step towards the theoretical understanding of this practice. 
After introducing a novel framework for modeling this scenario, we study the global convergence to the \emph{best deterministic} policy, under (weak) gradient domination assumptions. Then, we illustrate how to tune the exploration level used for learning to optimize the trade-off between the sample complexity and the performance of the deployed deterministic policy. Finally, we quantitatively compare action-based and parameter-based exploration,
giving a formal guise to intuitive results.
\end{abstract}

\section{Introduction}
\label{sec:intro}
Within \emph{reinforcement learning}~\citep[RL,][]{sutton2018reinforcement} approaches, \emph{policy gradient}~\citep[PG,][]{deisenroth2013survey} algorithms 
have proved very effective in dealing with real-world control problems. 
Their advantages include the applicability to {continuous state and action spaces}~\citep{peters2006policy}, resilience to {sensor and actuator noise}~\citep{gravell2020learning}, robustness to {partial observability}~\citep{azizzadenesheli2018policy}, and the possibility of incorporating {prior knowledge} in the policy design phase~\citep{ghavamzadeh2006bayesian}, improving {explainability}~\citep{likmeta2020combining}. 
PG algorithms search directly in the space of parametric policies for the one that maximizes a performance function.
Nonetheless, as always in RL, the \emph{exploration} problem has to be addressed, and practical methods involve injecting noise in the actions or in the parameters.
This limits the application of PG methods in many real-world scenarios, 
such as autonomous driving, industrial plants, and robotic controllers. 
Indeed, stochastic policies typically do not meet the \emph{reliability}, \emph{safety}, and \emph{traceability} standards of this kind of applications.

The problem of learning \emph{deterministic} policies has been explicitly addressed in the PG literature by~\citet{silver2014deterministic} with their \emph{deterministic policy gradient}, which spawned very successful deep RL algorithms~\citep{lillicrap2016continuous,fujimoto2018addressing}. This approach, however, is affected by several drawbacks, mostly due to its inherent \emph{off-policy} nature.
First, this makes DPG hard to analyze from a theoretical perspective: local convergence guarantees have been established only recently, and only under assumptions that are very demanding for deterministic policies~\citep{xiong2022deterministic}. Furthermore, its practical versions~\citep[DDPG,][]{lillicrap2016continuous} 
are known to be very susceptible to hyperparameter tuning.

We study here a simpler and fairly common approach:
that of learning \emph{stochastic} policies with PG algorithms, then \emph{deploying} the corresponding deterministic version, ``switching off" the noise.\footnote{This can be observed in several libraries~\citep[\eg][]{stable-baselines3} and benchmarks~\citep[e.g.,][]{duan2016benchmarking}.}
 Intuitively, the amount of exploration (\eg the variance of a Gaussian policy) should be selected wisely. Indeed, the smaller the exploration level, the closer the optimized objective is to that of a deterministic policy. At the same time, with a small exploration, learning can severely slow down and get stuck on bad local optima.

Policy gradient methods can be partitioned based on the space on which the exploration is carried out, distinguishing between: \emph{action-based} (AB) and \emph{parameter-based}~\citep[PB,][]{SEHNKE2010551} exploration. The first, of which REINFORCE~\citep{williams1992simple} and GPOMDP~\citep{baxter2001infinite,sutton1999policy} are the progenitor algorithms, performs exploration in the action space, with a stochastic (e.g., Gaussian) \emph{policy}. 
On the other hand, PB exploration, introduced by Parameter-Exploring Policy Gradients~\citep[PGPE,][]{SEHNKE2010551}, implements the exploration at the level of policy parameters by means of a stochastic \emph{hyperpolicy}. The latter performs perturbations of the parameters of a (typically deterministic) action policy. 
Of course, this dualism only considers the simplest form of noise-based, \emph{undirected} exploration. 
Efficient exploration in large-scale MDPs is a very active area of research, with a large gap between theory and practice~\citep{glp2020aaaitutorial}, placing the matter well beyond the scope of this paper.
Also, we consider noise magnitudes that are \emph{fixed} during the learning process, as the common practice of \emph{learning} the exploration parameters themselves breaks all known sample complexity guarantees of vanilla PG (see Appendix~\ref{apx:mapping}). 

To this day, a large effort has been put into providing convergence guarantees and sample complexity analyses for AB exploration algorithms~\citep[e.g.,][]{papini2018stochastic,yuan2022general,fatkhullin2023stochastic}, while the theoretical analysis of PB exploration has been taking a back seat since~\citep{zhao2011analysis}. We are not aware of any \emph{global} convergence results for parameter-based PGs. Furthermore, even for AB exploration, current studies focus on the \emph{convergence to the best stochastic policy}.

\textbf{Original Contributions.}~~In this paper, we make a step towards the theoretical understanding of the practice of \emph{deploying} a deterministic policy learned with PG methods:
\begin{itemize}[noitemsep, leftmargin=*, topsep=-2pt]
    \item We introduce a framework for modeling the practice of \emph{deploying} a deterministic policy, by formalizing the notion of \emph{white noise-based exploration}, allowing for a unified treatment of both AB and PB exploration.
    \item We study the \emph{convergence to the best deterministic} policy for both AB and PB exploration. For this reason, we focus on the \emph{global convergence}, rather than on the first-order stationary point (FOSP) convergence, and we leverage on commonly used \emph{(weak) gradient domination} assumptions.
    \item We quantitatively show how the exploration level (i.e., noise) generates a trade-off between the sample complexity and the performance of the deployed deterministic policy. Then, we illustrate how it can be tuned 
    to optimize such a trade-off, delivering 
    sample complexity guarantees. 
\end{itemize}
In light of these results, we compare the advantages and disadvantages of AB and PB exploration in terms of sample complexity and requested assumptions, giving a formal guise to intuitive results. We also elaborate on how the assumptions used in the convergence analysis can be reconnected to the basic characteristics of the MDP and the policy classes.
We conclude with a numerical validation to empirically illustrate the discussed trade-offs.
The proofs of the results presented in the main paper are reported in Appendix~\ref{apx:proofs}. 

\section{Preliminaries}
\label{sec:preliminaries}
\noindent\textbf{Notation.}~~For a measurable set $\Xs$, we denote with $\Delta({\Xs})$ the set of probability measures over $\Xs$. For $P \in \Delta(\mathcal{X})$, we denote with $p$ its density function.
With a little abuse of notation, we will interchangeably use $x \sim P$ or $x \sim p$ to denote that random variable $x$ is sampled from the $P$. 
For $n\in\mathbb{N}$, we denote by $\dsb{n} \coloneqq \{ 1, \, \ldots , \, n \}$. 

\noindent\textbf{Lipschitz Continuous and Smooth Functions.}~~A function $f: \mathcal{X} \subseteq \mathbb{R}^d \rightarrow \Reals$ is $L$-\emph{Lipschitz continuous} ($L$-LC) if $|f(\mathbf{x}) - f(\mathbf{x}') |\le L \|\mathbf{x}-\mathbf{x}'\|_2$ for every $\mathbf{x},\mathbf{x}' \in \mathcal{X}$.  $f$ is $L_2$-\emph{Lipschitz smooth} ($L_2$-LS) if it is continuously differentiable and its gradient $\nabla_{\mathbf{x}} f$ is $L_2$-LC, i.e.,   $\|\nabla_{\mathbf{x}} f(\mathbf{x}) - \nabla_{\mathbf{x}}  f(\mathbf{x}') \|_2\le L_2 \|\mathbf{x}-\mathbf{x}'\|_2$ for every $\mathbf{x},\mathbf{x}' \in \mathcal{X}$. 

\noindent\textbf{Markov Decision Processes.}~~A Markov Decision Process~\citep[MDP,][]{puterman1990markov} is represented by $\mathcal{M} \coloneqq \left( \mathcal{S}, \mathcal{A}, p, r, \rho_0, \gamma \right)$, where $\mathcal{S} \subseteq \mathbb{R}^{\ds}$ and $\mathcal{A} \subseteq \mathbb{R}^{\da}$ are the measurable state and action spaces, $p: \mathcal{S} \times \mathcal{A} \xrightarrow[]{} \Delta(\mathcal{S})$ is the transition model, where $p(\vs' | \vs, \va)$ specifies the probability density of landing in state $\vs'\in \mathcal{S}$ by playing action $\va\in \mathcal{A}$ in state $\vs\in \mathcal{S}$, $r: \mathcal{S} \times \mathcal{A} \xrightarrow[]{} [-R_{\max}, R_{\max}]$ is the reward function, where $r(\vs, \va)$ specifies the reward the agent gets by playing action $\va$ in state $\vs$, $\rho_0 \in \Delta(\mathcal{S})$ is the initial-state distribution, and $\gamma \in [0, 1]$ is the discount factor.
A trajectory $\tau = \left( \vs_{\tau,0}, \va_{\tau,0},\dots, \vs_{\tau,T-1}, \va_{\tau,T-1} \right)$ of length $T \in \mathbb{N} \cup \{+\infty\}$ is a sequence of $T$ state-action pairs. The \emph{discounted return} of a trajectory $\tau$ is 
$
    R(\tau) \coloneqq \sum_{t=0}^{T-1} \gamma^t r(\vs_{\tau,t}, \va_{\tau,t})
$.

\noindent\textbf{Deterministic Parametric Policies.}~~
We consider a \textit{parametric deterministic policy} $\bm{\mu}_{\vtheta}: \mathcal{S} \rightarrow \mathcal{A}$, where $\vtheta \in \Theta \subseteq \mathbb{R}^{\dt}$ is the parameter vector belonging to the parameter space $\Theta$.
The performance of $\bm{\mu}_{\vtheta}$ is assessed via the \emph{expected return} $\Jd: \Theta \rightarrow \mathbb{R}$, defined as:
\begin{align}
    \Jd(\vtheta) \coloneqq \E_{\tau \sim \pd(\cdot | \vtheta)} \left[R(\tau)\right], 
\end{align}
where $
    \pd(\tau; \vtheta) \coloneqq \rho_0(\vs_{\tau,0}) \prod_{t=0}^{T-1} p(\vs_{\tau, t+1} | \vs_{\tau, t}, \bm{\mu}_{\vtheta}(\vs_{\tau, t}))
$ is the density of trajectory $\tau$ induced by policy $\bm{\mu}_{\vtheta}$.\footnote{For both $\Jd$ (resp. $\Ja$, $\Jp$) and $\pd$ (resp. $p_{\text{A}}$, $p_{\text{P}}$), we use the $\text{D}$ (resp. $\text{A}$, $\text{P}$) subscript to denote that the dependence on $\vtheta$ (resp. $\vrho$) is through a $\text{D}$eterministic policy (resp. $\text{A}$ction-based exploration policy, $\text{P}$arameter-based exploration hyperpolicy).\label{footnotelabel}}
The agent's goal consists of finding an optimal parameter $\vtheta^*_{\text{D}} \in \argmax_{\vtheta \in \Theta} \Jd(\vtheta)$ and we denote $\Jd^* \coloneqq \Jd(\vtheta^*_\text{D})$.

\noindent\textbf{\textcolor{vibrantBlue}{Action-Based (AB) Exploration.}}~~In {AB exploration}, we consider a \textit{parametric stochastic policy} $\pi_{\vrho}: \mathcal{S} \rightarrow \Delta(\mathcal{A})$, where $\vrho \in \mathcal{P}$ is the parameter vector belonging to the parameter space $\mathcal{P} \subseteq \mathbb{R}^{d_{\mathcal{P}}}$. The policy is used to sample actions $\va_t \sim \pi_{\vrho}(\cdot|\vs_t)$ to be played in state $\vs_t$ for \emph{every step} $t$ of interaction. The performance of $\pi_{\vrho}$ is assessed via the \emph{expected return} $\Ja: \mathcal{P} \rightarrow \mathbb{R}$, defined as:
\begin{align}
    \Ja(\vrho) \coloneqq \E_{\tau \sim \pa(\cdot | \vrho)} \left[ R(\tau) \right],\qquad \text{where}
\end{align}
$\pa(\tau; \vrho) \coloneqq \rho_0(\vs_{\tau, 0}) \prod_{t=0}^{T-1} \pi_{\vrho}(\va_{\tau,t} | \vs_{\tau, t}) p(\vs_{\tau, t+1} | \vs_{\tau, t}, \va_{\tau, t})$
is the density of trajectory $\tau$ induced by policy $\pi_{\vrho}$.\footref{footnotelabel} In AB exploration, we aim at learning  $\vrho^*_{\text{A}} \in \argmax_{\vrho \in \mathcal{P}} \Ja(\vrho)$ and we denote $ \Ja^* \coloneqq \Ja( \vrho^*_{\text{A}})$. If $\Ja(\vrho)$ is differentiable w.r.t.~$\vrho$, PG methods~\citep{peters2008reinforcement} update the parameter $\vrho$ via gradient ascent: $\vrho_{t+1} \xleftarrow{} \vrho_{t} + \zeta_{t} \widehat{\nabla}_{\vrho} \Ja(\vrho_{t})$, where $\zeta_{t}  >0$ is the \emph{step size} and $\widehat{\nabla}_{\vrho} \Ja(\vrho)$ is an estimator of $\nabla_{\vrho} \Ja(\vrho)$.
In particular, the GPOMDP \emph{estimator} is:\footnote{We limit our analysis to the GPOMDP estimator~\cite{baxter2001infinite}, neglecting the REINFORCE one~\cite{williams1992simple} since it is known that the latter suffers from larger variance.}
{\thinmuskip=-1mu
\medmuskip=-1mu
\thickmuskip=-1mu
\begin{equation*}\resizebox{\linewidth}{!}{$\displaystyle
    \widehat{\nabla}_{\vrho} \Ja(\vrho) \coloneqq\frac{1}{N} \sum_{i=1}^{N} \sum_{t=0}^{T-1} \left( \sum_{k=0}^{t} \nabla_{\vrho} \log \pi_{\vrho} (\va_{\tau_i, k} | \vs_{\tau_i, k}) \right) \gamma^t r(\vs_{\tau_i, t}, \va_{\tau_i, t}),$}
\end{equation*}}%
where $N$ is the number of independent trajectories $\{\tau_i\}_{i=1}^N$ collected with policy $\pi_{\vrho}$ ($\tau_i\sim p_{\text{A}}(\cdot;\vrho)$), called \emph{batch size}.

\noindent\textbf{\textcolor{vibrantRed}{Parameter-Based (PB) Exploration.}}~~In PB exploration, we use a \textit{parametric stochastic hyperpolicy} $\nu_{\vrho} \subseteq \Delta(\Theta)$, where $\vrho \in \mathbb{R}^{d_{\mathcal{P}}}$ is the parameter vector. The hyperpolicy is used to sample parameters $\vtheta \sim \nu_{\vrho}$ to be plugged in the deterministic policy $\bm{\mu}_{\vtheta}$ at the beginning of \emph{every trajectory}. The performance index of $\nu_{\vrho}$ is $\Jp: \mathbb{R}^{d_{\vrho}} \xrightarrow[]{} \mathbb{R}$, that is the expectation over $\vtheta$ of $\Jd(\vtheta)$ defined as:\footref{footnotelabel}
\begin{align*}
    \Jp(\vrho) \coloneqq \E_{\vtheta \sim \nu_{\vrho}} \left[ \Jd(\vtheta) \right].
\end{align*}
PB exploration aims at learning $\vrho^*_{\text{P}} \in \argmax_{\vrho \in {\mathcal{P}}} \Jp(\vrho)$ and we denote $ \Jp^* \coloneqq \Jp(\vrho^*_{\text{P}})$.
If $\Jd(\vrho)$ is differentiable w.r.t.~$\vrho$, PGPE~\citep{SEHNKE2010551} updates the hyperparameter $\vrho$ via gradient accent: $\vrho_{t+1} \xleftarrow{} \vrho_{t} + \zeta_{t} \widehat{\nabla}_{\vrho} \Jp(\vrho_t)$. In particular, PGPE uses an estimator of $\nabla_{\vrho} \Jp(\vrho)$ defined as:
\begin{align*}
    \widehat{\nabla}_{\vrho} \Jp(\vrho) = \frac{1}{N} \sum_{i=1}^{N} \nabla_{\vrho} \log \nu_{\vrho}(\vtheta_i) R(\tau_{i}),
\end{align*}
where $N$ is the number of independent parameters-trajectories pairs $\{(\vtheta_i,\tau_i)\}_{i=1}^N$,  collected with hyperpolicy $\nu_{\vrho}$ ($\vtheta_i\sim \nu_{\vrho}$ and $\tau_i \sim p_{\text{D}}(\cdot;\vtheta_i)$), called \emph{batch size}.


\section{White-Noise Exploration}\label{sec:whiteNoise}
We formalize a class of stochastic (hyper)policies widely employed in the practice of AB and PB exploration, namely \emph{white noise-based (hyper)policies}. These policies $\pi_{\vtheta}(\cdot|s)$ (resp. hyperpolicies $\nu_{\vtheta}$) are obtained by adding a \emph{white noise} $\bm{\epsilon}$ to the deterministic action $\mathbf{a}=\bm{\mu}_{\vtheta}(\vs)$ (resp. to the parameter $\vtheta$) independent of the state $s$ (resp. parameter $\vtheta$). 

\begin{defi}[White Noise]\label{defi:zmbv}
Let $d \in \Nat$ and $\sigma >0$. A probability distribution $\Phi_{d} \in \Delta(\mathbb{R}^{d})$ is a white-noise if:
    \begin{align}
        & \E_{\bm{\epsilon} \sim \Phi_d}[\bm{\epsilon}] = \mathbf{0}_{d}, \quad \E_{\bm{\epsilon} \sim \Phi_{d}}[\|\bm{\epsilon}\|_2^2] \le d \sigma^2.
    \end{align}
\end{defi}

This definition complies with the zero-mean Gaussian distribution $\bm{\epsilon} \sim \mathcal{N}(\mathbf{0}_d, \bm{\Sigma})$, where $\E_{\bm{\epsilon} \sim \mathcal{N}(\mathbf{0}_d, \bm{\Sigma})}[\|\bm{\epsilon}\|_2^2] = \text{tr}(\bm{\Sigma}) \le d \lambda_{\max}(\bm{\Sigma})$. In particular, for an isotropic Gaussian $\bm{\Sigma} = \sigma^2 \mathbf{I}_d$, we have that $\text{tr}(\bm{\Sigma}) = d\sigma^2$. We now formalize the notion of \emph{white noise-based (hyper)policy}.


\begin{defi}[\textcolor{vibrantBlue}{\textbf{White noise-based policies}}]\label{defi:add}
    Let $\vtheta \in \Theta$ and $\bm{\mu}_{\vtheta}: \mathcal{S} \rightarrow \mathcal{A}$ be a parametric deterministic policy and let $\Phi_{\da}$ be a white noise (Definition~\ref{defi:zmbv}). A \emph{white noise-based policy} $\pi_{\vtheta} : \Ss \rightarrow \Delta(\As)$ is such that, for every state $\vs \in \Ss$, action $\mathbf{a} \sim \pi_{\vtheta}(\cdot|\vs)$ satisfies $\mathbf{a} = \bm{\mu}_{\vtheta}(\vs) + \vepsilon $ where $\vepsilon \sim \Phi_{\da}$ independently at every step.
\end{defi}
This definition considers stochastic policies $\pi_{\vtheta}(\cdot|\vs)$ that are obtained by adding noise $\vepsilon$ fulfilling Definition~\ref{defi:zmbv}, \emph{sampled independently at every step}, to the action $\bm{\mu}_{\vtheta}(\vs)$ prescribed by the deterministic policy (\ie AB exploration), resulting in playing action $\bm{\mu}_{\vtheta}(\vs)+\vepsilon$. An analogous definition can be formulated for hyperpolicies.

\begin{defi}[\textcolor{vibrantRed}{\textbf{White noise-based hyperpolicies}}]\label{defi:addHype}
    Let $\vtheta \in \Theta$ and $\bm{\mu}_{\vtheta}: \mathcal{S} \rightarrow \mathcal{A}$ be a parametric deterministic policy and let $\Phi_{\dt}$ be a white-noise (Definition~\ref{defi:zmbv}). A \emph{white noise-based hyperpolicy} $\nu_{\vtheta} \in \Delta(\Theta)$ is such that, for every parameter $\vtheta \in \Theta$, parameter $\vtheta' \sim \nu_{\vtheta}$ satisfies $\vtheta' = \vtheta + \vepsilon $ where $ \vepsilon \sim \Phi_{\dt}$ independently in every trajectory.
\end{defi}

This definition considers stochastic hyperpolicies $\nu_{\vtheta}$ obtained by adding noise $\vepsilon$ fulfilling Definition~\ref{defi:zmbv}, \emph{sampled independently at the beginning of each trajectory}, to the parameter $\vtheta$ defining the deterministic policy $\bm{\mu}_{\vtheta}$, resulting in playing deterministic policy $\bm{\mu}_{\vtheta+\vepsilon}$ (\ie PB exploration).
Definitions~\ref{defi:add} and~\ref{defi:addHype} allow to represent a class of widely-used (hyper)policies, like Gaussian hyperpolicies and Gaussian policies with state-independent variance. Furthermore, once the parameter $\vtheta$ is learned with either AB or PB exploration, \emph{deploying} the corresponding deterministic policy (\ie \quotes{switching off} the noise) is straightforward.\footnote{For white noise-based (hyper)policies there exists a \emph{one-to-one mapping} between the parameter space of (hyper)policies and that of deterministic policies ($\mathcal{P}=\Theta$). For simplicity, we assume $\Theta=\Reals^{\dt}$ and $\mathcal{A}= \mathbb{R}^{\da}$ (see Appendix~\ref{apx:mapping}).}
Finally, we remark that the noise can exhibit an inner structure, while it is required to be \quotes{white} among different realizations.



\section{Fundamental Assumptions}\label{sec:acc}
In this section, we present the \emph{fundamental} assumptions on the MDP ($p$ and $r$), deterministic policy $\bm{\mu}_{\vtheta}$, and white noise $\Phi$. For the sake of generality, we will consider \emph{abstract} assumptions in the next sections and, then, show their relation to the {fundamental} ones (see Appendix~\ref{app:quick} for details).

\textbf{Assumptions on the MDP.}~~We start with the assumptions on the regularity of the MDP, i.e., on transition model $p$ and reward function $r$, \wrt variations of the played action $\ba$. 

\begin{ass}[%
Lipschitz MDP ($\log p$, $r$) w.r.t. actions%
]\label{ass:lipMdp}
    The $log$ transition model $\log p(\vs'|\vs,\cdot)$ and the reward function $r(\vs,\cdot)$ are $L_{p}$-\emph{LC} and $L_r$-\emph{LC}, respectively, \wrt the action for every $\vs,\vs'\in \Ss$, \ie for every $\ba,\overline{\ba}\in \As$:
    \begin{align}
        & |\log p(\vs'|\vs,\ba)- \log p(\vs'|\vs,\overline{\ba})| \le L_{p}\| \ba-\overline{\ba}\|_2, \\
        & |r(\vs,\ba)- r(\vs,\overline{\ba})| \le L_r \| \ba-\overline{\ba}\|_2.
    \end{align}
\end{ass}

\begin{ass}[%
Smooth MDP ($\log p$, $r$) w.r.t. actions%
]\label{ass:smoothMdp}
    The log transition model $\log p(\vs'|\vs,\cdot)$ and the reward function $r(\vs,\cdot)$ are $L_{2,p}$-\emph{LS} and $L_{2,r}$-\emph{LS}, respectively, \wrt the action for every $\vs,\vs'\in \Ss$, \ie for every $\ba,\overline{\ba}\in \As$:
    \begin{align*}
        & \|\nabla_{\va} \log p(\vs'|\vs,\ba)- \nabla_{\va}\log p(\vs'|\vs,\overline{\ba})\|_2 \le L_{2,p}\| \ba-\overline{\ba}\|_2, \\
        & \|\nabla_{\va}r(\vs,\ba)-\nabla_{\va} r(\vs,\overline{\ba})\|_2 \le L_{2,r} \| \ba-\overline{\ba}\|_2.
    \end{align*}
\end{ass}

Intuitively, these assumptions ensure that when we perform AB and/or PB exploration altering the played action \wrt a deterministic policy, the effect on the environment dynamics and on reward (and on their gradients) is controllable.

\textbf{Assumptions on the deterministic policy.}~~We now move to the assumptions on the regularity of the deterministic policy $\bm{\mu}_{\vtheta}$ \wrt the parameter $\vtheta$.

\begin{ass}[%
Lipschitz deterministic policy $\bmu$ w.r.t. parameters $\vtheta$%
]\label{ass:lipPol}
    The deterministic policy $\bm{\mu}_{\vtheta}(\vs)$ is $L_\mu$-\emph{LC} \wrt parameter for every $\vs \in \Ss$, \ie for every $\vtheta,\overline{\vtheta}\in\Theta$:
    \begin{align}
        \|\bmu(\vs) - \bmu[\overline{\vtheta}](\vs)\|_2 \le L_{{\mu}} \|\vtheta - \overline{\vtheta} \|_2.
    \end{align}
\end{ass}

\begin{ass}[%
Smooth deterministic policy $\bmu$ w.r.t. parameters $\vtheta$%
]\label{ass:smoothPol}
    The deterministic policy $\bm{\mu}_{\vtheta}(\vs)$ is $L_{2,\mu}$-\emph{LS} \wrt parameter for every $\vs \in \Ss$, \ie for every $\vtheta,\overline{\vtheta}\in\Theta$:
    \begin{align}
        \| \nabla_{\vtheta} \bmu(\vs) - \nabla_{\vtheta} 
 \bmu[\overline{\vtheta}](\vs)\|_2 \le L_{2,{\mu}} \|\vtheta - \overline{\vtheta} \|_2.
    \end{align}
\end{ass}

Similarly, these assumptions ensure that if we deploy an altered parameter $\vtheta$, like in PB exploration, the effect on the played action (and on its gradient) is bounded.

Assumptions~\ref{ass:lipMdp} and~\ref{ass:lipPol} are standard in the DPG literature~\citep{silver2014deterministic}. Assumption~\ref{ass:smoothMdp}, instead, can be interpreted as the counterpart of the \emph{Q-function smoothness} used in the DPG analysis~\citep{kumar2020zeroth,xiong2022deterministic}, while Assumption~\ref{ass:smoothPol} has been used to study the convergence of DPG~\cite{xiong2022deterministic}.
Similar conditions to our Assumption~\ref{ass:lipMdp} were adopted by~\citet{pirotta2015policy}, but measuring the continuity of $p$ in the Kantorovich metric, a weaker requirement that, unfortunately, does not come with a corresponding smoothness condition.

\textbf{Assumptions on the (hyper)policies.}~~ We introduce the assumptions on the score functions of the white noise $\Phi$.
\begin{ass}[%
Bounded Scores of $\Phi$%
]\label{ass:magic}
    Let $\Phi \in \Delta(\mathbb{R}^d)$ be a white noise with variance bound $\sigma>0$ (Definition~\ref{defi:zmbv}) and density $\phi$. $\phi$ is differentiable in its argument and there exists a universal constant $c>0$ such that:
    \begin{enumerate}[leftmargin=*, noitemsep, label=($\roman*$), topsep=-2pt]
        \item  $\E_{\vepsilon \sim \Phi} [\| \nabla_{\vepsilon} \log \phi(\vepsilon)\|_2^2] \le c  d\sigma^{-2}$;
        \item  $\E_{\vepsilon \sim \Phi} [\| \nabla_{\vepsilon}^2 \log \phi(\vepsilon)\|_2] \le c  \sigma^{-2}$.
    \end{enumerate}
    \end{ass}
Intuitively, this assumption is equivalent to the more common ones requiring the boundedness of the expected norms of the score function and its gradient~\citep[][see Appendix~\ref{app:ass_impl}]{papini2022smoothing,yuan2022general}. %
Note that a zero-mean Gaussian  $\Phi = \mathcal{N}(\mathbf{0}_d, \bm{\Sigma})$ fulfills Assumption~\ref{ass:magic}. Indeed, one has $\nabla_{\vepsilon} \log \phi(\vepsilon) = \bm{\Sigma}^{-1}\vepsilon$ and $\nabla_{\vepsilon}^2 \log \phi(\vepsilon) = \bm{\Sigma}^{-1}$. Thus, $\E[\| \nabla_{\vepsilon} \log \phi(\vepsilon)\|_2^2]  = \text{tr}(\bm{\Sigma}^{-1}) \le d \lambda_{\min}(\bm\Sigma)^{-1}$ and $\E[\| \nabla_{\vepsilon}^2 \log \phi(\vepsilon)\|_2] = \lambda_{\min}(\bm{\Sigma})^{-1}$. In particular, for an isotropic Gaussian $\bm{\Sigma}=\sigma^2 \mathbf{I}$, we have $\lambda_{\min}(\bm{\Sigma}) = \sigma^2$, fulfilling Assumption~\ref{ass:magic} with $c=1$.



\section{Deploying Deterministic Policies}\label{sec:deploy}
\label{sec:deploying_det_pol}


In this section, we study the performance $\Jd$ of the \emph{deterministic} policy $\bmu$, when the parameter $\vtheta$ is learned via AB or PB white noise-based exploration (Section~\ref{sec:whiteNoise}). We will refer to this scenario as \emph{deploying} the parameters, which reflects the common practice of \quotes{switching off the noise} once the learning process is over. 

\textbf{{\color{\pbcolor}{PB Exploration.}}}~~Let us start with PB exploration by observing that for white noise-based hyperpolicies (Definition~\ref{defi:addHype}), we can express the expected return $\Jp$ as a function of $\Jd$ and of the noise $\vepsilon$ for every $\vtheta \in \Theta$:
\begin{align}
    \Jp(\vtheta) = \E_{\vepsilon \sim \Phi_{\dt}}[\Jd(\vtheta + \vepsilon)].
\end{align}
This illustrates that PB exploration can be obtained by \emph{perturbing the parameter} $\vtheta$ of a deterministic policy $\bm{\mu}_{\vtheta}$ via the noise $\vepsilon \sim \Phi_{\dt}$. To achieve guarantees on the deterministic performance $\Jd$ of a parameter $\vtheta$ learned with PB exploration, we enforce the following regularity condition.

\begin{ass}[Lipschitz  $\Jd$ \wrt $\vtheta$]\label{ass:Jd_lip}
    $\Jd$ is $L_J$-\emph{LC} in the parameter $\vtheta$, \ie for every $\vtheta, {\vtheta'} \in \Theta$:
    \begin{align}
        |\Jd(\vtheta) - \Jd({\vtheta'})| \le L_J \|\vtheta - {\vtheta'} \|_2.
    \end{align}
\end{ass}
When the MDP and the deterministic policy are LC as in Assumptions~\ref{ass:lipMdp} and~\ref{ass:lipPol}, $L_J$ is $O((1-\gamma)^{-2})$ (see Table~\ref{tab:magic} in Appendix~\ref{app:quick} for the full expression). 
This way, we guarantee that the perturbation $\vepsilon$ on the parameter $\vtheta$ determines a variation on function $\Jd$ depending on the magnitude of $\vepsilon$, which allows obtaining the following result.

\begin{restatable}[\textbf{\textcolor{vibrantRed}{Deterministic deployment of parameters learned with PB white-noise exploration}}]{thr}{deployPB}\label{thr:deployPB}
    If the hyperpolicy complies with Definition~\ref{defi:addHype}, under Assumption~\ref{ass:Jd_lip}:
    \begin{enumerate}[leftmargin=*, noitemsep, label=($\roman*$), topsep=-2pt]
        \item \emph{(Uniform bound)} for every $\vtheta \in \Theta$, it holds that $ |\Jd(\vtheta) - \Jp(\vtheta)| \le L_J \sqrt{\dt} \sigma_{\text{P}}$;
        \item \emph{($\Jd$ upper bound)} let $\vtheta^*_{\text{P}} \in \argmax_{\vtheta \in \Theta} \Jp(\vtheta)$, it holds that: $
        \Jd^* - \Jd(\vtheta^*_{\text{P}}) \le 2L_J \sqrt{\dt} \sigma_{\text{P}}$;
    \item \emph{($\Jd$ lower bound)} there exists an MDP, a deterministic policy class $\bmu$ fulfilling Assumption~\ref{ass:Jd_lip}, and a noise complying with Definition~\ref{defi:zmbv}, such that $ \Jd^* - \Jd(\vtheta^*_{\text{P}}) \ge 0.28 L_J \sqrt{\dt} \sigma_{\text{P}}$.
    \end{enumerate}
\end{restatable}


Some observations are in order. ($i$) shows that the performance of the hyperpolicy $\Jp(\vtheta)$ is representative of the deterministic performance $\Jd(\vtheta)$ up to an additive term depending on $ L_J \sqrt{\dt} \sigma_{\text{P}}$.
As expected, this term grows with the Lipschitz constant $L_J$ of the function $\Jd$, with the standard deviation $\sigma_\text{P}$ of the additive noise, and with the dimensionality of the parameter space $\dt$. In particular, this implies that $\lim_{\sigma_{\text{P}} \rightarrow 0^+} \Jp(\vtheta) = \Jd(\vtheta)$. ($ii$) is a consequence of ($i$) and provides an \emph{upper bound} between the optimal performance obtained if we were able to directly optimize the deterministic policy $\max_{\vtheta \in \Theta} \Jd(\vtheta)$ and the performance of the parameter $\vtheta^*_{\text{P}}$ learned by optimizing $\Jp(\vtheta)$, \ie via PB exploration, \emph{when deployed on the deterministic policy}.
Finally, ($iii$) provides a lower bound to the same quantity on a specific instance of MDP and hyperpolicy, proving that the dependence on  $ L_J \sqrt{\dt} \sigma_{\text{P}}$ is \emph{tight} up to constant terms.

\textbf{{\color{\abcolor}{AB Exploration.}}}~~Let us move to the AB exploration case, where understanding the effect of the noise is more complex since it is applied to every action \emph{independently at every step}. To this end, we introduce the notion of \emph{non-stationary} deterministic policy $\underline{\bm{\mu}} = (\bm{\mu}_t)_{t=0}^{T-1}$, where at time step $t$ the deterministic policy $\bm{\mu}_t: \Ss \rightarrow \As$ is played, and its expected return (with  abuse of notation) 
 is $\Jd(\underline{\bm{\mu}}) = \E_{\tau \sim p_{\text{D}}(\cdot|\underline{\bm{\mu}})}[R(\tau)]$ where $ p_{\text{D}}(\cdot|\underline{\bm{\mu}}) \coloneqq \rho_0(\vs_{\tau,0}) \prod_{t=0}^{T-1} p(\vs_{\tau, t+1} | \vs_{\tau, t}, \bm{\mu}_{t}(\vs_{\tau, t}))  $. Let $\underline{\vepsilon} = (\vepsilon_t)_{t=0}^{T-1} \sim \Phi_{\da}^T$  be a sequence of noises sampled independently, we denote with $\underline{\bm{\mu}}_{\vtheta} + \underline{\vepsilon} = (\bmu + \vepsilon_t)_{t=0}^{T-1}$ the non-stationary policy  that, at time $t$, \emph{perturbs the action} as $\bm{\mu}_{\vtheta}(\vs_t) + \vepsilon_t$. 
Since the noise is independent on the state, we express $\Ja$ as a function of $\Jd$ for every $\vtheta \in \Theta$ as follows:
\begin{align}
    \Ja(\vtheta) = \E_{\underline{\vepsilon} \sim \Phi_{\da}^T }\left[ {\Jd}(\underline{\bm{\mu}}_{\vtheta} + \underline{\vepsilon}) \right].
\end{align}
Thus,  to ensure that the parameter learned with AB exploration achieves performance guarantees when evaluated as a deterministic policy, we need to enforce some regularity condition on $\Jd$ as a function of $\underline{\bm{\mu}}$.
\begin{ass}[Lipschitz $\Jd$ \wrt $\underline{\bm{\mu}}$]\label{ass:Jd_lip_ns}
     $\Jd$ of the non-stationary deterministic policy $\underline{\bm{\mu}}$ is $(L_t)_{t=0}^{T-1}$-\emph{LC} in the non-stationary policy, \ie for every $\underline{\bm{\mu}},\underline{\bm{\mu}}'$:
    \begin{align}
        |{\Jd}(\underline{\bm{\mu}}) - {\Jd}(\underline{\bm{\mu}}')| \le \sum_{t =0}^{T-1} L_t \sup_{\vs \in \Ss} \left\| \bm{\mu}_t(\vs) - {\bm{\mu}}_t'(\vs) \right\|_2.
    \end{align}
    Furthermore, we denote $L \coloneqq \sum_{t=0}^{T-1} L_t$.
\end{ass}
When the MDP is LC as in Assumptions~\ref{ass:lipMdp}, $L$ is $O((1-\gamma)^{-2})$ (see Table~\ref{tab:magic} in Appendix~\ref{app:quick} for the full expression). 
The assumption enforces that changing the deterministic policy at step $t$  from $\bm{\mu}_t$ to $\bm{\mu}_t'$, the variation of $\Jd$ is controlled by the action distance (in the worst state $\vs$) multiplied by a \emph{time-dependent} Lipschitz constant. This form of condition allows us to show the following result.

\begin{restatable}[\textbf{\textcolor{vibrantBlue}{Deterministic deployment of parameters learned with AB white-noise exploration}}]{thr}{deployAB}\label{thr:deployAB}
    If the policy complies with Definition~\ref{defi:add} and under Assumption~\ref{ass:Jd_lip_ns}:
    \begin{enumerate}[leftmargin=*, noitemsep, label=($\roman*$), topsep=-2pt]
        \item \emph{(Uniform bound)} for every $\vtheta \in \Theta$, it holds that: $ |\Jd(\vtheta) - \Ja(\vtheta)| \le L \sqrt{\da} \sigma_{\text{A}}$;
        \item \emph{($\Jd$ upper bound)} letting $\vtheta^*_{\text{A}} \in \argmax_{\vtheta \in \Theta} \Ja(\vtheta)$, it holds that $
        \Jd^* - \Jd(\vtheta^*_{\text{A}}) \le 2L \sqrt{\da} \sigma_{\text{A}}$;
    \item \emph{($\Jd$ lower bound)} there exists an MDP, a deterministic policy class $\bmu$ fulfilling Assumption~\ref{ass:Jd_lip}, and a noise complying with Definition~\ref{defi:zmbv}, such that $ \Jd^* - \Jd(\vtheta^*_{\text{A}}) \ge 0.28 L \sqrt{\da} \sigma_{\text{A}}$.
    \end{enumerate}
\end{restatable}

Similarly to Theorem~\ref{thr:deployPB}, ($i$) and ($ii$) provide an upper bound on the difference between the policy performance $\Ja(\vtheta)$ and the corresponding deterministic policy $\Jd(\vtheta)$, and on the performance of $\vtheta^*_{\text{A}}$ when deployed on a deterministic policy. Clearly, also in the AB exploration, we have that $\lim_{\asigma \rightarrow 0^+} \Ja(\vtheta) = \Jd(\vtheta)$.
As in the PB case, ($iii$) shows that the upper bound ($ii$) is tight up to constant terms. 

 Finally, let us note that our bounds for PB exploration depend on the dimension of the parameter space $\dt$ that is replaced by that of the action space $\da$ in AB exploration. 

\section{Global Convergence Analysis}
\label{sec:convergence}
In this section, we present our main results about the convergence of AB and PB \emph{white noise-based exploration} to a \emph{global optimal parameter} $\vtheta^*_{\text{D}}$ for the performance of the deterministic policy $\Jd$.
Let $K \in \mathbb{N}$ be the number of \emph{iterations} and $N$ the \emph{batch size}; 
given an accuracy threshold $\epsilon > 0$, our goal is to bound the \emph{sample complexity} $NK$ to fulfill the following  \emph{last-iterate global} convergence condition:
\begin{align}\label{eq:conver}
   \Jd^* - \E \left[ \Jd(\vtheta_{K}) \right] \leq \epsilon,
\end{align}
where $\vtheta_{K}$ is the (hyper)parameter at the end of learning.
We start in Section~\ref{sec:generalConv}, introducing the \emph{abstract} assumptions and providing a general convergence analysis applicable to both AB and PB exploration for learning the corresponding objective ($\Ja$ or $\Jp$). Then, in Section~\ref{sec:convBoth}, we derive the convergence guarantees on the deterministic objective $\Jd$ for AB and PB exploration, respectively. Our results are first presented for a \emph{fixed} white noise variance $\sigma^2$ to highlight the trade-off between sample complexity and performance, then extended to an \emph{$\epsilon$-adaptive} choice of $\sigma$.

\subsection{General Global Convergence Analysis}\label{sec:generalConv}
In this section, we provide a global convergence analysis for a generic stochastic first-order algorithm optimizing the differentiable objective function $J_\simbolo$ on the parameters space $\Theta \subseteq \mathbb{R}^d$, that can be instanced for both AB (setting $J_\simbolo=\Ja$) and PB (setting $J_\simbolo=\Jp$) exploration, when optimizing the corresponding objective. At every iteration $k \in \dsb{K}$, the algorithm performs the gradient ascent update:
\begin{align}\label{eq:update}
     \vtheta_{k + 1} \xleftarrow{} \vtheta_{k} + \zeta_{k} \widehat{\nabla}_{\vtheta} J_\simbolo(\vtheta_k),
\end{align}
where $\zeta_k>0$ is the \emph{step size} and $\widehat{\nabla}_{\vtheta} J_\simbolo(\vtheta_{k})$ is an \emph{unbiased} estimate of $\nabla_{\vtheta} J_\simbolo(\vtheta_{k})$. We denote $J^*_\simbolo = \max_{\vtheta \in \Theta} J_\simbolo(\vtheta)$ and we enforce the following standard assumptions.

\begin{ass}[%
Weak gradient domination for $J_\simbolo$%
] \label{ass:J_wgd}
    There exist $\alphagen > 0$ and $\betagen \geq 0$ such that  for every $ \vtheta \in \Theta$ it holds that $J^*_\simbolo - J_\simbolo(\vtheta) \leq \alphagen \|\nabla_{\vtheta} J_\simbolo (\vtheta) \|_2 + \betagen$.
\end{ass}

Assumption~\ref{ass:J_wgd} is the gold standard for the global convergence of stochastic optimization~\citep{yuan2022general,masiha2022stochastic,fatkhullin2023stochastic}. Note that, when $\beta = 0$, we recover the (strong) \emph{gradient domination} (GD) property: $J^*_\simbolo - J_\simbolo(\vtheta) \leq \alpha \left\| \nabla_{\vtheta} J(_\simbolo\vtheta) \right\|_2$ for all $\vtheta\in\Theta$.
GD is stricter than WGD and requires that $J_\simbolo$ has no local optima. Instead, WGD admits local maxima as long as their performance is $\beta$-close to the globally optimal one.\footnote{In this section, we will assume that $J_\simbolo$ (i.e., either $\Ja$ or $\Jp$) is already endowed with the WGD property. In Section~\ref{sec:assumptions}, we illustrate how it can be obtained in several common scenarios.} 

\begin{ass}[%
Smooth $J_\simbolo$ w.r.t. parameters $\vtheta$%
]\label{ass:J_gen_conv}
    $J_\simbolo$ is $L_{2,\simbolo}$-\emph{LS} \wrt parameters $\vtheta$, \ie for every $\vtheta,\vtheta' \in \Theta$:
    \begin{align}
        \|\nabla_{\vtheta} J_\simbolo(\vtheta') - \nabla_{\vtheta} J_\simbolo(\vtheta)\|_2 \le L_{2,\simbolo} \| \vtheta' - \vtheta \|_2.
    \end{align}
\end{ass}

Assumption~\ref{ass:J_gen_conv} is ubiquitous in the convergence analysis of policy gradient algorithms~\cite{papini2018stochastic, agarwal2021theory,yuan2022general,bhandari2019global}, which is usually studied as an instance of (non-convex) \emph{smooth} stochastic optimization. 
The smoothness of $J_\simbolo \in \{\Ja,\Jp\}$ can be: ($i$) inherited from the deterministic objective $\Jd$ (originating, in turn, from the regularity of the MDP) and of the deterministic policy $\bm{\mu}_{\vtheta}$ (Asm.~\ref{ass:lipMdp} and~\ref{ass:smoothPol}); or ($ii$) enforced through the properties on the white noise $\Phi$ (Asm.~\ref{ass:magic}). {The first result was observed in a similar form by~\citet[][Theorem 3]{pirotta2015policy}, while a generalization of the second was established by~\citet{papini2022smoothing} and refined by~\citet{yuan2022general}.}

\begin{ass}[%
Bounded estimator variance $\widehat{\nabla}_{\vtheta} J_\simbolo(\vtheta)$%
]\label{ass:J_var_gen_conv}
    The estimator $\widehat{\nabla}_{\vtheta} J_\simbolo(\vtheta)$ computed with batch size $N$ has a bounded variance, i.e., there exists $  V_\simbolo \geq 0$ such that, for every $ \vtheta \in \Theta$, we have $\Var[\widehat{\nabla}_{\vtheta} J_\simbolo(\vtheta)] \leq V_\simbolo/N$.
\end{ass}

Assumption~\ref{ass:J_var_gen_conv} guarantees that the gradient estimator is characterized by a bounded variance $V_\simbolo$ which scales with the batch size $N$. Under Assumption~\ref{ass:magic} (and \ref{ass:lipPol} for GPOMDP), the term $V_\simbolo$ can be further characterized (see Table~\ref{tab:magic} in Appendix~\ref{app:quick}).

We are now ready to state the global convergence result.

\begin{restatable}{thr}{globalConvGeneral}\label{thr:gen_conv_new_main}
    Consider an algorithm running the update rule of Equation~\eqref{eq:update}. Under Assumptions~\ref{ass:J_wgd}, \ref{ass:J_gen_conv}, and~\ref{ass:J_var_gen_conv}, with a suitable constant step size, to guarantee $
         J_\simbolo^* - \E[J_\simbolo(\vtheta_K)] \le \epsilon + \beta
$ the sample complexity is at most:
     \begin{align}
         NK =  \frac{16\alpha^4L_{2,\simbolo}V_\simbolo}{\epsilon^3} \log \frac{\max\{0,J^*_\simbolo - J_\simbolo(\vtheta_0) - \beta\}}{\epsilon}.
 \end{align}
\end{restatable}

This result establishes a convergence of order $\widetilde{O}(\epsilon^{-3})$\footnote{The $\widetilde{\mathcal{O}}(\cdot)$ notation hides logarithmic factors.} to the global optimum $J^*_\simbolo$ of the general objective $J_\simbolo$. Recalling that $J_\simbolo \in \{\Ja,\Jp\}$, Theorem~\ref{thr:gen_conv_new_main} provides: ($i$) the first global convergence guarantee for PGPE for PB exploration (setting $J_\simbolo=\Jp$) and ($ii$) a global convergence guarantee for PG (e.g., GPOMDP) for AB exploration of the same order (up to logarithmic terms in $\epsilon^{-1}$) of the state-of-the-art one of~\citet{yuan2022general} (setting $J_\simbolo=\Ja$). Note that our guarantee is obtained for a \emph{constant} step size and holds for the last parameter $\vtheta_K$, delivering a \emph{last-iterate} result, rather than a \emph{best-iterate} one as in~\citep[][Corollary 3.7]{yuan2022general}. Clearly, this result is not yet our ultimate goal since, we need to assess how far the performance of the learned parameter $\vtheta_K$ is from that of the optimal deterministic objective $\Jd^*$.

\begin{figure*}
    \centering
    \resizebox{\linewidth}{!}{\includegraphics{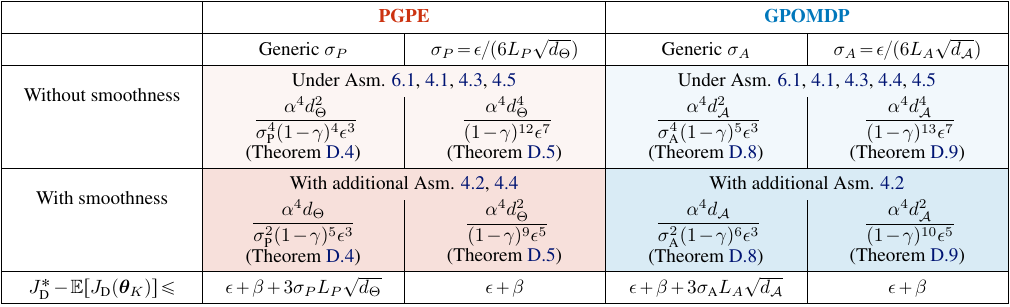}}
    
    \vspace{-.2cm}
    
    \captionof{table}{Sample complexity $NK=\widetilde{O}(\cdot)$ of GPOMDP and PGPE to converge to a deterministic optimal policy, retaining only dependencies on $\epsilon$, $(1-\gamma)^{-1}$, $\sigma_{\text{A}}$, $\sigma_\text{P}$, $\dt$, $\da$, and $\alpha$. Task-dependent constants $L_P$ and $L_A$ are $O((1-\gamma)^{-2})$---see Table~\ref{tab:magic} in Appendix~\ref{app:quick}.}
    \label{tab:convergence}
\end{figure*}

\subsection{Global Convergence of PGPE and GPOMDP}\label{sec:convBoth}
In this section, we provide results on the global convergence of PGPE and GPOMDP with white-noise exploration. The sample complexity bounds are summarized in Table~\ref{tab:convergence} and presented extensively in Appendix~\ref{apx:proofs}. They all follow from our general Theorem~\ref{thr:gen_conv_new_main} and our results on the deployment of deterministic policies from Section~\ref{sec:deploy}.


\textbf{\textcolor{vibrantRed}{PGPE.}}~~We start by commenting on the sample complexity of PGPE for a constant, generic hyperpolicy variance $\sigma_P$, shown in the first column (Table~\ref{tab:convergence}). First, the guarantee on $J_D^*-\E[J_D(\vtheta_K)]$ contains the additional variance-dependent term $3L_P\sqrt{\dt} \sigma_P
$ originating from the deterministic deployment. Second, the sample complexity scales with $\widetilde{O}(\epsilon^{-3})$.
Third, by enforcing the smoothness of the MDP and of the deterministic policy (Asm.~\ref{ass:smoothMdp} and~\ref{ass:smoothPol}), we improve the dependence on $\dt$ and on 
$\sigma_{\text{P}}$ at the price of an additional $(1-\gamma)^{-1}$ factor. 

A choice of $\sigma_{\text{P}}$ which adapts to $\epsilon$ allows us to achieve the global convergence on the deterministic objective $\Jd$, up to $\epsilon+\beta$ only.
Moving to the second column (Table~\ref{tab:convergence}), we 
observe that the convergence rate becomes $\widetilde{O}(\epsilon^{-7})$, which reduces to $\widetilde{O}(\epsilon^{-5})$ with the additional smoothness assumptions, which also improve the dependence on \emph{both} $(1-\gamma)^{-1}$ and $\dt$. The slower rate $\epsilon^{-5}$ or $\epsilon^{-7}$, compared to the $\epsilon^{-3}$ of the fixed-variance case, is easily explained by the more challenging requirement of converging to the \emph{optimal deterministic policy} rather than the \emph{optimal stochastic hyperpolicy}, as for standard PGPE. 
Note that we have set the standard deviation equal to $\sigma_{\text{P}} = \frac{\epsilon}{6L_P\sqrt{\dt}} = O(\epsilon (1-\gamma)^2 \dt^{-1/2})$ that, as expected, decreases with the desired accuracy $\epsilon$.\footnote{These results should be interpreted as a demonstration that global convergence to deterministic policies \emph{is possible} rather than a \emph{practical recipe} to set the value of $\sigma_{\text{P}}$. We do hope that our theory can guide the design of practical solutions in future works.}



\textbf{\textcolor{vibrantBlue}{GPOMDP.}}~~We now consider the global convergence of GPOMDP, starting again with a generic policy variance $\sigma_A$ (third column, Table~\ref{tab:convergence}).
The result is similar to that of PGPE with three notable exceptions. First, an additional $(1-\gamma)^{-1}$ factor appears in the sample complexity due to the variance bound of GPOMDP~\citep{papini2022smoothing}. This suggests that GPOMDP struggles more than PGPE in long-horizon environments, as already observed by~\citet{zhao2011analysis}. 
Second, the dependence on the dimensionality of the parameter space $\dt$ is replaced with the dimensionality of the action space $\da$. This is expected and derives from the nature of exploration that is performed in the parameter space for PGPE and in the action space for GPOMPD. Finally, the smoothness of the deterministic policy (Asm.~\ref{ass:smoothPol}) is always needed. Adding also the smoothness of the MDP (Asm.~\ref{ass:smoothMdp}), we lose a $\da$ factor getting a $(1-\gamma)^{-1}$ one.

Again, a careful $\epsilon$-dependent choice of $\sigma_{\text{A}}$ allows us to achieve global convergence on the deterministic objective $\Jd$.
In the last column (Table~\ref{tab:convergence}), we can notice that the convergence rates display the same dependence on $\epsilon$ as in PGPE. However, the dependence on the effective horizon $(1-\gamma)^{-1}$ is worse. In this case, the additional smoothness assumption improves the dependency on $\da$ and $(1-\gamma)^{-1}$.

\section{About the Weak Gradient Domination}
\label{sec:assumptions}
So far, we have assumed WGD for the AB $\Ja$ and PB $\Jp$ (Asm.~\ref{ass:J_wgd}). In this section, we discuss several scenarios in which such an assumption holds.

\subsection{Inherited Weak Gradient Domination}\label{sec:inherited}
We start by discussing the case in which the deterministic policy objective $\Jd$ already enjoys the (W)GD property.

\begin{ass}[%
Weak gradient domination for $\Jd$%
] \label{ass:J_wgdD}
    There exist $\dalpha > 0$ and $\dbeta \geq 0$ such that  for every $ \vtheta \in \Theta$ it holds that $\Jd^* - \Jd(\vtheta) \leq \dalpha \|\nabla_{\vtheta} \Jd (\vtheta) \|_2 + \dbeta$.
\end{ass}

Although the notion of WGD has been mostly applied to stochastic policies in the literature~\citep{liu2020improved,yuan2022general}, there is no reason why it should not be plausible for deterministic policies.~\citet{bhandari2019global} provide sufficient conditions for the performance function not to have any local optima, which is a stronger condition, without discriminating between deterministic and stochastic policies (see their Remark 1). Moreover, one of their examples is linear-quadratic regulators with \emph{deterministic} linear policies.

We show that, under Lipschiztianity and smoothness of the MDP and the deterministic policy (Asm.~\ref{ass:lipMdp} and~\ref{ass:smoothPol}), this is sufficient to enforce the WGD property for both the PB $\Jp$ and the AB $\Ja$ objectives. Let us start with $\Jp$.
\begin{restatable}[%
Inherited weak gradient domination for $\Jp$%
]{thr}{pgpeInheritedWGD} \label{thr:Jp_weak_gradient_domination}
	Under Assumptions~\ref{ass:lipMdp}, \ref{ass:smoothMdp}, \ref{ass:lipPol}, \ref{ass:smoothPol}, \ref{ass:J_wgdD}, for every $ \vtheta \in \Theta$:
    \begin{equation*}\resizebox{8cm}{!}{$\displaystyle
        \Jp^* - \Jp(\vtheta) \leq \dalpha \| \nabla_{\vtheta} \Jp(\vtheta) \|_2 + \dbeta + (\dalpha L_2 + L_P) \psigma \sqrt{\dt},$}
    \end{equation*}
    where $L_2 = O((1-\gamma)^{-3})$ (full expression in Lemma~\ref{lem:L_2_characterization}).
\end{restatable}
The result shows that the WGD property of $\Jd$ entails that of $\Jp$ with the same $\dalpha$ coefficient, but a different $\beta =\dbeta (\dalpha L_2 + L_P) \psigma \sqrt{\dt}$ that accounts for the gap between the two objectives encoded in $\sigma_{\text{P}}$. Note that even if $\Jd$ enjoys a (strong) GD (i.e., $\dbeta=0$), in general, $\Jp$ inherits a WGD property. 
In the setting of Theorem~\ref{thr:Jp_weak_gradient_domination}, convergence in the sense of $\Jd^*-\E[\Jd(\vtheta_K)]\le\epsilon+\dbeta$ can be achieved with $\widetilde{O}(\dalpha^6\epsilon^{-5}\dt^2(1-\gamma)^{-11})$ samples by carefully setting the hyperpolicy variance (see Theorem~\ref{thr:pgpe_sam_comp_wgdInherited} for details). 

An analogous result can be obtained for AB exploration.
\begin{restatable}[%
Inherited weak gradient domination on $\Ja$%
]{thr}{pgInheritedWGD} \label{thr:Ja_weak_gradient_domination}
	Under Assumptions~\ref{ass:lipMdp}, \ref{ass:smoothMdp}, \ref{ass:lipPol}, \ref{ass:smoothPol}, \ref{ass:J_wgdD}, for every $ \vtheta \in \Theta$:
	\begin{equation*}\resizebox{8cm}{!}{$\displaystyle
		\Ja^* - \Ja(\vtheta)  \leq \dalpha \| \nabla_{\vtheta} \Ja(\vtheta) \|_2 + \dbeta + (\dalpha \psi + L_A) \asigma \sqrt{\da},$}
	\end{equation*}
    where $\psi = O((1 - \gamma)^{-4})$ (full expression in the proof).
\end{restatable}
The sample complexity, in this case, is $\widetilde{O}(\alpha_D^6\epsilon^{-5}\da^2(1-\gamma)^{-14})$ (see Theorem~\ref{thr:gpomgp_sam_comp_wgdInherited} for details). 

\subsection{Policy-induced Weak Gradient Domination}\label{sec:fisher}
When the objective function does not enjoy weak gradient domination in the space of deterministic policies, we can still have WGD \wrt \emph{stochastic} policies if they satisfy a condition known as Fisher-non-degeneracy~\citep{liu2020improved,ding2022global}. As far as we know, WGD by Fisher-non-degeneracy is a peculiar property of AB exploration that has no equivalent in PB exploration.
%
White-noise policies satisfying Assumption~\ref{ass:magic}  are Fisher-non-degenerate under the following standard assumption~\citep{liu2020improved}.
\begin{ass}[Explorability]\label{ass:explore} There exists $\lambdaexp>0$ s.t.
    $\E_{\pi_{\vtheta}}[\nabla_{\vtheta}\bm{\mu}_{\vtheta}(\vs)\nabla_{\vtheta}\bm{\mu}_{\vtheta}(\vs)^\top] \succeq \lambdaexp \mathbf{I}$ for all $\vtheta\in\Theta$, where the expectation over states is induced by the \emph{stochastic} policy.
\end{ass}
We can use this fact to prove WGD for white-noise policies. 
\begin{restatable}[Policy-induced weak gradient domination]{thr}{wdgByFisher}\label{lem:wgd_by_fisher}
    Under Assumptions~\ref{ass:magic} and~\ref{ass:explore}, we have:  
    \begin{equation*}
        \Ja^* - \Ja(\vtheta) \le C\frac{\sqrt{d_{\As}}\sigma_A}{\lambdaexp}\|\nabla_{\vtheta}\Ja(\vtheta)\|_2 + \frac{\sqrt{\epsilon_\mathrm{bias}}}{1-\gamma},
    \end{equation*}
    for some numerical constant $C>0$. Thus, Assumption~\ref{ass:J_wgd} ($\dagger$=A) is satisfied with $\alpha=C\frac{\sqrt{d_{\As}}\sigma_A}{\lambdaexp}$ and $\beta=\frac{\sqrt{\epsilon_\mathrm{bias}}}{1-\gamma}$.
\end{restatable}
Here $\epsilon_\mathrm{bias}$ is the \emph{compatible-critic error}, which can be very small for rich policy classes~\citep{ding2022global}.\footnote{A formal definition of $\varepsilon_{\mathrm{bias}}$ can be found in Appendix~\ref{app:fisher}.}
We can leverage this to prove the global convergence of GPOMDP as in Section~\ref{sec:inherited}, this time to $\Jd-\E[\Jd(\vtheta)]\le \epsilon+\frac{\sqrt{\epsilon_\mathrm{bias}}}{1-\gamma}$.

Tuning $\asigma$, we can achieve a sample complexity of 
$\widetilde{O}(\epsilon^{-1}\lambdaexp^{-4}d_{\As}^4(1-\gamma)^{-10})$ 
(see Theorem~\ref{thr:gpomdp_fisher_optimal} for details)
%
This seems to violate the $\Omega(\epsilon^{-2})$ lower bound by~\citet{azar2013minimax}.
However, the factor $\lambdaexp$ can depend on $\asigma= O(\epsilon)$ in highly non-trivial ways and, thus, can hide additional factors of $\epsilon$. 
For this reason, the results granted by the Fisher-non-degeneracy of white-noise policies are not compared with the ones granted by inherited WGD from Section~\ref{sec:inherited}. Intuitively, $\lambdaexp$ encodes some difficulties of exploration that are absent in ``nice" MDPs satisfying Assumption~\ref{ass:J_wgdD}.
See Appendix~\ref{app:fisher} for further discussion and omitted proofs. 

\section{Related Works}
\label{sec:related}
In this section, we provide a discussion of previous works that addressed similar questions to the ones considered in this paper. Additional related works in Appendix~\ref{apx:addRel}.

\noindent\textbf{Convergence rates.}~~The convergence of PG to stationary points at a rate of $O(\epsilon^{-4})$ was clear at least since~\cite{sutton1999policy}, although the recent work by~\citet{yuan2022general} clarifies several aspects of the analysis and the required assumptions. Variants of REINFORCE with faster convergence, based on stochastic variance reduction, were explored much later~\citep{papini2018stochastic, xu2019improved}, and the $O(\epsilon^{-3})$ rate of~\cite{xu2020sample} is now believed to be optimal due to lower bounds from nonconvex stochastic optimization~\citep{yossi2023lower}. The same holds for second-order methods~\citep{shen2019hessian,yossi2020second}. Although the convergence properties of PGPE are analogous to those of PG, they have not received the same attention, with the exception of~\cite{xu2020sample}, where the $O(\epsilon^{-3})$ rate is proved for a variance-reduced version of PGPE. Studying the convergence of PG to globally optimal policies under additional assumptions is a more recent endeavor, pioneered by works such as~\citet{scherrer2014local},~\citet{fazel2018global},~\citet{bhandari2019global}. These works introduced to the policy gradient literature the concept of gradient domination, or gradient \emph{dominance}, or Polyak-Łojasiewicz condition, which has a long history in the optimization literature~\citep{lojasiewicz1963propriete,polyak1963gradient,karimi2016linear}. Several works study the iteration complexity of policy gradient with exact gradients~\citep[e.g., ][]{agarwal2021theory,mei2020global,lisoftmax2021}. These results are restricted to specific policy classes (e.g., softmax, direct tabular parametrization) for which gradient domination is guaranteed. A notable exception is the study of sample-based \emph{natural} policy gradient for general smooth policies~\citep{agarwal2021theory}. As for vanilla sample-based PG (i.e., GPOMDP),~\citet{liu2020improved} were the first to study the sample complexity of this algorithm in converging to a global optimum. They also introduced the concept of Fisher-non-degeneracy~\citep{ding2022global}, which allows to exploit a form of gradient domination for a general class of policies. We refer the reader to~\cite{yuan2022general} which achieves a better $\widetilde{O}(\epsilon^{-3})$ sample complexity under weaker assumptions. More sophisticated algorithms, such as variance-reduced methods mentioned above, can achieve even better sample complexity. The current state of the art is~\citep{fatkhullin2023stochastic}: $\widetilde{O}(\epsilon^{-2.5})$ for hessian-free and $\widetilde{O}(\epsilon^{-2})$ for second-order algorithms. The latter is optimal up to logarithmic terms~\citep{azar2013minimax}. When instantiated to Gaussian policies, all of the works mentioned in this paragraph implicitly assume that the covariance parameters are fixed. In this case, our Theorem~\ref{thr:pgpe_sam_comp_wgd} recovers the $\widetilde{O}(\epsilon^{-3})$ rate of~\citet[][Corollary~3.7]{yuan2022general}, the best-known result for GPOMDP under general WGD.

\vspace{-0.1cm}

\noindent\textbf{Deterministic policies.}~~Value-based RL algorithms, such as Q-learning, naturally produce deterministic policies as their final solution, while most policy-gradient methods must search, by design, in a space of non-degenerate stochastic policies. In~\citep{sutton1999policy}, this is presented as an opportunity rather than as a limitation since the optimal policy is often stochastic for partially observable problems. The possibility of deploying deterministic policies only is one of the appeals of PGPE and related evolutionary techniques~\citep{schwefel1993evolution}, but also of model-based approaches~\citep{deisenroth2011pilco}. In the context of action-based policy search, the DPG algorithm by~\citet{silver2014deterministic} was the first to search in a space of deterministic policies. Differently from PGPE, stochastic policies are run during the learning process for exploration purposes, similarly to value-based methods. Moreover, the distribution mismatch due to off-policy sampling is largely ignored. Nonetheless, popular deep RL algorithms were derived from DPG~\citep{lillicrap2016continuous,fujimoto2018addressing}. 
\citep{xiong2022deterministic} proved the convergence of \emph{on-policy} (hence, fully deterministic) DPG to a stationary point, with $O(\epsilon^{-4})$ sample complexity. However, they rely on an explorability assumption (Asm.~4 in their paper) that is standard for stochastic policies, but very demanding for deterministic policies.
A more practical way of achieving fully deterministic DPG was proposed by~\citet{saleh2022truly}, who also provide a discussion of the advantages of deterministic policies. Unsurprisingly, truly deterministic learning is only possible under strong assumptions on the regularity of the environment. In this paper, for PG, we considered the more common scenario of evaluating stochastic policies at training time, only to deploy a good deterministic policy in the end. PGPE, by design, does the same with hyperpolicies.

\section{Numerical Validation}
\label{sec:experiments}
In this section, we empirically validate the theoretical results presented in the paper.
We conduct a study on the gap in performance between the deterministic objective $\Jd$ and the ones 
of GPOMDP and PGPE (respectively $\Ja$ and $\Jp$) by varying the value of their exploration parameters ($\asigma$ and $\psigma$, respectively).
Details on the employed versions of PGPE and GPOMDP can be found in Appendix~\ref{apx:algo}.
Additional experimental results can be found in Appendix~\ref{apx:experiments}.

We run PGPE and GPOMDP for $K=2000$ iterations with batch size $N=100$ on three environments from the MuJoCo~\citep{todorov2012mujoco} suite: \emph{Swimmer-v4} ($T=200$), \emph{Hopper-v4} ($T=100$), and \emph{HalfCheetah-v4} ($T=100$).
For all the environments 
the deterministic policy 
is linear in the state and the noise is Gaussian. 
We consider $\sigma_{\simbolo}^2 \in \{0.01, 0.1, 1, 10, 100\}$.
More details in Appendix~\ref{apx:experiments:curves}.\footnote{The code is available at \url{https://github.com/MontenegroAlessandro/MagicRL}.}

From Figure~\ref{fig:variance_study}, we note that as the exploration parameter grows, the distance of $\Jp(\vtheta_{K})$ and $\Ja(\vtheta_{K})$ from $\Jd(\vtheta_{K})$ increases, coherently with Theorems~\ref{thr:deployPB} and \ref{thr:deployAB}.
Among the tested values for $\psigma$ and $\asigma$, 
some lead to the highest values of $\Jd(\vtheta_{K})$.
Empirically, we note that PGPE delivers the best deterministic policy with $\psigma^2 = 10$ for \emph{Swimmer} and  with $\psigma^2 = 1$ for the other environments. 
GPOMDP performs the best with $\asigma^2=1$ for \emph{Swimmer}, and with $\asigma^2 = 10$ in the other cases.
These outcomes agree with the theoretical results in showing that there exists an optimal value for $\sigma_{\simbolo}$.

We can also appreciate the trade-off between GPOMDP and PGPE w.r.t. $\dt$ and $T$, by comparing the best values of $\Jd$ found by the two algorithms in each environment.
GPOMDP 
is better than 
PGPE 
in \emph{Hopper} and \emph{HalfCheetah}.
Indeed, such environments are characterized by higher values of $\dt$.
Instead in \emph{Swimmer}, PGPE performs better than GPOMDP, since $T$ is higher and $\dt$ is lower.

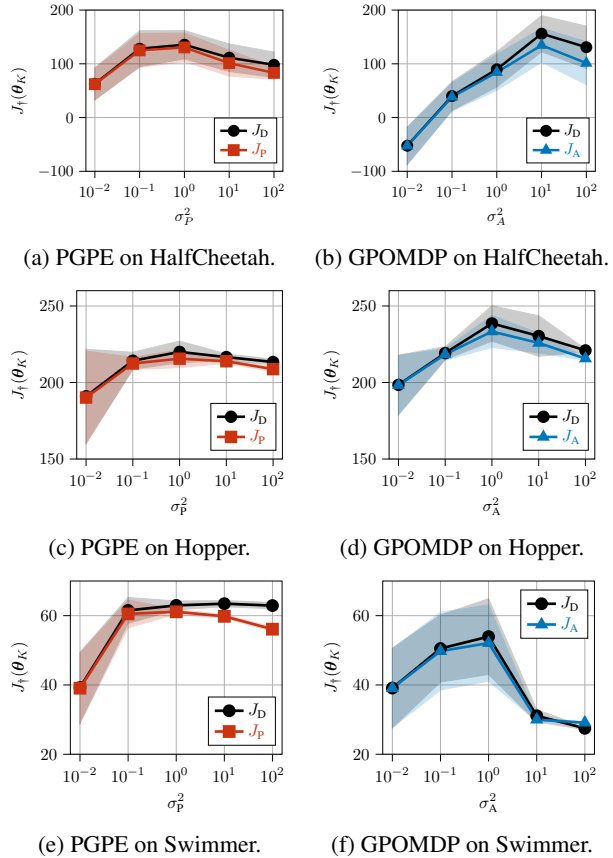
\begin{figure}[t!]
    \centering
    \subfloat[PGPE on HalfCheetah.]{
        \resizebox{0.465\linewidth}{!}{
\begin{tikzpicture}

\definecolor{darkgray176}{RGB}{176,176,176}
\definecolor{firebrick2035016}{RGB}{203,50,16}

\begin{axis}[
legend cell align={left},
legend style={
  draw opacity=1,
  text opacity=1,
  at={(0.97,0.05)},
  anchor=south east,
},
height=5.2cm,
width=6cm,
log basis x={10},
tick align=outside,
tick pos=left,
x grid style={darkgray176},
xlabel={\(\displaystyle \sigma_{P}^2\)},
xmajorgrids,
xmin=0.00630957344480193, xmax=158.489319246111,
xmode=log,
xtick style={color=black},
xtick={0.0001,0.001,0.01,0.1,1,10,100,1000,10000},
xticklabels={
  \(\displaystyle {10^{-4}}\),
  \(\displaystyle {10^{-3}}\),
  \(\displaystyle {10^{-2}}\),
  \(\displaystyle {10^{-1}}\),
  \(\displaystyle {10^{0}}\),
  \(\displaystyle {10^{1}}\),
  \(\displaystyle {10^{2}}\),
  \(\displaystyle {10^{3}}\),
  \(\displaystyle {10^{4}}\)
},
y grid style={darkgray176},
ylabel={$J_{\simbolo}(\vtheta_{K})$},
ymajorgrids,
ymin=-100, ymax=200,
ytick style={color=black}
]
\path [draw=black, fill=black, opacity=0.2]
(axis cs:0.01,92.7308720709284)
--(axis cs:0.01,32.7763258207019)
--(axis cs:0.1,94.6814069514699)
--(axis cs:1,109.240949401271)
--(axis cs:10,85.7342071298404)
--(axis cs:100,74.5916159766706)
--(axis cs:100,121.736515261894)
--(axis cs:100,121.736515261894)
--(axis cs:10,137.008354595159)
--(axis cs:1,161.889562396076)
--(axis cs:0.1,161.560914706746)
--(axis cs:0.01,92.7308720709284)
--cycle;

\path [draw=firebrick2035016, fill=firebrick2035016, opacity=0.2]
(axis cs:0.01,91.298007309583)
--(axis cs:0.01,32.950428912458)
--(axis cs:0.1,93.6581630350401)
--(axis cs:1,103.882767815887)
--(axis cs:10,77.2487703445653)
--(axis cs:100,70.9993094819309)
--(axis cs:100,95.2622561447675)
--(axis cs:100,95.2622561447675)
--(axis cs:10,125.279068922501)
--(axis cs:1,157.189129498703)
--(axis cs:0.1,156.37532533071)
--(axis cs:0.01,91.298007309583)
--cycle;

\addplot [ultra thick, black, mark=*, mark size=3, mark options={solid}]
table {%
0.01 62.7535989458152
0.1 128.121160829108
1 135.565255898673
10 111.3712808625
100 98.1640656192822
};
\addlegendentry{$\Jd$}
\addplot [ultra thick, firebrick2035016, mark=square*, mark size=3, mark options={solid}]
table {%
0.01 62.1242181110205
0.1 125.016744182875
1 130.535948657295
10 101.263919633533
100 83.1307828133492
};
\addlegendentry{$\Jp$}
\end{axis}
\end{tikzpicture}}
        \label{fig:variance_hc_pgpe}
    }
    \hfill
    \subfloat[GPOMDP on HalfCheetah.]{
        \resizebox{0.465\linewidth}{!}{
\begin{tikzpicture}

\definecolor{darkcyan0118186}{RGB}{0,118,186}
\definecolor{darkgray176}{RGB}{176,176,176}

\begin{axis}[
legend cell align={left},
legend style={
  draw opacity=1,
  text opacity=1,
  at={(0.97,0.05)},
  anchor=south east,
},
height=5.2cm,
width=6cm,
log basis x={10},
tick align=outside,
tick pos=left,
x grid style={darkgray176},
xlabel={\(\displaystyle \sigma_{A}^2\)},
xmajorgrids,
xmin=0.00630957344480193, xmax=158.489319246111,
xmode=log,
xtick style={color=black},
xtick={0.0001,0.001,0.01,0.1,1,10,100,1000,10000},
xticklabels={
  \(\displaystyle {10^{-4}}\),
  \(\displaystyle {10^{-3}}\),
  \(\displaystyle {10^{-2}}\),
  \(\displaystyle {10^{-1}}\),
  \(\displaystyle {10^{0}}\),
  \(\displaystyle {10^{1}}\),
  \(\displaystyle {10^{2}}\),
  \(\displaystyle {10^{3}}\),
  \(\displaystyle {10^{4}}\)
},
y grid style={darkgray176},
ylabel={$J_{\simbolo}(\vtheta_{K})$},
ymajorgrids,
ymin=-100, ymax=200,
ytick style={color=black}
]
\path [draw=black, fill=black, opacity=0.2]
(axis cs:0.01,-17.6763254473344)
--(axis cs:0.01,-87.2648304647577)
--(axis cs:0.1,13.0792734693349)
--(axis cs:1,56.8113757447342)
--(axis cs:10,122.4058808567)
--(axis cs:100,91.6984614423087)
--(axis cs:100,170.073909835926)
--(axis cs:100,170.073909835926)
--(axis cs:10,189.752596267244)
--(axis cs:1,122.576834110584)
--(axis cs:0.1,66.7926075219865)
--(axis cs:0.01,-17.6763254473344)
--cycle;

\path [draw=darkcyan0118186, fill=darkcyan0118186, opacity=0.2]
(axis cs:0.01,-18.282425709171)
--(axis cs:0.01,-87.0608809873365)
--(axis cs:0.1,12.0916275608749)
--(axis cs:1,50.6464639375196)
--(axis cs:10,102.973555917025)
--(axis cs:100,61.3555759246927)
--(axis cs:100,141.274423613582)
--(axis cs:100,141.274423613582)
--(axis cs:10,166.202817982415)
--(axis cs:1,117.605257951943)
--(axis cs:0.1,64.4423089242809)
--(axis cs:0.01,-18.282425709171)
--cycle;

\addplot [ultra thick, black, mark=*, mark size=3, mark options={solid}]
table {%
0.01 -52.4705779560461
0.1 39.9359404956607
1 89.6941049276592
10 156.079238561972
100 130.886185639117
};
\addlegendentry{$\Jd$}
\addplot [ultra thick, darkcyan0118186, mark=triangle*, mark size=3, mark options={solid}]
table {%
0.01 -52.6716533482537
0.1 38.2669682425779
1 84.1258609447312
10 134.58818694972
100 101.314999769137
};
\addlegendentry{$\Ja$}
\end{axis}
\end{tikzpicture}}
        \label{fig:variance_hc_pg}
    }

    \vspace{.2cm}

    \subfloat[PGPE on Hopper.]{
        \resizebox{0.465\linewidth}{!}{
\begin{tikzpicture}

\definecolor{darkgray176}{RGB}{176,176,176}
\definecolor{firebrick2035016}{RGB}{204,51,17}

\begin{axis}[
legend cell align={left},
legend style={
  draw opacity=1,
  text opacity=1,
  at={(0.97,0.05)},
  anchor=south east,
},
height=5.2cm,
width=6cm,
log basis x={10},
tick align=outside,
tick pos=left,
x grid style={darkgray176},
xlabel={\(\displaystyle \psigma^2\)},
xmajorgrids,
xmin=0.00630957344480193, xmax=158.489319246111,
xmode=log,
xtick style={color=black},
xtick={0.0001,0.001,0.01,0.1,1,10,100,1000,10000},
xticklabels={
  \(\displaystyle {10^{-4}}\),
  \(\displaystyle {10^{-3}}\),
  \(\displaystyle {10^{-2}}\),
  \(\displaystyle {10^{-1}}\),
  \(\displaystyle {10^{0}}\),
  \(\displaystyle {10^{1}}\),
  \(\displaystyle {10^{2}}\),
  \(\displaystyle {10^{3}}\),
  \(\displaystyle {10^{4}}\)
},
y grid style={darkgray176},
ylabel={$J_{\simbolo}(\vtheta_{K})$},
ymajorgrids,
ymin=150, ymax=260,
ytick style={color=black}
]
\path [draw=black, fill=black, opacity=0.2]
(axis cs:0.01,221.718647016274)
--(axis cs:0.01,160.3413273946)
--(axis cs:0.1,208.707719662799)
--(axis cs:1,212.901520188875)
--(axis cs:10,214.530847571976)
--(axis cs:100,211.626848580043)
--(axis cs:100,215.185497509745)
--(axis cs:100,215.185497509745)
--(axis cs:10,218.707088396945)
--(axis cs:1,227.065506218703)
--(axis cs:0.1,219.762790304797)
--(axis cs:0.01,221.718647016274)
--cycle;

\path [draw=firebrick2035016, fill=firebrick2035016, opacity=0.2]
(axis cs:0.01,220.472115374865)
--(axis cs:0.01,159.82294692517)
--(axis cs:0.1,208.431569371054)
--(axis cs:1,209.775007469089)
--(axis cs:10,212.462315739838)
--(axis cs:100,207.08874575076)
--(axis cs:100,210.377810452219)
--(axis cs:100,210.377810452219)
--(axis cs:10,215.309363830919)
--(axis cs:1,221.373362074865)
--(axis cs:0.1,216.463118764155)
--(axis cs:0.01,220.472115374865)
--cycle;

\addplot [ultra thick, black, mark=*, mark size=3, mark options={solid}]
table {%
0.01 191.029987205437
0.1 214.235254983798
1 219.983513203789
10 216.61896798446
100 213.406173044894
};
\addlegendentry{$\Jd$}
\addplot [ultra thick, firebrick2035016, mark=square*, mark size=3, mark options={solid}]
table {%
0.01 190.147531150017
0.1 212.447344067605
1 215.574184771977
10 213.885839785379
100 208.73327810149
};
\addlegendentry{$\Jp$}
\end{axis}
\end{tikzpicture}}
        \label{fig:variance_hop_pgpe}
    }
    \hfill
    \subfloat[GPOMDP on Hopper.]{
        \resizebox{0.465\linewidth}{!}{
\begin{tikzpicture}

\definecolor{darkcyan0118186}{RGB}{0,119,187}
\definecolor{darkgray176}{RGB}{176,176,176}

\begin{axis}[
legend cell align={left},
legend style={
  draw opacity=1,
  text opacity=1,
  at={(0.97,0.05)},
  anchor=south east,
},
height=5.2cm,
width=6cm,
log basis x={10},
tick align=outside,
tick pos=left,
x grid style={darkgray176},
xlabel={\(\displaystyle \asigma^2\)},
xmajorgrids,
xmin=0.00630957344480193, xmax=158.489319246111,
xmode=log,
xtick style={color=black},
xtick={0.0001,0.001,0.01,0.1,1,10,100,1000,10000},
xticklabels={
  \(\displaystyle {10^{-4}}\),
  \(\displaystyle {10^{-3}}\),
  \(\displaystyle {10^{-2}}\),
  \(\displaystyle {10^{-1}}\),
  \(\displaystyle {10^{0}}\),
  \(\displaystyle {10^{1}}\),
  \(\displaystyle {10^{2}}\),
  \(\displaystyle {10^{3}}\),
  \(\displaystyle {10^{4}}\)
},
y grid style={darkgray176},
ylabel={$J_{\simbolo}(\vtheta_{K})$},
ymajorgrids,
ymin=150, ymax=260,
ytick style={color=black}
]
\path [draw=black, fill=black, opacity=0.2]
(axis cs:0.01,217.81819560921)
--(axis cs:0.01,179.093713140226)
--(axis cs:0.1,215.035868352128)
--(axis cs:1,227.008423235486)
--(axis cs:10,217.133509618678)
--(axis cs:100,219.983496402366)
--(axis cs:100,221.951400935447)
--(axis cs:100,221.951400935447)
--(axis cs:10,243.576795349739)
--(axis cs:1,250.147835460241)
--(axis cs:0.1,223.409225815208)
--(axis cs:0.01,217.81819560921)
--cycle;

\path [draw=darkcyan0118186, fill=darkcyan0118186, opacity=0.2]
(axis cs:0.01,217.734210255714)
--(axis cs:0.01,178.815699175739)
--(axis cs:0.1,214.623187990061)
--(axis cs:1,222.870625948286)
--(axis cs:10,219.394101051028)
--(axis cs:100,214.511520218674)
--(axis cs:100,216.797584297395)
--(axis cs:100,216.797584297395)
--(axis cs:10,232.264399894417)
--(axis cs:1,243.794447854486)
--(axis cs:0.1,222.141147073994)
--(axis cs:0.01,217.734210255714)
--cycle;

\addplot [ultra thick, black, mark=*, mark size=3, mark options={solid}]
table {%
0.01 198.455954374718
0.1 219.222547083668
1 238.578129347863
10 230.355152484209
100 220.967448668907
};
\addlegendentry{$\Jd$}
\addplot [ultra thick, darkcyan0118186, mark=triangle*, mark size=3, mark options={solid}]
table {%
0.01 198.274954715726
0.1 218.382167532028
1 233.332536901386
10 225.829250472723
100 215.654552258034
};
\addlegendentry{$\Ja$}
\end{axis}
\end{tikzpicture}}
        \label{fig:variance_hop_pg}
    }

    \vspace{.2cm}

    \subfloat[PGPE on Swimmer.]{
        \resizebox{0.465\linewidth}{!}{
\begin{tikzpicture}

\definecolor{darkgray176}{RGB}{176,176,176}
\definecolor{firebrick2035016}{RGB}{204,51,17}

\begin{axis}[
legend cell align={left},
legend style={
  draw opacity=1,
  text opacity=1,
  at={(0.97,0.05)},
  anchor=south east,
},
height=5.2cm,
width=6cm,
log basis x={10},
tick align=outside,
tick pos=left,
x grid style={darkgray176},
xlabel={\(\displaystyle \psigma^2\)},
xmajorgrids,
xmin=0.00630957344480193, xmax=158.489319246111,
xmode=log,
xtick style={color=black},
xtick={0.0001,0.001,0.01,0.1,1,10,100,1000,10000},
xticklabels={
  \(\displaystyle {10^{-4}}\),
  \(\displaystyle {10^{-3}}\),
  \(\displaystyle {10^{-2}}\),
  \(\displaystyle {10^{-1}}\),
  \(\displaystyle {10^{0}}\),
  \(\displaystyle {10^{1}}\),
  \(\displaystyle {10^{2}}\),
  \(\displaystyle {10^{3}}\),
  \(\displaystyle {10^{4}}\)
},
y grid style={darkgray176},
ylabel={$J_{\simbolo}(\vtheta_{K})$},
ymajorgrids,
ymin=20, ymax=70,
ytick style={color=black}
]
\path [draw=black, fill=black, opacity=0.2]
(axis cs:0.01,49.4640825694345)
--(axis cs:0.01,29.0549432870804)
--(axis cs:0.1,57.7790171277888)
--(axis cs:1,61.6936228402995)
--(axis cs:10,62.5668394299987)
--(axis cs:100,61.98955115903)
--(axis cs:100,63.8379517208324)
--(axis cs:100,63.8379517208324)
--(axis cs:10,64.3410246808092)
--(axis cs:1,64.2423343201004)
--(axis cs:0.1,65.2539978908142)
--(axis cs:0.01,49.4640825694345)
--cycle;

\path [draw=firebrick2035016, fill=firebrick2035016, opacity=0.2]
(axis cs:0.01,49.2092616050713)
--(axis cs:0.01,28.9354801488941)
--(axis cs:0.1,56.5106293079868)
--(axis cs:1,60.4315137188044)
--(axis cs:10,59.3036607250727)
--(axis cs:100,55.4742264592327)
--(axis cs:100,56.7260766311501)
--(axis cs:100,56.7260766311501)
--(axis cs:10,60.3228781349179)
--(axis cs:1,61.8092474565359)
--(axis cs:0.1,64.5385073971904)
--(axis cs:0.01,49.2092616050713)
--cycle;

\addplot [ultra thick, black, mark=*, mark size=3, mark options={solid}]
table {%
0.01 39.2595129282574
0.1 61.5165075093015
1 62.9679785802
10 63.4539320554039
100 62.9137514399312
};
\addlegendentry{$\Jd$}
\addplot [ultra thick, firebrick2035016, mark=square*, mark size=3, mark options={solid}]
table {%
0.01 39.0723708769827
0.1 60.5245683525886
1 61.1203805876702
10 59.8132694299953
100 56.1001515451914
};
\addlegendentry{$\Jp$}
\end{axis}
\end{tikzpicture}}
        \label{fig:variance_swimmer_pgpe}
    }
    \hfill
    \subfloat[GPOMDP on Swimmer.]{
        \resizebox{0.465\linewidth}{!}{
\begin{tikzpicture}

\definecolor{darkcyan0118186}{RGB}{0,119,187}
\definecolor{darkgray176}{RGB}{176,176,176}

\begin{axis}[
legend cell align={left},
legend style={
  draw opacity=1,
  text opacity=1,
  at={(0.97,0.95)},
  anchor=north east,
},
height=5.2cm,
width=6cm,
log basis x={10},
tick align=outside,
tick pos=left,
x grid style={darkgray176},
xlabel={\(\displaystyle \asigma^2\)},
xmajorgrids,
xmin=0.00630957344480193, xmax=158.489319246111,
xmode=log,
xtick style={color=black},
xtick={0.0001,0.001,0.01,0.1,1,10,100,1000,10000},
xticklabels={
  \(\displaystyle {10^{-4}}\),
  \(\displaystyle {10^{-3}}\),
  \(\displaystyle {10^{-2}}\),
  \(\displaystyle {10^{-1}}\),
  \(\displaystyle {10^{0}}\),
  \(\displaystyle {10^{1}}\),
  \(\displaystyle {10^{2}}\),
  \(\displaystyle {10^{3}}\),
  \(\displaystyle {10^{4}}\)
},
y grid style={darkgray176},
ylabel={$J_{\simbolo}(\vtheta_{K})$},
ymajorgrids,
ymin=20, ymax=70,
ytick style={color=black}
]
\path [draw=black, fill=black, opacity=0.2]
(axis cs:0.01,50.6064724615339)
--(axis cs:0.01,27.6069772524778)
--(axis cs:0.1,40.8377268505465)
--(axis cs:1,43.019545936134)
--(axis cs:10,29.3451563299489)
--(axis cs:100,27.0988348144394)
--(axis cs:100,27.8882882252758)
--(axis cs:100,27.8882882252758)
--(axis cs:10,32.8812888021278)
--(axis cs:1,64.8994422524388)
--(axis cs:0.1,60.1865383041738)
--(axis cs:0.01,50.6064724615339)
--cycle;

\path [draw=darkcyan0118186, fill=darkcyan0118186, opacity=0.2]
(axis cs:0.01,50.5176760048779)
--(axis cs:0.01,27.6684052556549)
--(axis cs:0.1,38.5595560130926)
--(axis cs:1,41.0753621730268)
--(axis cs:10,29.5188649613307)
--(axis cs:100,28.9513834953741)
--(axis cs:100,29.2877975171338)
--(axis cs:100,29.2877975171338)
--(axis cs:10,30.4686945604018)
--(axis cs:1,63.1222868250538)
--(axis cs:0.1,60.9608695717433)
--(axis cs:0.01,50.5176760048779)
--cycle;

\addplot [ultra thick, black, mark=*, mark size=3, mark options={solid}]
table {%
0.01 39.1067248570058
0.1 50.5121325773602
1 53.9594940942864
10 31.1132225660383
100 27.4935615198576
};
\addlegendentry{$\Jd$}
\addplot [ultra thick, darkcyan0118186, mark=triangle*, mark size=3, mark options={solid}]
table {%
0.01 39.0930406302664
0.1 49.760212792418
1 52.0988244990403
10 29.9937797608662
100 29.119590506254
};
\addlegendentry{$\Ja$}
\end{axis}
\end{tikzpicture}}
        \label{fig:variance_swimmer_pg}
    }
    
    \caption{Variance study on 
    Mujoco ($5$ runs, mean $\pm$ $95\%$ C.I.).}
    \label{fig:variance_study}
\end{figure}

\section{Conclusions}
\label{sec:conclusions}
In this work, we have perfected recent theoretical results on the global convergence of policy gradient algorithms 
to address the practical problem of finding a good \emph{deterministic} parametric policy.
We have studied the effects of noise on the learning process and identified a theoretical value of the variance of the (hyper)policy that allows to find a good deterministic policy using a polynomial number of samples.
We have compared the two common forms of noisy exploration, action-based and parameter-based, both from a theoretical and an empirical perspective. 

Our work paves the way for several exciting research directions. 
First, our theoretical selection of the policy variance is not practical, but our theoretical findings should guide the design of sound and efficient adaptive-variance schedules.
We have shown how white-noise exploration preserves \emph{weak} gradient domination---the natural next question is whether a sufficient amount of noise can smooth or even eliminate the local optima of the objective function. 
Finally, we have focused on ``vanilla" policy gradient methods, but 
our ideas could be applied to more advanced algorithms, such as the ones recently proposed by~\citet{fatkhullin2023stochastic}, to find optimal deterministic policies with $\widetilde{O}(\epsilon^{-2})$ samples.

\clearpage
\section*{Impact Statement} 
This paper presents work whose goal is to advance the field of Machine Learning. There are many potential societal consequences of our work, none which we feel must be specifically highlighted here.

\section*{Acknowledgements}
Funded by the European Union – Next Generation EU within the project NRPP M4C2, Investment 1.,3 DD. 341 -  15 march 2022 – FAIR – Future Artificial Intelligence Research – Spoke 4 - PE00000013 - D53C22002380006.

\bibliography{biblio}
\bibliographystyle{icml2024}

\newpage
\appendix
\onecolumn
\section{Assumptions and Constants: Quick Reference}\label{app:quick}
As mentioned in Section~\ref{sec:assumptions}, we can start from fundamental assumptions on the MDP and the (hyper)policy classes to satisfy more abstract assumptions that can be used directly in convergence analyses. Figure~\ref{fig:assum} shows the relationship between the assumptions, and Table~\ref{tab:magic} the constants obtained in the process. All proofs of the assumptions' implications can be found in Appendix~\ref{app:ass_impl}.

\newcolumntype{C}[1]{>{\centering\arraybackslash}p{#1}}
\begin{table*}[h!]
    \centering
    \renewcommand{\arraystretch}{1}
    \resizebox{\linewidth}{!}{
        \begin{tabular}{|C{2.5cm}|C{2.7cm}|C{6.8cm}|C{3.2cm}|C{3cm}|}\hline
             & $L_{\dagger} $ (Lipschitz) & \multicolumn{2}{c|}{ $L_{2,\dagger}$ (Smooth)} &   $V_{\dagger}$ (Variance bound)\\ \hline
             \textbf{\textcolor{vibrantBlue}{AB Exploration}}  ($\dagger$=A) & \cellcolor{brightYellow!20} $  \frac{L_p R_{\max}}{(1-\gamma)^2} + \frac{ L_r}{1 - \gamma} $ & \cellcolor{vibrantMagenta!20} $ \frac{2L_p^2 L_\mu^2 R_{\max}}{(1-\gamma)^3} + \frac{2L_\mu^2L_p L_r+ L_{2,\mu}L_{2,p}R_{\max} }{(1-\gamma)^2} + \frac{L_{2,\mu}L_{2,r}}{1-\gamma}  $  & \cellcolor{vibrantMagenta!20}  $ \frac{R_{\max}c (\da+1)(L_\mu^2 + L_{2,\mu})}{\sigma^2_\text{A}(1-\gamma)^2} $\textsuperscript{$\ddagger$} & \cellcolor{vibrantBlue!20}  $ \frac{R_{\max} c \da L_\mu^2}{\sigma_{\text{A}}^2(1-\gamma)^3}$ \\\cdashline{2-5}
             {\small \hfill  Assumptions:} & {\small \ref{ass:lipMdp}} & {\small \ref{ass:lipMdp}, \ref{ass:smoothMdp}, \ref{ass:lipPol}, \ref{ass:smoothPol} } & {\small  \ref{ass:lipPol}, \ref{ass:smoothPol}, \ref{ass:magic}} & {\small \ref{ass:lipPol}, \ref{ass:magic} } \\\cdashline{2-5}
             {\small \hfill Reference:}  & {\small Lemma~\ref{thr:assImpl}} 
             &  \multicolumn{2}{c|}{\small Lemma~\ref{lem:Ja_bounded_hessian}}&  {\small Lemma \ref{ass:Ja_estim_bounded_variance}} \\\hline
             \textbf{\textcolor{vibrantRed}{PB Exploration}}  ($\dagger$=P)& $ \cellcolor{brightYellow!30}  \frac{L_p L_\mu R_{\max}}{(1-\gamma)^2} + \frac{ L_r L_\mu }{1 - \gamma} $  & \cellcolor{vibrantMagenta!20} $ \frac{2L_p^2 L_\mu^2 R_{\max}}{(1-\gamma)^3} + \frac{2L_\mu^2L_p L_r+ L_{2,\mu}L_{2,p}R_{\max} }{(1-\gamma)^2} + \frac{L_{2,\mu}L_{2,r}}{1-\gamma}  $ & \cellcolor{vibrantMagenta!20} $ \frac{R_{\max}c (\dt+1)}{\sigma^2_\text{P}(1-\gamma)^2} $  & \cellcolor{vibrantBlue!20} $  \frac{R_{\max} c \dt}{\sigma_{\text{P}}^2(1-\gamma)^2}$ \\\cdashline{2-5}
             {\small \hfill Assumptions:} & {\small \ref{ass:lipMdp}, \ref{ass:lipPol} } & {\small \ref{ass:lipMdp}, \ref{ass:smoothMdp}, \ref{ass:lipPol}, \ref{ass:smoothPol} } & {\small \ref{ass:magic}} & {\small \ref{ass:magic} } \\\cdashline{2-5}
             {\small \hfill Reference:}  & {\small Lemma~\ref{thr:assImpl}} 
             &   \multicolumn{2}{c|}{\small Lemma~\ref{lem:Jp_bounded_hessian}}  &  {\small Lemma~\ref{ass:Jp_estim_bounded_variance}} \\
             \hline
        \end{tabular}
    }
    \caption{Bounds to the Lipschitz and smoothness constants for the AB and PB objectives ($\Ja$ and $\Jp$) and variance of the GPOMDP and PGPE estimators. Both presented bounds on $L_{2,\dagger}$ hold under different sets of assumptions. ${}^\ddagger$ if $\sigma_{\text{A}} < \sqrt{\da}$.}\label{tab:magic}
\end{table*}

\begin{landscape}

\tikzstyle{block} = [draw=black, align=center, text width=4cm, node distance=2cm]
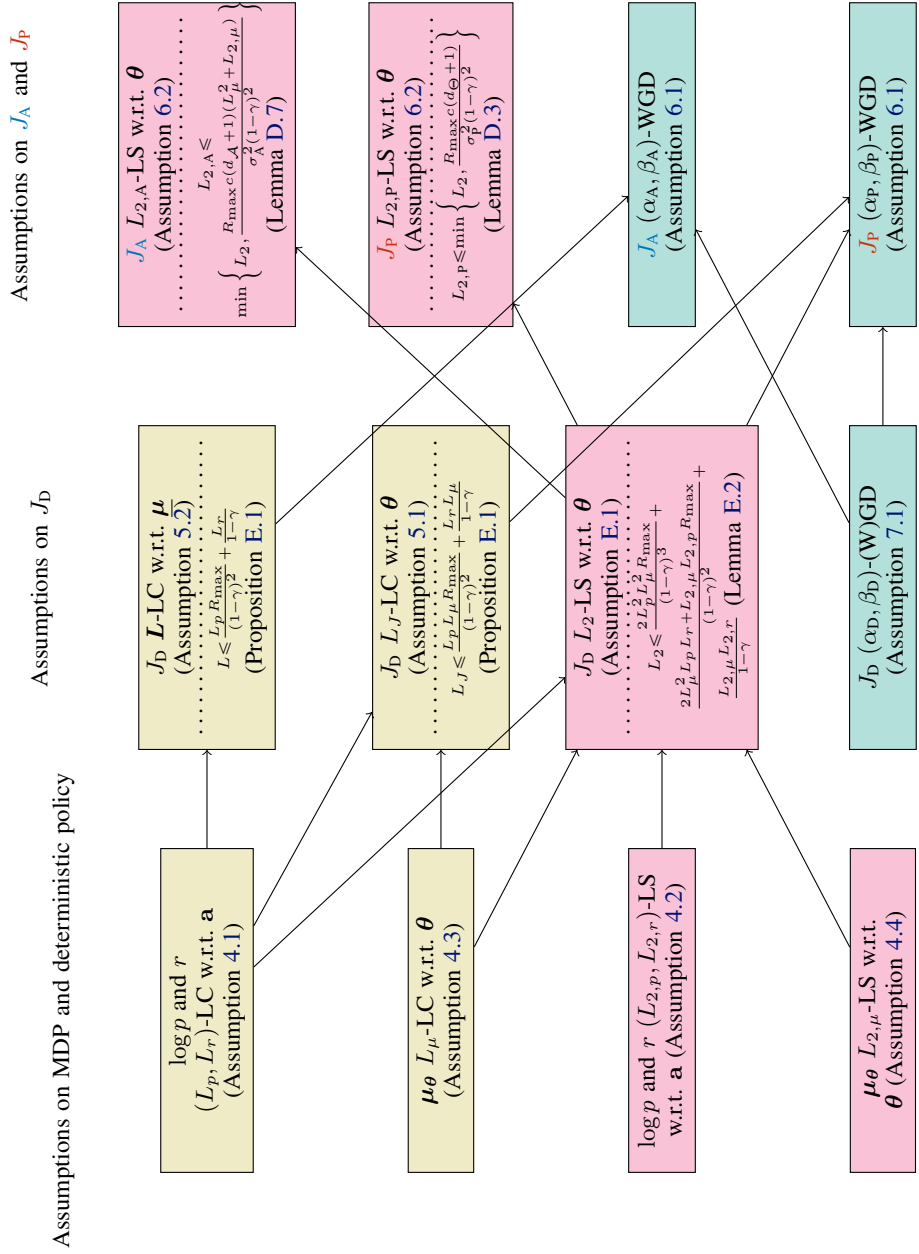
\begin{figure*}[t]
    \centering
    \resizebox{.8\linewidth}{!}{
    \small 
    \begin{tikzpicture}[text width=6.5cm]
        \node[block, fill=brightYellow!30] (a) {$\log p$ and $r$ $(L_p,L_r)$-LC \wrt $\ba$      \\(Assumption~\ref{ass:lipMdp})};
        \node[block, below=of a, fill=brightYellow!30] (b) {$\bmu$ $L_{{\mu}}$-LC \wrt $\vtheta$ (Assumption~\ref{ass:lipPol})};
        \node[block, below=of b, fill=vibrantMagenta!30] (c) {$\log p$ and $r$ $(L_{2,p},L_{2,r})$-LS \wrt $\ba$ (Assumption~\ref{ass:smoothMdp})};
        \node[block, below=of c, fill=vibrantMagenta!30] (d) {$\bmu$ $L_{2,{\mu}}$-LS \wrt $\vtheta$ (Assumption~\ref{ass:smoothPol})};
        \node[above=of a, align=center] (x) {Assumptions on MDP and deterministic policy};

        \node[block, right of=a, node distance=5.5cm, fill=brightYellow!30] (e) {$\Jd$ $\bm{L}$-LC \wrt $\underline{\bm{\mu}}$ (Assumption~\ref{ass:Jd_lip_ns}) 
        
        \vspace{-.2cm}\par\noindent\dotfill
        
        $\scriptstyle L \le \frac{L_p R_{\max}}{(1-\gamma)^2} + \frac{ L_r}{1 - \gamma} $ (Proposition~\ref{thr:assImpl})
        };
         \node[block, right of=b, node distance=5.5cm, , fill=brightYellow!30] (f) {$\Jd$ $L_{J}$-LC \wrt $\vtheta$ (Assumption~\ref{ass:Jd_lip})\vspace{-.2cm}\par\noindent\dotfill
        
        $\scriptstyle L_J \le \frac{L_p L_\mu R_{\max}}{(1-\gamma)^2} + \frac{ L_rL_\mu}{1 - \gamma} $ (Proposition~\ref{thr:assImpl})};
          \node[block, right of=c, node distance=5.5cm, fill=vibrantMagenta!30] (g) {$\Jd$ $L_2$-LS \wrt $\vtheta$ 
          (Assumption~\ref{ass:smoothJ})
          
        \vspace{-.2cm}\par\noindent\dotfill
        
        $\scriptstyle L_{2} \le  \frac{2L_p^2 L_\mu^2 R_{\max}}{(1-\gamma)^3} + \frac{2L_\mu^2L_p L_r+ L_{2,\mu}L_{2,p}R_{\max} }{(1-\gamma)^2} + \frac{L_{2,\mu}L_{2,r}}{1-\gamma}   $
        (Lemma~\ref{lem:L_2_characterization})
          };
        
        \node[block, right of=d, node distance=5.5cm, fill=vibrantTeal!30] (j) {$\Jd$ $(\alpha_{\text{D}},\beta_{\text{D}})$-(W)GD (Assumption~\ref{ass:J_wgdD})};

         \node[above=of e, align=center] (y) {Assumptions on $\Jd$};

          \node[block, right of=g, node distance=5.5cm, fill=vibrantTeal!30] (k) {$\Ja$  $(\alpha_{\text{A}},\beta_{\text{A}})$-WGD (Assumption~\ref{ass:J_wgd})};
         \node[block, right of=j, node distance=5.5cm, fill=vibrantTeal!30] (l) {$\Jp$  $(\alpha_{\text{P}},\beta_{\text{P}})$-WGD (Assumption~\ref{ass:J_wgd})};
          \node[block, right of=e, node distance=5.5cm, fill=vibrantMagenta!30] (m) {$\Ja$ $L_{2,\text{A}}$-LS \wrt $\vtheta$ (Assumption~\ref{ass:J_gen_conv})
          
        \vspace{-.2cm}\par\noindent\dotfill
        
        $\scriptstyle L_{2,\text{A}} \le \min\Big\{L_2, \frac{R_{\max}c (\da+1)(L_\mu^2+L_{2,\mu})}{\sigma^2_\text{A}(1-\gamma)^2}\Big\} $ (Lemma~\ref{lem:Ja_bounded_hessian})};
        \node[block, right of=f, node distance=5.5cm, fill=vibrantMagenta!30] (n) {$\Jp$ $L_{2,\text{P}}$-LS \wrt $\vtheta$ (Assumption~\ref{ass:J_gen_conv})
        
        \vspace{-.2cm}\par\noindent\dotfill
        
        $\scriptstyle L_{2,\text{P}} \le \min\Big\{L_2, \frac{R_{\max}c (\dt+1)}{\sigma^2_\text{P}(1-\gamma)^2}\Big\} $ (Lemma~\ref{lem:Jp_bounded_hessian})};

         \node[above=of m, align=center] (x) {Assumptions on $\Ja$ and $\Jp$};

  
         
        

         
        

        
        \draw [->] (a) -- (e);
        \draw [->] (a) -- (g);
        \draw [->] (b) -- (f);
        \draw [->] (a) -- (f);
        \draw [->] (b) -- (g);
        \draw [->] (c) -- (g);
        \draw [->] (d) -- (g);
        \draw [->] (j) -- (k);
        \draw [->] (j) -- (l);
        \draw [->] (e) -- (k);
        \draw [->] (f) -- (l);
        \draw [->] (g) -- (l);
        \draw [->] (g) -- (m);
        \draw [->] (g) -- (n);


\end{tikzpicture}
}
    \caption{Dependency of some of the assumptions presented in the paper.}
    \label{fig:assum}
\end{figure*}

\end{landscape}

\section{Additional Related Works}\label{apx:addRel}

\noindent\textbf{Policy variance.}~~When optimizing Gaussian policies with policy-gradient methods, the scale parameters (those of the variance or, more in general, of the covariance matrix of the policy) are typically fixed in theory, and optimized via gradient descent in practice. To the best of our knowledge, there is no satisfying theory of the effects of a varying policy (or hyperpolicy) variance on the convergence rates of PG (or PGPE).
~\citet{ahmed2019understanding} were the first to take into serious consideration the impact of the policy stochasticity on the geometry of the objective function, although their focus was on entropy regularization. 
~\citet{papini2020balancing}, focusing on monotonic improvement rather than convergence, proposed to use second-order information to overcome the greediness of gradient updates, arguing that the latter is particularly harmful for scale parameters.~\citet{bolland2023policy} propose to study PG with Gaussian policies under the lens of \emph{optimization by continuation}~\citep{allgower1990numerical}, that is, as a sequence of smoothed version of the deterministic policy optimization problem. Unfortunately, the theory of optimization by continuation is rather scarce. We studied the impact of a \emph{fixed} policy variance on the number of samples needed to find a good \emph{deterministic} policy. We hope that this can provide some insight on how to design adaptive policy-variance strategies in future work. We remark here that the common practice of \emph{learning} the exploration parameters together with all the other policy parameters breaks all of the known convergence results of GPOMDP, since the smoothness of the stochastic objective is inversely proportional to the policy variance~\citep{papini2022smoothing}. In this regard, entropy-regularized policy optimization is different, and is better studied using mirror descent theory, rather than stochastic gradient descent theory~\citep{shani2020adaptive}.

\noindent\textbf{Comparing AB and PB exploration.}~~A classic on the topic is the paper by~\citet{zhao2011analysis}. They prove upper bounds on the variance of the REINFORCE and PGPE estimators, highlighting the better dependence on the task horizon of the latter. 
The idea that variance reduction does not tell the whole story about the efficiency of policy gradient methods is rather recent~\citep{ahmed2019understanding}. We revisited the comparison of action-based and parameter based methods under the lens of modern sample complexity theory. We reached similar conclusions but achieved, we believe, a more complete understanding of the matter.
To our knowledge, the only other work that thoroughly compares AB and PB exploration is~\citep{metelli2018policy,MetelliPMR20,MetelliPDR21}, where the trade-off between the task horizon and the number of policy parameters is discussed both in theory and experiments, but in the context of trust-region methods.
\clearpage
\section{Additional Considerations}\label{apx:mapping}
We only considered (hyper)policy variances $\asigma^2,\psigma^2$ that are \emph{fixed} for the duration of the learning process, albeit they can be set as functions of problem-dependent constants and of the desired accuracy $\epsilon$. This is due to our focus on convergence guarantees based on smooth optimization theory, as explained in the following.

\begin{remark}[About learning the (hyper)policy variance]
    It is a well established practice to parametrize the policy variance and learn these exploration parameters via gradient descent together with all the other policy parameters~\citep[again, for examples, see ][]{duan2016benchmarking,stable-baselines3}. The same is true for parameter-based exploration~\citep{schwefel1993evolution,SEHNKE2010551}.
    However, it is easy to see that an adaptive (in the sense of time-varying) policy variance breaks the sample complexity guarantees of GPOMDP~\citep{yuan2022general} and its variance-reduced variants~\citep[\eg][]{liu2020improved}. That is because these guarantees all rely on Assumption~\ref{ass:J_gen_conv}, or equivalent smoothness conditions, and obtain sample complexity upper bounds that scale with the smoothness constant $L_{2,A}$. However, the latter can depend \emph{inversely} on $\asigma^2$, as already observed by~\citet{papini2022smoothing} for Gaussian policies. Thus, unconstrained learning of $\sigma_A$ breaks the convergence guarantees. Analogous considerations hold for PGPE with adaptive hyperpolicy variance.
    Different considerations apply to entropy-regularized policy optimization methods, which were not considered in this paper, mostly because they converge to a surrogate objective that is even further from optimal deterministic performance. These methods are better analyzed using the theory of mirror descent. We refer the reader to~\citep{shani2020adaptive}. 
\end{remark}

In order to properly define the white noise-based (hyper)policies, we need that $\bm{\mu}_{\vtheta}(\vs)+\vepsilon \in \As$ (for AB exploration) and $\vtheta+\vepsilon \in \Theta$ (for PB exploration), we will assume that $\As = \Reals^{\da}$ and $\Theta=\Reals^{\dt}$ for simplicity.

\begin{remark}[About  $\As = \Reals^{\da}$ and $\Theta=\Reals^{\dt}$ assumption]
    We have assumed that the action space $\mathcal{A}$ and the parameter space $\Theta$ correspond to $\Reals^{\da}$ and $\Reals^{\dt}$, respectively. If this is not the case, we can easily alter the transition model $p$ and the reward function $r$ (for the AB exploration), and the deterministic policy $\mu_{\vtheta}$ (for the PB exploration) by means of a \emph{retraction} function. Let $\mathcal{X}\subseteq \Reals^d$ be a measurable set, a retraction function $\iota_{\mathcal{X}}: \Reals^d \rightarrow \mathcal{X}$ is such that $\iota_{\mathcal{X}}(x) = x$ if $x \in \mathcal{X}$, \ie it is the identity over $\mathcal{X}$.
    \begin{itemize}[noitemsep, leftmargin=*, topsep=-2pt]
        \item For the AB exploration, we redefine the transition model as $\overline{p}(\vs'|\vs,\mathbf{a}) \coloneqq p(\vs'|\vs,\bm{\iota}_{\mathcal{A}}(\mathbf{a}))$ for every $\vs,\vs' \in \Ss$ and $\va \in \As$. Furthermore, we redefine the reward function as $\overline{r}(\vs, \va) \coloneqq r(\vs, \iota_{\As}(\va))$ for every $\vs \in \Ss$ and $\va \in \As$.
        \item For the PB exploration, we redefine the deterministic policy as $\overline{\bm{\mu}}_{\vtheta}(\vs) \coloneqq \bm{\mu}_{\bm{\iota}_{\Theta}(\vtheta)}$, for every $\vtheta \in \Theta$.
    \end{itemize}
\end{remark}

\clearpage
\section{Proofs}\label{apx:proofs}

\subsection{Proofs from Section~\ref{sec:deploy}}
\begin{lemma}\label{lemma:lb}
Let $L > 0$, consider the function $f : \Reals \rightarrow \Reals$ defined for every $x \in \Reals$ as follows:
\begin{align}
	f(x) = \begin{cases}
	0 & \text{if } x < -1/L \text{ or } x > 2/L \\
	Lx +1 & \text{if } -1/L \le x < 0 \\
	1-\frac{L}{2}x & \text{if } 0\le x \le 2/L
	\end{cases}.
\end{align}
Consider the function $\psi_{\sigma} : \Reals \rightarrow \mathbb{R}_{\ge 0}$ defined for every $x \in \Reals$ as follows:
\begin{align}
    \psi_{\sigma}(x) = \begin{cases}
        \frac{1}{2\sqrt{3}\sigma} & \text{if } -\sqrt{3}\sigma \le x \le \sqrt{3}\sigma \\
        0 & \text{otherwise}
    \end{cases},
\end{align}
\ie the p.d.f. of a uniform distribution with zero mean and variance $\sigma^2$. Let $f_\sigma \coloneqq f * \psi_\sigma$, let $x^* = \argmax_{x \in \Reals }f(x)$, and let $x^*_\sigma = \argmax_{x \in \Reals} f_\sigma(x)$. Then $f$ is $L$-LC and, if $\sqrt{3}\sigma \le 1/L$, it holds that $f(x^*) - f(x^*_\sigma) = L\sigma /(2\sqrt{3})$.
\end{lemma}

\begin{proof}

    Let us first verify that the distribution whose p.d.f. is $\phi_\sigma$ has zero mean and variance $\sigma^2$:
    \begin{align}
        & \int_{\Reals} \psi_\sigma(x) x \de x = 0, \\
        & \int_{\Reals} \psi_\sigma(x) x^2 \de x = 2 \int_{0}^{\sqrt{3}\sigma} \psi_\sigma(x) x^2 \de x = \sigma^2.
    \end{align}
Under the assumption $\sqrt{3}\sigma \le 1/L$, functions $f$ and $\psi_\sigma$ can be represented as follows:  
    \begin{figure}[H]
        \centering
        \includegraphics[width=0.5\linewidth,viewport=0 0 248.97 117.93,clip]{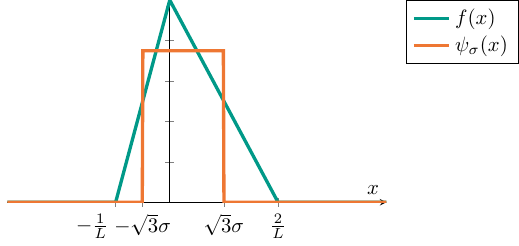}
    \end{figure}
    Let us now compute the convolution:
    \begin{align}
        f_\sigma(x) = f * \psi_\sigma = \int_{\Reals} \psi_\sigma(x-t) f(t) \de t.
    \end{align}
    It is clear that the global optimum of function $f_\sigma$ is located in the interval given by $|x| \le 1/L$. This combined, with  the assumption $\sqrt{3}\sigma \le 1/L$, allows to simplify the integral as:
    \begin{align}
        \int_{\Reals} \psi_\sigma(x-t) f(t) \de t & = \int_{x-\sqrt{3}\sigma}^0 \frac{1}{2\sqrt{3}\sigma}(Lt+1) \de t + \int_0^{x+\sqrt{3}\sigma} \frac{1}{2\sqrt{3}\sigma}\left(1-\frac{L}{2}t\right) \de t  \\
        &  1 - =\frac{1}{2\sqrt{3}\sigma} \left(\frac{L}{2}(x-\sqrt{3}\sigma)^2 + \frac{L}{4} (x+\sqrt{3}\sigma)^2 \right).
    \end{align}
    The latter is a concave (quadratic) function of $x$, which is maximized for $x^*_\sigma = \sigma/\sqrt{3}$. Noticing that $x^* = \argmax_{x \in \Reals} f(x)$=0, we have:
    \begin{align}
        f(x^*) - f(x^*_\sigma) = f(0) - f(\sigma/\sqrt{3}) = 1 - \left(1 - \frac{L\sigma}{2\sqrt{3}}\right) = \frac{L\sigma}{2\sqrt{3}}.
    \end{align}
\end{proof}

\deployPB*

\begin{proof}
    Before starting the derivation, we remark that:
	\begin{align}
		\Jp(\vtheta) = \E_{\vepsilon \sim \Phi_{\dt}} \left[\Jd(\vtheta + \vepsilon)\right],
	\end{align}
	where $\E_{\vepsilon \sim \Phi_{\dt}}[\|\vepsilon\|_2^2]\le {\dt} \sigma^2_{\text{P}}$.	
	From Assumption~\ref{ass:Jd_lip}, we can easily derive ($i$):
	\begin{align}
		|\Jd(\vtheta) - \Jp(\vtheta) | & = |\Jd(\vtheta) - 
 \E_{\vepsilon \sim \Phi_{\dt}} \left[\Jd(\vtheta + \vepsilon) \right]|  \\
 & \le \E_{\vepsilon \sim \Phi_{\dt}} \left[ |\Jd(\vtheta) - \Jd(\vtheta + \vepsilon)| \right] \\
 & \le L_J \E_{\vepsilon \sim \Phi_{\dt}} \left[ \|\vepsilon\|_2 \right] \\
 & \le L_J \sqrt{\E_{\vepsilon \sim \Phi_{\dt}} \left[ \|\vepsilon\|_2^2 \right]} \\
 & \le L_J\psigma \sqrt{\dt}.
 \end{align}
 For ($ii$), let $\vtheta^* \in \argmax_{\vtheta \in \Theta} \Jd(\vtheta)$, we have:
 \begin{align}
     \max_{\vtheta \in \Theta} \Jd(\vtheta) - \Jd(\vtheta^*_{\text{P}}) & = \Jd(\vtheta^*) - \Jd(\vtheta^*_{\text{P}}) \pm \Jp(\vtheta^*)\\
     &\leq \Jd(\vtheta^*) - \Jp(\vtheta^*) + \Jp(\vtheta^*_{\text{P}}) - \Jd(\vtheta_{\text{P}}^*) \label{eq:pgpe_global_performance_distance_eq_1}\\
        &\leq 2 \max_{\vtheta \in \Theta} \left| \Jd(\vtheta) - \Jp(\vtheta) \right| \label{eq:pgpe_global_performance_distance_eq_2}\\
        &\leq 2 L_J \psigma \sqrt{\dt},
    \end{align}
	where line~(\ref{eq:pgpe_global_performance_distance_eq_1}) follows from $\Jp(\vtheta^*_{\text{P}}) =\max_{\vtheta \in \Theta} \Jp(\vtheta)\geq \Jp(\vtheta^*)$, and line~(\ref{eq:pgpe_global_performance_distance_eq_2}), follows by applying twice result ($i$). 
 
    To prove ($iii$) we construct the MDP $(\{s\},\Reals^{\dt},p,r,\rho_0,\gamma)$ (\ie a bandit), where $r(s,\ba) = \frac{1}{\dt} \sum_{i=1}^{\dt} f(a_i)$, where $f$ is defined in Lemma~\ref{lemma:lb} and $\bmu(s) = \vtheta$ with $\vtheta \in \Reals^{\dt}$. Thus, we can compute the expected return as follows:
 \begin{align}
     \Jd(\vtheta) = \frac{1-\gamma^T}{1-\gamma}\cdot \frac{1}{\dt} \sum_{i=1}^{\dt} f(\theta_i).
 \end{align}
Let us compute its Lipschitz constant recalling that $f$ is $L$-LC thanks to Lemma~\ref{lemma:lb}. In particular, we take $\vtheta = \mathbf{0}_{\dt}$ and $\vtheta' = - \eta \mathbf{1}_{\dt}$ with $\eta \in (0,1/L)$, recalling that $\|\vtheta-\vtheta'\|_2 = \eta \sqrt{\dt} $ and that $f(\theta_i) = 1$ and $f(\theta_i') = -\eta L + 1$, we have:
\begin{align}
    | \Jd(\vtheta)- \Jd(\vtheta')| & =  \left| \frac{1-\gamma^T}{1-\gamma}\cdot \frac{1}{\dt} \sum_{i=1}^{\dt} f(\theta_i) - \frac{1-\gamma^T}{1-\gamma}\cdot \frac{1}{\dt} \sum_{i=1}^{\dt} f(\theta_i')\right| \\
    & = \frac{1-\gamma^T}{1-\gamma}\cdot \frac{1}{\dt} \sum_{i=1}^{\dt}|f(\theta_i)-f(\theta_i')| \\
    & = \frac{1-\gamma^T}{1-\gamma}\cdot \frac{L}{\dt} \sum_{i=1}^{\dt} |\theta_i - \theta_i'| \\
    & = \frac{1-\gamma^T}{1-\gamma} L \eta \\
    & = \frac{1-\gamma^T}{1-\gamma}\cdot \frac{L}{\sqrt{\dt}} \| \vtheta-\vtheta'\|_2.
\end{align}
Thus, we have that $\Jd(\vtheta)$ is $\left( \frac{1-\gamma^T}{1-\gamma}\cdot \frac{L}{\sqrt{\dt}} \right)$-LC. By naming $L_J =  \frac{1-\gamma^T}{1-\gamma}\cdot \frac{L}{\sqrt{\dt}}$, we have $L = \frac{1-\gamma}{1-\gamma^T} \sqrt{\dt} L_J$. We now consider the additive noise $\Phi_{\dt} =  \otimes_{i=1}^{\dt} \text{Uni}([-\sqrt{3}\sigma,\sqrt{3}\sigma])$, \ie the $\dt$-dimensional uniform distribution with independent components over the hypercube $[-\sqrt{3}\sigma,\sqrt{3}\sigma]^{\dt}$. From Lemma~\ref{lemma:lb}, we know that each dimension has variance $\sigma^2$, consequently:
 \begin{align}
     \E_{\vepsilon \sim \Phi_{\dt}}[\|\vepsilon\|_2^2] = \sum_{i=1}^{\dt} \E_{\epsilon_i \sim \text{Uni}([-\sqrt{3}\sigma,\sqrt{3}\sigma])}[\epsilon_i^2] = \dt \sigma^2,
 \end{align}
 thus complying with Definition~\ref{defi:add}. Consequently:
 \begin{align}
     \Jp(\vtheta) = \E_{\vepsilon \sim \Phi_{\dt}}[\Jd(\vtheta+\vepsilon)] = \sum_{i=1}^{\dt} \E_{\epsilon_i \sim \text{Uni}([-\sqrt{3}\sigma,\sqrt{3}\sigma])}[f(\theta_i+\epsilon_i)] =  \sum_{i=1}^{\dt} (f*\psi_\sigma)(\theta_i),
 \end{align}
 where $\psi_\sigma$ is the p.d.f. of the considered uniform distribution as defined in Lemma~\ref{lemma:lb}. From Lemma~\ref{lemma:lb} and observing that both $\Jd$ and $\Jp$ decompose into a sum over the $\dt$ dimensions, we have for $\sqrt{3}\sigma < 1/L$:
 \begin{align}
      \vtheta^*= \argmax_{\vtheta \in \Reals^{\dt}} \Jd(\vtheta) = \mathbf{0}_{\dt}, \qquad \vtheta^*_{\text{P}} = \argmax_{\vtheta \in \Reals^{\dt}} \Jp(\vtheta) = \frac{\sigma}{\sqrt{3}} \mathbf{1}_{\dt}.
 \end{align}
 It follows that:
 \begin{align}
     \Jd(\vtheta^*) - \Jd(\vtheta^*_{\text{P}}) & = \Jd(\mathbf{0}_{\dt}) - \Jd\left( \frac{\sigma}{\sqrt{3}} \mathbf{1}_{\dt}\right) \\
     & =\frac{1-\gamma^T}{1-\gamma} \frac{1}{\dt} \sum_{i=1}^{\dt} f(0) - f(\sigma/\sqrt{3}) \\
     & =\frac{1-\gamma^T}{1-\gamma}\frac{L\sigma}{2\sqrt{3}} \\
     & = \frac{1}{ 2\sqrt{3} } L_J  \sqrt{\dt} \sigma.
 \end{align}
\end{proof}

\deployAB*

\begin{proof}
    From Assumption~\ref{ass:Jd_lip_ns}, noting that $\Jd(\vtheta) = \Jd(\underline{\bm{\mu}}_{\vtheta})$ we can easily derive ($i$):
    \begin{align}
        \left|\Jd(\vtheta)-\Ja(\vtheta)\right|  & = \left| \Jd(\vtheta)- \E_{\underline{\vepsilon} \sim \Phi_{\da}^T }\left[ {\Jd}(\underline{\bm{\mu}}_{\vtheta} + \underline{\vepsilon}) \right] \right| \\
        & = \left| \Jd(\underline{\bm{\mu}}_{\vtheta} )- \E_{\underline{\vepsilon} \sim \Phi_{\da}^T }\left[ {\Jd}(\underline{\bm{\mu}}_{\vtheta} + \underline{\vepsilon}) \right] \right| \\
        & \le \E_{\underline{\vepsilon} \sim \Phi_{\da}^T } \left[ \sum_{t=0}^{T-1} L_t \sup_{s_t \in \mathcal{S}} \left\| \bmu(s_t) - ( \bmu(s_t) + \vepsilon_t) \right\|_2 \right] \\
        & = \sum_{t=0}^{T-1} L_t  \E_{\vepsilon\sim \Phi_{\da}}[\|\vepsilon\|_2^2] \\
        & \le L \sqrt{{\da}} \sigma_{\text{A}}.
    \end{align}
   For ($ii$), let $\vtheta^* \in \argmax_{\vtheta \in \Theta} \Jd(\vtheta)$, we have:
 \begin{align}
     \max_{\vtheta \in \Theta} \Jd(\vtheta) - \Jd(\vtheta^*_{\text{A}}) & = \Jd(\vtheta^*) - \Jd(\vtheta^*_{\text{A}}) \pm \Ja(\vtheta^*)\\
     &\leq \Jd(\vtheta^*) - \Ja(\vtheta^*) + \Ja(\vtheta^*_{\text{A}}) - \Jd(\vtheta_{\text{A}}^*) \label{eq:pgpe_global_performance_distance_eq_12}\\
        &\leq 2 \max_{\vtheta \in \Theta} \left| \Jd(\vtheta) - \Ja(\vtheta) \right| \label{eq:pgpe_global_performance_distance_eq_22}\\
        &\leq 2 L \asigma \sqrt{\da},
    \end{align}
	where line~(\ref{eq:pgpe_global_performance_distance_eq_12}) follows from $\Ja(\vtheta^*_{\text{A}}) =\max_{\vtheta \in \Theta} \Ja(\vtheta)\geq \Ja(\vtheta^*)$, and line~(\ref{eq:pgpe_global_performance_distance_eq_22}) follows by applying twice result ($i$). The proof of $(iii)$ is identical to that of Theorem~\ref{thr:deployPB} since, for the particular instance, we have enforced $\bmu(s) = \vtheta$ (which implies $\da=\dt$) and, thus, AB exploration is equivalent to PB exploration.
\end{proof}
\subsection{Proofs from Section~\ref{sec:convergence}}
\begin{lemma}[Variance of $\widehat{\nabla}_{\vtheta} \Jp(\vtheta)$ bounded] \label{ass:Jp_estim_bounded_variance}
    Under Assumption~\ref{ass:magic}, the variance the PGPE estimator with batch size $N$ is bounded for every $ \vtheta \in \Theta$ as:
    \begin{align*}
        \Var\left[ \widehat{\nabla}_{\vtheta} \Jp(\vtheta) \right] \leq \frac{R_{\max}^2 \xi_2 (1-\gamma^T)^2}{N (1-\gamma)^2}\le \frac{R_{\max}^2 \xi_{1}^2}{N (1-\gamma)^2}.
    \end{align*}
    with $\xi_2 \le  c \dt \sigma_\text{P}^{-2}$.
\end{lemma} 
\begin{proof}
    We recall that the estimator employed by PGPE in its update rule is:
    \begin{align*}
        \widehat{\nabla}_{\vtheta} \Jp(\vtheta) = \frac{1}{N} \sum_{i=1}^{N} \nabla_{\vtheta} \log \nu_{\vtheta}(\vtheta_{i}) R(\tau_{i}),
    \end{align*}
    where $N$ is the number of parameter configuration tested (on one trajectory) at each iteration.
    Thus, we can compute the variance of such an estimator as:
    \begin{align*}
        \Var_{\vtheta' \sim \nu_{\vtheta}} \left[ \widehat{\nabla}_{\vtheta} \Jp(\vtheta') \right] &= \frac{1}{N} \Var_{\vtheta' \sim \nu_{\vtheta}} \left[ \nabla_{\vtheta} \log \nu_{\vtheta}(\vtheta') R(\tau_{1}) \right] \\
        &= \frac{1}{N} \E_{\vtheta' \sim \nu_{\vtheta}} \left[ \left\| \nabla_{\vtheta}\log \nu_{\vtheta}(\vtheta') \right\|^{2}_2 R(\tau_{1})^{2} \right] \\
        &\leq \frac{R_{\max}^2 \xi_{1}^2 (1 - \gamma^T)^2}{N (1-\gamma)^2},
    \end{align*}
    where the last line follows form Assumption~\ref{ass:magic} and Lemma~\ref{lemma:boundsMagic2} after having defined $\xi_2 = \E_{\vtheta' \sim \nu_{\vtheta}} \left[ \left\| \nabla_{\vtheta}\log \nu_{\vtheta}(\vtheta') \right\|^{2}_2 \right]$ and from the fact that, given a trajectory $\tau$, $R(\tau)$ is defined as:
    \begin{align*}
        R(\tau) = \sum_{t=0}^{T-1} \gamma^{t} r(s_{\tau, t}, \va_{\tau, t}),
    \end{align*}
with $r(\vs, \va) \in [-R_{\max}, R_{\max}]$ for every $\vs \in \mathcal{S}$ and $\va \in \mathcal{A}$.
\end{proof}

\begin{lemma}[Bounded $\Jp$ Hessian] \label{lem:Jp_bounded_hessian}
	Under Assumption~\ref{ass:magic} and using a hyperpolicy complying with Definition~\ref{defi:add}, $\forall \vtheta \in \Theta$ it holds that:
	\begin{align*}
		\left\| \nabla^2_{\vtheta} \Jp(\vtheta)\right\|_2 \leq L_{2,\text{P}} \frac{R_{\max} (1 - \gamma^T)}{1 - \gamma} \left(\xi_2 + \xi_3\right),
	\end{align*}
 where $\xi_2 \le c \dt\sigma_{\text{P}}^{-2}$ and $\xi_3 \le c\sigma_{\text{P}}^{-2}$. Furthermore, under Assumptions~\ref{ass:lipMdp}, \ref{ass:lipPol}, \ref{ass:smoothMdp}, and \ref{ass:smoothPol}, and using a hyperpolicy complying with Definition~\ref{defi:add}, $\forall \vtheta \in \Theta$ it holds that:
	\begin{align*}
		\left\| \nabla^2_{\vtheta} \Jp(\vtheta)\right\|_2 \leq L_2,
	\end{align*}
\end{lemma}
where $L_2$ is bounded as in Lemma~\ref{lem:L_2_characterization}.
\begin{proof}
	The performance index $\Jp$ of a hyperpolicy $\nu_{\vtheta}$ can be seen as the expectation over the sampling of a parameter configuration $\vtheta'$ from the hyperpolicy $\nu_{\vtheta}$, or as the perturbation according to the realization $\bm{\vepsilon}$ of a sub-gaussian noise $\psigma$ of the parameter configuration of the deterministic policy $\bm{\mu}_{\vtheta}$.
	
	Using the first characterization we can write:
	\begin{align}
		\Jp(\vtheta) = \E_{\vtheta' \sim \nu_{\vtheta}} \left[ \Jd(\vtheta') \right]. \label{eq:pgpe_def_1}
	\end{align}
	
	Equivalently, we can write:
	\begin{align}
		\Jp(\vtheta) = \E_{\bm{\vepsilon} \sim \Phi} \left[ \Jd(\vtheta + \bm{\vepsilon}) \right]. \label{eq:pgpe_def_2}
	\end{align}
	
	By using the latter, we have that:
	\begin{align}
		\left\| \nabla^2_{\vtheta} \Jp(\vtheta) \right\|_2 &= \left\| \nabla^2_{\vtheta} \E_{\bm{\vepsilon} \sim g} \left[ \Jd(\vtheta + \bm{\vepsilon}) \right] \right\|_2 \nonumber\\
		&= \left\|  \E_{\bm{\vepsilon} \sim g} \left[ \nabla^2_{\vtheta} \Jd(\vtheta + \bm{\vepsilon}) \right] \right\|_2 \nonumber\\
		&\leq \E_{\bm{\vepsilon} \sim g} \left[ \left\| \nabla^2_{\vtheta} \Jd(\vtheta + \bm{\vepsilon}) \right\|_2 \right] \nonumber\\
		&\leq L_2, \label{eq:J_pb_hessian_bounded:1}
	\end{align}
	where the last inequality simply follows from Assumption~\ref{ass:smoothJ}.
 
	By using Equation~\eqref{eq:pgpe_def_1}, instead, we have the following:
	\begin{align*}
		\nabla^2_{\vtheta} \Jp(\vtheta) &= \nabla^2_{\vtheta} \E_{\vtheta' \sim \nu_{\vtheta}} \left[ \Jd(\vtheta') \right] \\
		&= \int \nabla^2_{\vtheta} \left(\nu_{\vtheta}(\vtheta') \Jd(\vtheta') \right) \text{d} \vtheta' \\
		&= \int \nabla_{\vtheta} \left( \nabla_{\vtheta} \nu_{\vtheta}(\vtheta') \Jd(\vtheta') + \nu_{\vtheta}(\vtheta') \nabla_{\vtheta} \Jd (\vtheta') \right) \text{d} \vtheta' \\
		&= \int \nabla_{\vtheta} \left( \nu_{\vtheta}(\vtheta') \left(\nabla_{\vtheta} \log \nu_{\vtheta} (\vtheta') \Jd(\vtheta')\right)\right) \text{d} \vtheta' \\
		&= \int  \nabla_{\vtheta} \nu_{\vtheta}(\vtheta') \nabla_{\vtheta} \log \nu_{\vtheta} (\vtheta') \Jd(\vtheta') + \nu_{\vtheta}(\vtheta') \bigg( \nabla_{\vtheta}^2 \log \nu_{\vtheta}(\vtheta') \Jd(\vtheta') + \nabla_{\vtheta} \log \nu_{\vtheta}(\vtheta') \nabla_{\vtheta} \Jd(\vtheta') \bigg)  \text{d} \vtheta' \\
		&= \int  \nu_{\vtheta}(\vtheta') \bigg( \nabla_{\vtheta} \log \nu_{\vtheta}(\vtheta')\nabla_{\vtheta} \log \nu_{\vtheta}(\vtheta')^\top \Jd(\vtheta') +  \nabla_{\vtheta}^2 \log \nu_{\vtheta}(\vtheta') \Jd(\vtheta') \bigg) \text{d} \vtheta' \\
		&= \E_{\vtheta' \sim \nu_{\vtheta}} \bigg[ \bigg( \nabla_{\vtheta} \log \nu_{\vtheta}(\vtheta')\nabla_{\vtheta} \log \nu_{\vtheta}(\vtheta')^\top + \nabla_{\vtheta}^2 \log \nu_{\vtheta}(\vtheta') \bigg)\Jd(\vtheta') \bigg].
	\end{align*}
	
	Now, given the previous argument, it follows that:
	\begin{align}
		\left\| \nabla^2_{\vtheta} \Jp(\vtheta) \right\|_2 &=  \left\| \E_{\vtheta' \sim \nu_{\vtheta}} \bigg[ \bigg( \nabla_{\vtheta} \log \nu_{\vtheta}(\vtheta')\nabla_{\vtheta} \log \nu_{\vtheta}(\vtheta')^\top + \nabla_{\vtheta}^2 \log \nu_{\vtheta}(\vtheta') \bigg)\Jd(\vtheta') \bigg] \right\|_2 \nonumber\\
		&\leq \E_{\vtheta' \sim \nu_{\vtheta}} \left[ \left\| \nabla_{\vtheta} \log \nu_{\vtheta}(\vtheta') \right\|_2^2 \left| \Jd(\vtheta') \right| + \left\| \nabla_{\vtheta}^2 \log \nu_{\vtheta}(\vtheta') \right\|_2 \left| \Jd(\vtheta') \right| \right] \nonumber\\
		&\leq \frac{R_{\max} (1 - \gamma^T)}{1 - \gamma} \left(\xi_2 + \xi_3\right).\label{eq:J_pb_hessian_bounded:2}
	\end{align}
	We employ Lemma~\ref{lemma:boundsMagic2} to bound $\xi_2$ and $\xi_3$.
\end{proof}

\begin{restatable}[\textbf{\textcolor{vibrantRed}{Global convergence of PGPE - Fixed $\sigma_{\text{P}}$}}]{thr}{pgpeSamCompWGD} \label{thr:pgpe_sam_comp_wgd}
     Under Assumptions~\ref{ass:J_wgd} (with $J_{\simbolo}=\Jp$), \ref{ass:lipMdp}, \ref{ass:lipPol}, \ref{ass:magic}, with a suitable constant step size, to guarantee $\Jd^* - \E[\Jd(\vtheta_K)] \le \epsilon + \beta + 3L_P\sqrt{\dt} \sigma_P$, where $3L_P\sqrt{\dt} \sigma_P = O(\sqrt{\dt}\sigma_\text{P}(1-\gamma)^{-2})$ the sample complexity of PGPE is at most:
    \begin{align}
        NK = \widetilde{O} \left( \frac{\alpha^4 \dt^2}{\sigma_{\text{P}}^4(1-\gamma)^4\epsilon^3}\right).
    \end{align}
    Furthermore, under Assumptions~\ref{ass:smoothMdp} and~\ref{ass:smoothPol}, the same guarantee is obtained with a sample complexity at most:
    \begin{align}
        NK = \widetilde{O} \left( \frac{\alpha^4\dt}{\sigma_{\text{P}}^2(1-\gamma)^5\epsilon^3}\right).
    \end{align}
 \end{restatable}
\begin{proof}
We first apply Theorem~\ref{thr:gen_conv_new} with $J_\simbolo = \Jp$, recalling that the  assumptions enforced in the statement entail those of Theorem~\ref{thr:gen_conv_new}:
\begin{align}
    \Jp^* - \E[\Jp(\vtheta_K)] \le \epsilon + \beta \qquad \text{with}\quad NK = \frac{16\alpha^2L_{2,\text{P}} V_{\text{P}}}{\epsilon^3}\log \frac{\max\{0,\Jp^*-\Jp(\vtheta_0)-\beta\}}{\epsilon^3}.
\end{align}
By Theorem~\ref{thr:deployPB} ($i$) and ($ii$), we have that:
\begin{align}
    \Jd^* - \E \left[\Jd(\vtheta_{K})\right] = (\Jd^* - \Jp^*) + \E \left[ \Jp(\vtheta_{K})-\Jd(\vtheta_{K})\right] +  \Jp^* - \E[\Jp(\vtheta_K)] \le  \Jp^* - \E[\Jp(\vtheta_K)] + 3 L_J \sqrt{\dt} \sigma_{\text{P}}. 
\end{align}
After renaming $L_P\coloneqq L_J$ for the sake of exposition, the result follows by replacing in the sample complexity $NK$ the bounds on $ L_P$, $L_{2,\text{P}}$, and $V_{\text{P}}$ from Table~\ref{tab:magic} under the two set of assumptions and retaining only the desired dependences with the Big-$\widetilde{O}$ notation.

\end{proof}

\begin{restatable}[\textbf{\textcolor{vibrantRed}{Global convergence of PGPE - $\epsilon$-adaptive $\sigma_{\text{P}}$}}]{thr}{pgpeSamCompWGDAdaptive} \label{thr:pgpe_sam_comp_wgdAdaptive}
     Under Assumptions~\ref{ass:J_wgd} (with $J_{\simbolo}=\Jp$), \ref{ass:lipMdp}, \ref{ass:lipPol}, \ref{ass:magic}, with a suitable constant step size and {$\sigma_{\text{P}} = \frac{\epsilon}{6L_P\sqrt{\dt}} = O(\epsilon (1-\gamma)^2 \dt^{-1/2})$}, to guarantee $
         \Jd^* - \E[\Jd(\vtheta_K)] \le \epsilon + \beta$ the sample complexity of PGPE is at most:
     \begin{align}
         NK = \widetilde{O} \left( \frac{\alpha^4\dt^4}{(1-\gamma)^{12}\epsilon^7}\right).
 \end{align}
 Furthermore, under Assumptions~\ref{ass:smoothMdp} and~\ref{ass:smoothPol}, the same guarantee is obtained with a sample complexity at most:
          \begin{align}
         NK = \widetilde{O} \left( \frac{\alpha^4\dt^2}{(1-\gamma)^9\epsilon^5}\right).
 \end{align}
 \end{restatable}

\begin{proof}
    We apply Theorem~\ref{thr:pgpe_sam_comp_wgd} with $\epsilon \leftarrow \epsilon/2$ and set $\sigma_{\text{P}}$ so that:
    \begin{align}
        3L_J \sqrt{\dt} \sigma_{\text{P}} = \frac{\epsilon}{2} \implies  \sigma_{\text{P}} = \frac{\epsilon}{6L_J \sqrt{\dt}}.
    \end{align}
        After renaming $L_P\coloneqq L_J$ for the sake of exposition,
    the result follows substituting this value in the sample complexity and bounding the constant $L_{P}$ as in Table~\ref{tab:magic}.
\end{proof}

\begin{lemma}[Variance of $\widehat{\nabla}_{\vtheta} \Ja(\vtheta)$ bounded] \label{ass:Ja_estim_bounded_variance}
    Under Assumptions~\ref{ass:lipPol} and \ref{ass:magic}, the variance the GPOMDP estimator with batch size $N$ is bounded for every $ \vtheta \in \Theta$ as:
    \begin{align*}
        \Var\left[ \widehat{\nabla}_{\vtheta} \Ja(\vtheta) \right] \leq  \frac{R_{\max}^2 \xi_2 (1-\gamma^T)}{N (1-\gamma)^3} \le   \frac{R_{\max}^2 \xi_2}{N (1-\gamma)^3}.
    \end{align*}
    with $\xi_2 \le c \da \sigma_\text{A}^{-2} L_{\mu}^2$.
\end{lemma} 
\begin{proof}
    It follows from Lemma 29 of \citet{papini2022smoothing} and from the application of Lemma~\ref{lemma:boundsMagic} to bound $\xi_2$. 
\end{proof}

\begin{lemma}[Bounded $\Ja$ Hessian] \label{lem:Ja_bounded_hessian}
    Under Assumptions~\ref{ass:lipPol},~\ref{ass:smoothPol}, and~\ref{ass:magic} $\forall \vtheta \in \Theta$ it holds that:
    \begin{align*}
        \left\| \nabla^2_{\vtheta} \Ja(\vtheta) \right\|_2 \leq \frac{R_{\max} \left( 1 - \gamma^{T+1}\right)}{(1-\gamma)^2} (\upsilon_2 + \upsilon_3),
    \end{align*}
    where $\upsilon_2 \le c \da \sigma{-2}_{\text{A}} L_\mu^2$ and $\upsilon_3\le c \sigma_{\text{A}}^{-2}L_\mu^2 + c\sqrt{\da} \sigma_{\text{A}}^{-1} L_{2,\mu}$.
    Furthermore,  under Assumptions \ref{ass:lipMdp}, \ref{ass:lipPol}, \ref{ass:smoothMdp}, and \ref{ass:smoothPol}, $\forall \vtheta \in \Theta$ it holds that:
    \begin{align*}
        \left\| \nabla^2_{\vtheta} \Ja(\vtheta) \right\|_2 \le L_2,
        \end{align*}
    where $L_2$ is bounded in Lemma~\ref{lem:L_2_characterization}.
\end{lemma}
\begin{proof}
    Under Assumption~\ref{ass:magic}, by a slight modification of the proof of Lemma~4.4 by \citet{yuan2022general} (in which we consider a finite horizon $T$), it follows that:
    \begin{align*}
        \left\| \nabla_{\vtheta}^2 \Ja(\vtheta) \right\|_2 \leq \frac{R_{\max} \left( 1 - (T+1) \gamma^T + T \gamma^{T+1}\right)}{(1-\gamma)^2} (\upsilon_1 + \upsilon_2) \leq \frac{R_{\max} \left( 1 - \gamma^{T}\right)}{(1-\gamma)^2} (\upsilon_1 + \upsilon_2).
    \end{align*}

    As in the proof of Theorem~\ref{thr:assImpl}, we introduce the following convenient expression for the trajectory density function having fixed a sequence of noise $\underline{\vepsilon} \sim \Phi_{\da}^T$:
    \begin{align*}
        p_{\text{D}}(\tau; \underline{\bm{\mu}}_{\vtheta} + \underline{\vepsilon}) = \rho_0(s_{\tau, 0}) \prod_{t=0}^{T-1}  p(s_{\tau, t+1} | s_{\tau, t}, \bmu(s_{\tau, t}) + \vepsilon_t).
    \end{align*}

    This allows us to express the function $\Ja(\vtheta)$, for a generic $\vtheta \in \Theta$, as:
    \begin{align*}
        \Ja(\vtheta) = \E_{\underline{\vepsilon} \sim \Phi_{\da}^T} \left[\int_{\tau}  p_{\text{D}}(\tau; \underline{\bm{\mu}}_{\vtheta} + \underline{\vepsilon})  \sum_{t=0}^{T-1} \gamma^t r(s_{\tau, t}, \bm{\mu}_{\vtheta}(s_{\tau, t})+ \vepsilon_t ) \de \tau \right].
    \end{align*}

    With a slight abuse of notation, let us call $\Jd(\underline{\bm{\mu}_{\vtheta}} + \underline{\vepsilon})$ the following quantity:
    \begin{align*}
        \Jd(\underline{\bm{\mu}_{\vtheta}} + \underline{\vepsilon}) \coloneqq \int_{\tau}  p_{\text{D}}(\tau; \underline{\bm{\mu}}_{\vtheta} + \underline{\vepsilon})  \sum_{t=0}^{T-1} \gamma^t r(s_{\tau, t}, \bm{\mu}_{\vtheta}(s_{\tau, t})+ \vepsilon_t ) \de \tau.
    \end{align*}

    Now, considering the norm of the hessian \wrt $\vtheta$ of $\Ja$, we have that:
    \begin{align*}
        \left\| \nabla_{\vtheta}^2 \Ja(\vtheta) \right\|_2 \leq \E_{\underline{\vepsilon} \sim \Phi_{\da}^T} \left[ \left\| \nabla_{\vtheta}^2 \Jd(\underline{\bm{\mu}_{\vtheta}} + \underline{\vepsilon}) \right\|_2 \right] \leq L_{2},
    \end{align*}
    which follows from Assumptions~\ref{ass:smoothJ}.
\end{proof}

\begin{restatable}[\textbf{\textcolor{vibrantBlue}{Global convergence of GPOMDP - Fixed $\sigma_{\text{A}}$}}]{thr}{gpomdpSamCompWGD} \label{thr:gpomgp_sam_comp_wgd}
      Under Assumptions~\ref{ass:J_wgd} (with $J_{\simbolo}=\Ja$), \ref{ass:lipMdp}, \ref{ass:lipPol}, \ref{ass:smoothPol}, \ref{ass:magic}, with a suitable constant step size, to guarantee $
         \Jd^* - \E[\Jd(\vtheta_K)] \le \epsilon + \beta + 3L_A\sqrt{\da} \sigma_\text{A}
$, where $3L_A\sqrt{\da} \sigma_\text{A} = O(\sqrt{\da}\sigma_\text{A}(1-\gamma)^{-2})$ the sample complexity of GPOMDP is at most:
     \begin{align}
         NK = \widetilde{O} \left( \frac{\alpha^4\da^2}{\sigma_{\text{A}}^4(1-\gamma)^5\epsilon^3}\right).
 \end{align}
 Furthermore, under Assumption~\ref{ass:smoothMdp}, the same guarantee is obtained with a sample complexity at most:
          \begin{align}
         NK = \widetilde{O} \left( \frac{\alpha^4\da}{\sigma_{\text{A}}^2(1-\gamma)^6\epsilon^3}\right).
 \end{align}
 \end{restatable}

\begin{proof}
We first apply Theorem~\ref{thr:gen_conv_new} with $J_\simbolo = \Ja$, recalling that the  assumptions enforced in the statement entail those of Theorem~\ref{thr:gen_conv_new}:
\begin{align}
    \Ja^* - \E[\Ja(\vtheta_K)] \le \epsilon + \beta \qquad \text{with}\quad NK = \frac{16\alpha^2L_{2,\text{A}} V_{\text{A}}}{\epsilon^3}\log \frac{\max\{0,\Ja^*-\Ja(\vtheta_0)-\beta\}}{\epsilon^3}.
\end{align}
By Theorem~\ref{thr:deployAB} ($i$) and ($ii$), we have that:
\begin{align}
    \Jd^* - \E \left[\Jd(\vtheta_{K})\right] = (\Jd^* - \Ja^*) + \E \left[\Ja(\vtheta_{K}) - \Jd(\vtheta_{K})\right] +  \Ja^* - \E[\Ja(\vtheta_K)] \le  \Ja^* - \E[\Ja(\vtheta_K)] + 3 L \sqrt{\da} \sigma_{\text{A}}. 
\end{align}
After renaming $L_A\coloneqq L$ for the sake of exposition, the result follows by replacing in the sample complexity $NK$ the bounds on $ L_A$, $L_{2,\text{A}}$, and $V_{\text{A}}$ from Table~\ref{tab:magic} under the two set of assumptions and retaining only the desired dependences with the Big-$\widetilde{O}$ notation.
\end{proof}

\begin{restatable}[\textbf{\textcolor{vibrantBlue}{Global convergence of GPOMDP - $\epsilon$-adaptive $\sigma_{\text{P}}$}}]{thr}{gpomdpSamCompWGDAdaptive} \label{thr:pg_sam_comp_wgdAdaptive}
   Under Assumptions~\ref{ass:J_wgd} (with $J_{\simbolo}=\Ja$), \ref{ass:lipMdp}, \ref{ass:lipPol}, \ref{ass:smoothPol}, \ref{ass:magic}, with a suitable constant step size and {setting $\sigma_{\text{A}} = \frac{\epsilon}{6L_A\sqrt{\da}} = O(\epsilon (1-\gamma)^2 \da^{-1/2})$}, to guarantee $\Jd^* - \E[\Jd(\vtheta_K)] \le \epsilon + \beta$ the sample complexity of GPOMDP is at most:
     \begin{align}
         NK = \widetilde{O} \left( \frac{\alpha^4\da^4}{(1-\gamma)^{13}\epsilon^7}\right).
    \end{align}
    Furthermore, under Assumption~\ref{ass:smoothMdp}, the same guarantee is obtained with a sample complexity at most:
    \begin{align}
         NK = \widetilde{O} \left( \frac{\alpha^4\da^2}{(1-\gamma)^{10}\epsilon^5}\right).
    \end{align}
 \end{restatable}

\begin{proof}
    We apply Theorem~\ref{thr:gpomgp_sam_comp_wgd} with $\epsilon \leftarrow \epsilon/2$ and set $\sigma_{\text{A}}$ so that:
    \begin{align}
        3L \sqrt{\da} \sigma_{\text{A}} = \frac{\epsilon}{2} \implies  \sigma_{\text{A}} = \frac{\epsilon}{6L \sqrt{\da}}.
    \end{align}
    After renaming $L_A\coloneqq L$ for the sake of exposition, the result follows substituting this value in the sample complexity and bounding the constant $L_A$ as in Table~\ref{tab:magic}.
\end{proof}

\subsection{Proofs from Section~\ref{sec:inherited}}

\label{apx:inherited:pgpe}
\begin{lemma} \label{lem:Jd_Jp_mixed_weak_grad_dom}
    Under Assumptions~\ref{ass:lipMdp}, \ref{ass:lipPol}, \ref{ass:smoothMdp}, \ref{ass:smoothPol}, \ref{ass:J_wgdD}, and using a hyperpolicy complying with Definition~\ref{defi:add}, $\forall \vtheta \in \Theta$ it holds that:
    \begin{align*}
        \Jd^* - \Jd(\vtheta) \leq \dalpha \| \nabla_{\vtheta} \Jp(\vtheta) \|_2 + \dbeta + \dalpha L_2 \psigma \sqrt{\dt}.
    \end{align*}
\end{lemma}
\begin{proof}
    We start by observing that
    \begin{align*}
        \Jp(\vtheta) = \E_{\vtheta' \sim \nu_{\vtheta}} \left[ \Jd(\vtheta') \right] = \E_{\bm{\vepsilon} \sim \Phi} \left[ \Jd(\vtheta + \bm{\vepsilon}) \right].  
    \end{align*}

    From this fact, we can proceed as follows:
    \begin{align*}
        \nabla_{\vtheta} \Jp(\vtheta) &= \nabla_{\vtheta} \E_{\bm{\vepsilon} \sim \Phi} \left[ \Jd(\vtheta + \bm{\vepsilon}) \right] \\
        &= \E_{\bm{\vepsilon} \sim \Phi} \left[ \nabla_{\vtheta} \Jd(\vtheta + \bm{\vepsilon}) \right].
    \end{align*}

    For what follows, we define $\Tilde{\vtheta}_{\bm{\vepsilon}}$ as an intermediate parameter configuration between $\vtheta$ and $\vtheta + \bm{\vepsilon}$. 
    More formally, let $\lambda \in [0,1]$, then $ \Tilde{\vtheta}_{\bm{\vepsilon}} = \lambda \vtheta + (1 - \lambda) (\vtheta + \bm{\vepsilon})$.   
    We can proceed by rewriting the term $\nabla_{\vtheta} \Jd(\vtheta + \bm{\vepsilon})$ exploiting the first-order Taylor expansion centered in $\bm{\vepsilon}$: there exists a $\lambda\in[0,1]$ such that
    \begin{align*}
        \E_{\bm{\vepsilon} \sim g} \left[ \nabla_{\vtheta} \Jd(\vtheta + \bm{\vepsilon}) \right] &= \E_{\bm{\vepsilon} \sim g} \left[ \nabla_{\vtheta} \Jd(\vtheta) + \bm{\vepsilon}^T \nabla_{\vtheta}^2 \Jd(\Tilde{\vtheta}_{\bm{\vepsilon}}) \right] \\
        &= \nabla_{\vtheta} \Jd(\vtheta) + \E_{\bm{\vepsilon} \sim \Phi} \left[ \bm{\vepsilon}^T \nabla_{\vtheta}^2 \Jd(\Tilde{\vtheta}_{\bm{\vepsilon}}) \right].
    \end{align*}

    Now, we can consider the 2-norm of the gradient:
    \begin{align}
        \| \nabla_{\vtheta} \Jp(\vtheta)\|=\left\| \nabla_{\vtheta} \Jd(\vtheta) + \E_{\bm{\vepsilon} \sim \Phi} \left[ \bm{\vepsilon}^T \nabla^2_{\vtheta} \Jd(\Tilde{\vtheta}_{\bm{\vepsilon}}) \right] \right\|_2 &\geq \| \nabla_{\vtheta} \Jd(\vtheta) \|_2 - \left\| \E_{\bm{\vepsilon} \sim \Phi} \left[ \bm{\vepsilon}^T \nabla^2_{\vtheta} \Jd(\Tilde{\vtheta}_{\bm{\vepsilon}}) \right] \right\|_2 \nonumber\\
        &\geq \| \nabla_{\vtheta} \Jd(\vtheta) \|_2 - L_2 \E_{\bm{\vepsilon} \sim \Phi} \left[ \| \bm{\vepsilon} \|_2 \right] \label{eq:Jd_Jp_mixed_weak_grad_dom:1}\\
        &\geq \frac{1}{\dalpha} \left( \Jd^* - \Jd(\vtheta) \right) - \frac{\dbeta}{\dalpha} - L_2 \E_{\bm{\vepsilon} \sim \Phi} \left[ \| \bm{\vepsilon} \|_2 \right] \label{eq:Jd_Jp_mixed_weak_grad_dom:2}\\
        &\geq \frac{1}{\dalpha} \left( \Jd^* - \Jd(\vtheta) \right) - \frac{\dbeta}{\dalpha} - L_2 \psigma \sqrt{\dt}, \nonumber
    \end{align}
    where Equation~(\ref{eq:Jd_Jp_mixed_weak_grad_dom:1}) follows from Assumption~\ref{ass:smoothJ}, and Equation~(\ref{eq:Jd_Jp_mixed_weak_grad_dom:2}) follows from Assumption~\ref{ass:J_wgd}.
    Thus, it simply follows that:
    \begin{align*}
        \Jd^* - \Jd(\vtheta) \leq \dalpha \| \nabla_{\vtheta} \Jp(\vtheta) \|_2 + \dbeta + \dalpha L_2 \psigma \sqrt{\dt}.
    \end{align*}
\end{proof}

\pgpeInheritedWGD*
\begin{proof}
	We recall that under the assumptions in the statement, the results of Lemma~\ref{lem:Jd_Jp_mixed_weak_grad_dom} and of Theorem~\ref{thr:deployPB} hold
	In particular, we need the result from Theorem~\ref{thr:deployPB}, saying that $\forall \vtheta \in \Theta$ it holds that
	\begin{align}
		\Jp(\vtheta) - L_J \psigma \sqrt{\dt} \leq \Jd(\vtheta) \leq \Jp(\vtheta) + L_J \psigma \sqrt{\dt}. \label{eq:Jp_weak_gradient_domination_eq_1}
	\end{align}
	Thus, using the result of Lemma~\ref{lem:Jd_Jp_mixed_weak_grad_dom}, we need to work on the left-hand side of the following inequality:
	\begin{align*}
		\Jd^* - \Jd(\vtheta) \leq \dalpha \| \nabla_{\vtheta} \Jp(\vtheta) \|_2 + \dbeta + \dalpha L_2 \psigma \sqrt{\dt}.
	\end{align*}

    Moreover, by definition of $\Jp$, we have that $\Jd^* \geq \Jp^*$.
    Thus, it holds that:
    \begin{align*}
        \Jd^* - \Jd(\vtheta) &\geq \Jp^* - \Jp(\vtheta) \\
        &\geq \Jp^* - \Jp(\vtheta) - L_J \sqrt{\dt},
    \end{align*}
    where the last line follows from Line~(\ref{eq:Jp_weak_gradient_domination_eq_1}). We rename $L_P\coloneqq L_J$ in the statement.

\end{proof}

\begin{lemma} \label{lem:Jd_Ja_mixed_weak_grad_dom}
    Under Assumptions~\ref{ass:J_wgdD}, \ref{ass:lipMdp}, \ref{ass:lipPol}, \ref{ass:smoothMdp}, \ref{ass:smoothPol}, using a policy complying
with Definition~\ref{defi:add}, $\forall \vtheta \in \Theta$, it holds that:
    \begin{align*}
        \Jd^* - \Jd(\vtheta) \leq \dalpha \| \nabla_{\vtheta} \Ja(\vtheta) \|_2 + \dbeta + \dalpha \psi \asigma \sqrt{\da},
    \end{align*}

    where
    \begin{align*}
        \psi = L_{\mu}\left(\frac{L_{ p}^2 R_{\max} \gamma}{(1 - \gamma)^4} + \frac{(L_{r} L_{ p} + R_{\max} L_{2,  p} + L_{ p} L_{r} \gamma)}{(1 - \gamma)^2} + \frac{L_{2, r}}{1-\gamma}\right)(1 - \gamma^{T}).
    \end{align*}
\end{lemma}
\begin{proof}
    As in the proof of Theorem~\ref{thr:assImpl}, we introduce the following convenient expression for the trajectory density function having fixed a sequence of noise $\underline{\vepsilon} \sim \Phi_{\da}^T$:
    \begin{align}
        p_{\text{D}}(\tau; \underline{\bm{\mu}}_{\vtheta} + \underline{\vepsilon}) = \rho_0(s_{\tau, 0}) \prod_{t=0}^{T-1}  p(s_{\tau, t+1} | s_{\tau, t}, \bmu(s_{\tau, t}) + \vepsilon_t).
    \end{align}
    Also in this case, we denote with $ p_{\text{D}}(\tau_{0:l}; \underline{\bm{\mu}}_{\vtheta} + \underline{\vepsilon})$ the density function of a trajectory prefix of length $l$:
    \begin{align}
       p_{\text{D}}(\tau_{0:l}; \underline{\bm{\mu}}_{\vtheta} + \underline{\vepsilon}) = \rho_0(s_{\tau, 0}) \prod_{t=0}^{l-1}  p(s_{\tau, t+1} | s_{\tau, t}, \bmu(s_{\tau, t}) + \vepsilon_t).  
    \end{align}
    From the proof of Proposition~\ref{thr:assImpl},
    considering a generic parametric configuration $\vtheta \in \Theta$, we can write the AB performance index $\Ja(\vtheta)$ as:
    \begin{align*}
        \Ja(\vtheta) & = \E_{\underline{\vepsilon} \sim \Phi_{\da}^T} \left[\int_{\tau}  p_{\text{D}}(\tau; \underline{\bm{\mu}}_{\vtheta} + \underline{\vepsilon})  \sum_{t=0}^{T-1} \gamma^t r(s_t, \bm{\mu}_{\vtheta}(s_t)+ \vepsilon_t ) \de \tau \right] \\
        & = \E_{\underline{\vepsilon} \sim \Phi_{\da}^T} \left[\underbrace{\sum_{t=0}^{T-1}  \int_{\tau_{0:t}}  p_{\text{D}}(\tau_{0:t}; \underline{\bm{\mu}}_{\vtheta} + \underline{\vepsilon})  \gamma^t r(s_t, \bm{\mu}_{\vtheta}(s_t)+ \vepsilon_t ) \de \tau_{0:t}}_{\eqqcolon f(\underline{\vepsilon})}\right],
    \end{align*}

    moreover, by using the Taylor expansion centered in $\underline{\vepsilon}=\underline{\bm{0}}$, for $\widetilde{\underline{\vepsilon}} = x \underline{\vepsilon}$ (for some $x \in [0,1]$) the following holds:
    \begin{align*}
        \Ja(\vtheta) &= \E_{\underline{\vepsilon}\sim \Phi_{\da}^T}\left[
        f(\underline{\bm{0}}) + \sum_{t=0}^{T-1} \vepsilon_t^\top \nabla_{\vepsilon_t} f(\underline{\vepsilon})\rvert_{\underline{\vepsilon} = \widetilde{\underline{\vepsilon}}}\right] \\
        &=\Jd(\vtheta) + \sum_{t=0}^{T-1} \E_{\underline{\vepsilon}\sim \Phi_{\da}^T}[\vepsilon_t^\top \nabla_{\vepsilon_t} f(\underline{\vepsilon})\rvert_{\underline{\vepsilon} = \widetilde{\underline{\vepsilon}}}].
    \end{align*}

    Here, we are interested in the gradient of $\Ja$:
    \begin{align*}
        \nabla_{\vtheta} \Ja(\vtheta) &= \nabla_{\vtheta} \Jd(\vtheta) +\sum_{t=0}^{T-1} \nabla_{\vtheta} \E_{\underline{\vepsilon}\sim \Phi_{\da}^T}[\vepsilon_t^\top \nabla_{\vepsilon_t} f(\underline{\vepsilon})\rvert_{\underline{\vepsilon} = \widetilde{\underline{\vepsilon}}}] \\
        &= \nabla_{\vtheta} \Jd(\vtheta) + \sum_{t=0}^{T-1} \E_{\underline{\vepsilon}\sim \Phi_{\da}^T}[\vepsilon_t^\top \nabla_{\vtheta} \nabla_{\vepsilon_t} f(\underline{\vepsilon})\rvert_{\underline{\vepsilon} = \widetilde{\underline{\vepsilon}}}].
    \end{align*}

    Now, considering the norm of the gradient we have:
    \begin{align*}
        \left\| \nabla_{\vtheta} \Ja(\vtheta)\right\|_2 &\geq \left\| \nabla_{\vtheta} \Jd(\vtheta) \right\|_2 - \left\| \sum_{t=0}^{T-1} \E_{\underline{\vepsilon}\sim \Phi_{\da}^T}[\vepsilon_t^\top \nabla_{\vtheta} \nabla_{\vepsilon_t} f(\underline{\vepsilon})\rvert_{\underline{\vepsilon} = \widetilde{\underline{\vepsilon}}}] \right\|_2 \\
        &\geq \frac{1}{\dalpha} \left( \Jd^* - \Jd(\vtheta) \right) - \frac{\dbeta}{\dalpha} - \left\| \sum_{t=0}^{T-1} \E_{\underline{\vepsilon}\sim \Phi_{\da}^T}[\vepsilon_t^\top \nabla_{\vtheta} \nabla_{\vepsilon_t} f(\underline{\vepsilon})\rvert_{\underline{\vepsilon} = \widetilde{\underline{\vepsilon}}}] \right\|_2 \\
        &\geq \frac{1}{\dalpha} \left( \Jd^* - \Jd(\vtheta) \right) - \frac{\dbeta}{\dalpha} - \sum_{t=0}^{T-1} \E_{{\vepsilon_t}\sim \Phi_{\da}}[\|\vepsilon_t\|_2^2]^{1/2} \E_{\underline{\vepsilon}\sim \Phi_{\da}^T}[ \| \nabla_{\vtheta} \nabla_{\vepsilon_t} f(\underline{\vepsilon})\rvert_{\underline{\vepsilon} = \widetilde{\underline{\vepsilon}}}\|_2^2]^{1/2} \\
        &\geq \frac{1}{\dalpha} \left( \Jd^* - \Jd(\vtheta) \right) - \frac{\dbeta}{\dalpha} - \asigma \sqrt{\da} \sum_{t=0}^{T-1} \E_{\underline{\vepsilon}\sim \Phi_{\da}^T}[ \| \nabla_{\vtheta} \nabla_{\vepsilon_t} f(\underline{\vepsilon})\rvert_{\underline{\vepsilon} = \widetilde{\underline{\vepsilon}}}\|_2^2]^{1/2},
    \end{align*}    
    where the second inequality is by Assumption~\ref{ass:J_wgdD}.
    Re-arranging the last inequality, we have:
    \begin{align*}
        \Jd^* - \Jd(\vtheta) \leq \dalpha \left\| \nabla_{\vtheta} \Ja(\vtheta)\right\|_2 + \dbeta + \dalpha \asigma \sqrt{\da} \sum_{t=0}^{T-1} \E_{\underline{\vepsilon}\sim \Phi_{\da}^T}[ \| \nabla_{\vtheta} \nabla_{\vepsilon_t} f(\underline{\vepsilon})\rvert_{\underline{\vepsilon} = \widetilde{\underline{\vepsilon}}}\|_2^2]^{1/2}.
    \end{align*}
    In order to conclude the proof, we need to bound the term $\sum_{t=0}^{T-1} \E_{\underline{\vepsilon}\sim \Phi_{\da}^T}[ \| \nabla_{\vtheta} \nabla_{\vepsilon_t} f(\underline{\vepsilon})\rvert_{\underline{\vepsilon} = \widetilde{\underline{\vepsilon}}}\|_2^2]^{1/2}$.
    From the proof of Proposition~\ref{thr:assImpl}, for any index $k \in \dsb{T}$, we have that:
    \begin{align*}
        \nabla_{\vepsilon_{k}} f(\underline{\vepsilon}) = \E_{\tau \sim p_{\text{D}}(\cdot; \underline{\bm{\mu}}_{\vtheta} + \underline{\vepsilon})} \left[ \sum_{t=k}^{T-1} \gamma^{t} r(s_{t}, \mu_{\vtheta}(s_{t}) + \vepsilon_{t}) \nabla_{\vepsilon_{k}} \log p(s_{k+1} | s_{k}, \mu_{\vtheta}(s_{k}) + \vepsilon_{k}) + \gamma^{k} \nabla_{\vepsilon_{k}} r(s_{k}, \mu_{\vtheta}(s_{k}) + \vepsilon_{k}) \right],
    \end{align*}
    from which we can derive $\nabla_{\vtheta} \nabla_{\vepsilon_{k}} f(\underline{\vepsilon})$ as follows:
    \begin{align*}
        &\nabla_{\vtheta} \nabla_{\vepsilon_{k}} f(\underline{\vepsilon}) \\
        &= \nabla_{\vtheta} \int_{\tau} p_{\text{D}}(\tau; \underline{\bm{\mu}}_{\vtheta} + \underline{\vepsilon}) \left( \sum_{t=k}^{T-1} \gamma^{t} r(s_{t}, \mu_{\vtheta}(s_{t}) + \vepsilon_{t}) \nabla_{\vepsilon_{k}} \log p(s_{k+1} | s_{k}, \mu_{\vtheta}(s_{k}) + \vepsilon_{k}) + \gamma^{k} \nabla_{\vepsilon_{k}} r(s_{k}, \mu_{\vtheta}(s_{k}) + \vepsilon_{k}) \right) \de \tau\\
        &= \underbrace{\nabla_{\vtheta} \int_{\tau} p_{\text{D}}(\tau; \underline{\bm{\mu}}_{\vtheta} + \underline{\vepsilon}) \sum_{t=k}^{T-1} \gamma^{t} r(s_{t}, \mu_{\vtheta}(s_{t}) + \vepsilon_{t}) \nabla_{\vepsilon_{k}} \log p(s_{k+1} | s_{k}, \mu_{\vtheta}(s_{k}) + \vepsilon_{k}) \de \tau}_{\text{(i)}} \\
        &\qquad + \underbrace{\nabla_{\vtheta} \int_{\tau} p_{\text{D}}(\tau; \underline{\bm{\mu}}_{\vtheta} + \underline{\vepsilon}) \gamma^{k} \nabla_{\vepsilon_{k}} r(s_{k}, \mu_{\vtheta}(s_{k}) + \vepsilon_{k}) \de \tau}_{\text{(ii)}}.
    \end{align*}
    We will consider the terms (i) and (ii) separately.
    However, we first need to clarify what happens when we try to compute the gradient \wrt $\vtheta$ and $\vepsilon_{t}$, for a generic $t \in \{0,\dots, T-1\}$.
    To this purpose let $g(\cdot, \va)$ be a generic differentiable function of the action $\va= \mu_{\vtheta}(s_t) + \vepsilon_t$.
	The norm of its gradient \wrt $\vepsilon_t$ can be written as:
	\begin{align*}
		\| \nabla_{\vepsilon_t} g(\cdot, \va) \rvert_{\va = \mu_{\vtheta}(s_t) + \vepsilon_t} \|_2 &= \| \nabla_{\va}  g(\cdot, \va) \rvert_{\va = \mu_{\vtheta}(s_t) + \vepsilon_t} \nabla_{\vepsilon_t}(\mu_{\vtheta}(s_t) + \vepsilon_t) \|_2 \\
		&= \|  \nabla_{\va}  g(\cdot, \va) \rvert_{\va = \mu_{\vtheta}(s_t) + \vepsilon_t} \|_2.
	\end{align*}
	On the other hand, the norm of the gradient of $g$ \wrt $\vtheta$ can be written as:
	\begin{align*}
			\| \nabla_{\vtheta} g(\cdot, \va) \rvert_{\va = \mu_{\vtheta}(s_t) + \vepsilon_t} \|_2 &= \| \nabla_{\va}  g(\cdot, \va) \rvert_{\va = \mu_{\vtheta}(s_t) + \vepsilon_t} \nabla_{\vtheta}(\mu_{\vtheta}(s_t) + \vepsilon_t) \|_2 \\
			&= \| \nabla_{\va}  g(\cdot, \va) \rvert_{\va = \mu_{\vtheta}(s_t) + \vepsilon_t} \nabla_{\vtheta}\mu_{\vtheta}(s_t) \|_2.
	\end{align*}
	Moreover, the norm of the gradient \wrt $\vtheta$ of the gradient of $g$ \wrt $\vepsilon_t$, can be written as:
	\begin{align*}
			\| \nabla_{\vtheta} \nabla_{\vepsilon_t} g(\cdot, \va) \rvert_{\va = \mu_{\vtheta}(s_t) + \vepsilon_t} \|_2 &= \|  \nabla_{\vtheta} \nabla_{\va}  g(\cdot, \va) \rvert_{\va = \mu_{\vtheta}(s_t) + \vepsilon_t} \|_2 \\
			&= \| \nabla_{\va}^2 g(\cdot, \va) \rvert_{\va = \mu_{\vtheta} + \vepsilon_t} \nabla_{\vtheta} \mu_{\vtheta}(s_t)\|_2.
	\end{align*}
	
    Having said this, we can proceed by analyzing the terms (i) and (ii).

    The term (i) can be rewritten as:
    \begin{align*}
        \text{(i)} &= \nabla_{\vtheta} \int_{\tau} p_{\text{D}}(\tau; \underline{\bm{\mu}}_{\vtheta} + \underline{\vepsilon}) \sum_{t=k}^{T-1} \gamma^{t} r(s_{t}, \mu_{\vtheta}(s_{t}) + \vepsilon_{t}) \nabla_{\vepsilon_{k}} \log p(s_{k+1} | s_{k}, \mu_{\vtheta}(s_{k}) + \vepsilon_{k}) \de \tau \\
        &= \sum_{t=k}^{T-1} \gamma^{t}\nabla_{\vtheta} \int_{\tau_{0:t}} p_{\text{D}}(\tau_{0:t}; \underline{\bm{\mu}}_{\vtheta} + \underline{\vepsilon})  r(s_{t}, \mu_{\vtheta}(s_{t}) + \vepsilon_{t}) \nabla_{\vepsilon_{k}} \log p(s_{k+1} | s_{k}, \mu_{\vtheta}(s_{k}) + \vepsilon_{k}) \de \tau_{0:t} \\
        &= \sum_{t=k}^{T-1} \gamma^{t}\E_{\tau_{0:t} \sim p_{\text{D}}(\cdot; \underline{\bm{\mu}}_{\vtheta} + \underline{\vepsilon})} \Bigg[ \nabla_{\vtheta} \log p_{\text{D}}(\tau_{0:t}, \underline{\bm{\mu}_{\vtheta}} + \underline{\vepsilon}) r(s_{t}, \mu_{\vtheta}(s_{t}) + \vepsilon_{t}) \nabla_{\vepsilon_{k}} \log p(s_{k+1} | s_{k}, \mu_{\vtheta}(s_{k}) + \vepsilon_{k}) \\
        &\qquad \qquad + \nabla_{\vtheta} r(s_{t}, \mu_{\vtheta}(s_{t}) + \vepsilon_{t}) \nabla_{\vepsilon_{k}} \log p(s_{k+1} \rvert s_{k}, \mu_{\vtheta}(s_{k}) + \vepsilon_{k}) \\
        &\qquad \qquad +  r(s_{t}, \mu_{\vtheta}(s_{t}) + \vepsilon_{t}) \nabla_{\vtheta} \nabla_{\vepsilon_{k}} \log p(s_{k+1} \rvert s_{k}, \mu_{\vtheta}(s_{k}) + \vepsilon_{k}) \Bigg] \\
        &= \sum_{t=k}^{T-1} \gamma^{t}\E_{\tau_{0:t} \sim p_{\text{D}}(\cdot; \underline{\bm{\mu}}_{\vtheta} + \underline{\vepsilon})} \Bigg[ \sum_{j=0}^{t-1} \nabla_{\vtheta} \log p(s_{j+1} \rvert s_{j}, \mu_{\vtheta}(s_{j}) +\vepsilon_{j})  r(s_{t}, \mu_{\vtheta}(s_{t}) + \vepsilon_{t}) \nabla_{\vepsilon_{k}} \log p(s_{k+1} | s_{k}, \mu_{\vtheta}(s_{k}) + \vepsilon_{k}) \\
        &\qquad \qquad + \nabla_{\vtheta} r(s_{t}, \mu_{\vtheta}(s_{t}) + \vepsilon_{t}) \nabla_{\vepsilon_{k}} \log p(s_{k+1} \rvert s_{k}, \mu_{\vtheta}(s_{k}) + \vepsilon_{k}) \\
        &\qquad \qquad + r(s_{t}, \mu_{\vtheta}(s_{t}) + \vepsilon_{t}) \nabla_{\vtheta} \nabla_{\vepsilon_{k}} \log p(s_{k+1} \rvert s_{k}, \mu_{\vtheta}(s_{k}) + \vepsilon_{k}) \Bigg].
    \end{align*}

    We need to bound its norm, thus we can proceed as follows:
    \begin{align*}
        &\left\| (i) \right\|_2 \\
        &\leq \sum_{t=k}^{T-1} \gamma^{t}\E_{\tau_{0:t} \sim p_{\text{D}}(\cdot; \underline{\bm{\mu}}_{\vtheta} + \underline{\vepsilon})} \Bigg[ \sum_{j=0}^{t-1} \left\| \nabla_{\vtheta} \log p(s_{j+1} \rvert s_{j}, \mu_{\vtheta}(s_{j}) +\vepsilon_{j}) \right\|_2  \left| r(s_{t}, \mu_{\vtheta}(s_{t}) + \vepsilon_{t}) \right| \left\| \nabla_{\vepsilon_{k}} \log p(s_{k+1} | s_{k}, \mu_{\vtheta}(s_{k}) + \vepsilon_{k}) \right\|_2 \\
        &\qquad \qquad +  \left\| \nabla_{\vtheta} r(s_{t}, \mu_{\vtheta}(s_{t}) + \vepsilon_{t}) \right\|_2 \left\| \nabla_{\vepsilon_{k}} \log p(s_{k+1} \rvert s_{k}, \mu_{\vtheta}(s_{k}) + \vepsilon_{k}) \right\|_2 \\
        &\qquad \qquad +  \left| r(s_{t}, \mu_{\vtheta}(s_{t}) + \vepsilon_{t}) \right| \left\| \nabla_{\vtheta} \nabla_{\vepsilon_{k}} \log p(s_{k+1} \rvert s_{k}, \mu_{\vtheta}(s_{k}) + \vepsilon_{k}) \right\|_2 \Bigg] \\
        &\leq L_{\mu} L_{\log p}^2 R_{\max}\sum_{t=k}^{T-1} t \gamma^{t}  + \left(L_{\mu} L_{r} L_{\log p}+ R_{\max} L_{\mu} L_{2, \log p} \right)\sum_{t=k}^{T-1}\gamma^t \\
        &= L_{\mu} L_{\log p}^2 R_{\max} \gamma \left( \frac{1 - T \gamma^{T-1} + (T-1)\gamma^{T}}{(1 - \gamma)^2} - \frac{1 - k \gamma^{k-1} + (k-1) \gamma^{k}}{(1-\gamma)^{2}}\right) + (L_{\mu} L_{r} L_{\log p} + R_{\max} L_{\mu} L_{2, \log p}) \frac{\gamma^k - \gamma^T}{1 - \gamma} \\
        &\leq L_{\mu} L_{\log p}^2 R_{\max} \gamma \left( \frac{1 - \gamma^{T}}{(1 - \gamma)^2} - \frac{1 - k \gamma^{k-1}}{(1-\gamma)^{2}}\right) + (L_{\mu} L_{r} L_{\log p} + R_{\max} L_{\mu} L_{2, \log p}) \frac{\gamma^k - \gamma^T}{1 - \gamma} \\
        &\leq L_{\mu} L_{\log p}^2 R_{\max} \gamma \left( \frac{1 - \gamma^{T}}{(1 - \gamma)^2} - \frac{1 - k \gamma^{k-1}}{(1-\gamma)^{2}}\right) + (L_{\mu} L_{r} L_{\log p} + R_{\max} L_{\mu} L_{2, \log p}) \frac{\gamma^k - \gamma^T}{1 - \gamma}.
    \end{align*}

    Finally, we have to sum over $k \in \dsb{T}$:
    \begin{align*}
        \sum_{k=0}^{T-1} \left\| (i) \right\|_2 &\leq \sum_{k=0}^{T-1} L_{\mu} L_{\log p}^2 R_{\max} \gamma \left( \frac{1 - \gamma^{T}}{(1 - \gamma)^2} - \frac{1 - k \gamma^{k-1}}{(1-\gamma)^{2}}\right) + (L_{\mu} L_{r} L_{\log p} + R_{\max} L_{\mu} L_{2, \log p}) \frac{\gamma^k - \gamma^T}{1 - \gamma} \\
        &= L_{\mu} L_{\log p}^2 R_{\max} \gamma \left( T \frac{1 - \gamma^{T}}{(1 - \gamma)^2} - \frac{T}{(1-\gamma)^2}  + \frac{1-\gamma^T}{(1-\gamma)^4} -\frac{T\gamma^{T-1}}{(1-\gamma)^3}\right) \\
        &\qquad + (L_{\mu} L_{r} L_{\log p} + R_{\max} L_{\mu} L_{2, \log p}) \frac{1 - 2 \gamma^T + \gamma^{T+1}}{(1 - \gamma)^2} \\
        &\leq L_{\mu} L_{\log p}^2 R_{\max} \gamma  \frac{1 - \gamma^{T}}{(1 - \gamma)^4} + (L_{r} L_{\log p} + R_{\max} L_{2, \log p}) L_{\mu} \frac{1 - \gamma^T}{(1 - \gamma)^2}.
    \end{align*}
    
    The term (ii) can be rewritten as:
    \begin{align*}
        \text{(ii)} &= \nabla_{\vtheta} \int_{\tau} p_{\text{D}}(\tau; \underline{\bm{\mu}}_{\vtheta} + \underline{\vepsilon}) \gamma^{k} \nabla_{\vepsilon_{k}} r(s_{k}, \mu_{\vtheta}(s_{k}) + \vepsilon_{k}) \de \tau \\
        &= \nabla_{\vtheta} \int_{\tau_{0:k}} p_{\text{D}}(\tau_{0:k}; \underline{\bm{\mu}}_{\vtheta} + \underline{\vepsilon}) \gamma^{k} \nabla_{\vepsilon_{k}} r(s_{k}, \mu_{\vtheta}(s_{k}) + \vepsilon_{k}) \de \tau_{0:k} \\
        &= \gamma^{k}\int_{\tau_{0:k}} p_{\text{D}}(\tau_{0:k}; \underline{\bm{\mu}}_{\vtheta} + \underline{\vepsilon}) \left( \nabla_{\vtheta} \log p_{\text{D}}(\tau_{0:k}; \underline{\bm{\mu}}_{\vtheta} + \underline{\vepsilon})  \nabla_{\vepsilon_{k}} r(s_{k}, \mu_{\vtheta}(s_{k}) + \vepsilon_{k}) + \nabla_{\vtheta} \nabla_{\vepsilon_{k}} r(s_{k}, \mu_{\vtheta}(s_{k}) + \vepsilon_{k}) \right) \\
        &= \gamma^{k}\E_{\tau_{0:k} \sim p_{\text{D}}(\cdot; \underline{\bm{\mu}}_{\vtheta} + \underline{\vepsilon})} \left[ \sum_{j=0}^{k-1} \nabla_{\vtheta} \log p(s_{j+1} \rvert s_{j}, \mu_{\vtheta}(s_{j}) + \vepsilon_{j} )  \nabla_{\vepsilon_{k}} r(s_{k}, \mu_{\vtheta}(s_{k}) + \vepsilon_{k}) + \nabla_{\vtheta} \nabla_{\vepsilon_{k}} r(s_{k}, \mu_{\vtheta}(s_{k}) + \vepsilon_{k}) \right].
    \end{align*}
    We need to bound its norm, thus we can proceed as follows:
    \begin{align*}
        \left\| \text{(ii)} \right\|_2 &\leq \gamma^{k}\E_{\tau_{0:k} \sim p_{\text{D}}(\cdot; \underline{\bm{\mu}}_{\vtheta} + \underline{\vepsilon})} \Bigg[ \sum_{j=0}^{k-1} \left\|\nabla_{\vtheta} \log p(s_{j+1} \rvert s_{j}, \mu_{\vtheta}(s_{j}) + \vepsilon_{j} ) \right\|_2  \left\| \nabla_{\vepsilon_{k}} r(s_{k}, \mu_{\vtheta}(s_{k}) + \vepsilon_{k}) \right\|_2 \\
        &\qquad \qquad + \left\| \nabla_{\vtheta} \nabla_{\vepsilon_{k}} r(s_{k}, \mu_{\vtheta}(s_{k}) + \vepsilon_{k}) \right\|_2 \Bigg]\\
        &\leq L_{\mu} L_{ p} L_{r} k \gamma^{k} + L_{\mu} L_{2, r} \gamma^{k}.
    \end{align*}

    Finally, we have to sum over $k \in \dsb{T}$:
    \begin{align*}
        \sum_{k=0}^{T-1} \left\| \text{(ii)} \right\|_2 &= L_{\mu} L_{ p} L_{r} \sum_{k=0}^{T-1} k \gamma^{k} + L_{\mu} L_{2, r} \sum_{k=0}^{T-1} \gamma^{k} \\
        &\leq L_{\mu} L_{ p} L_{r} \gamma \frac{1 - T \gamma^{T-1} + (T-1)\gamma^{T}}{(1 - \gamma)^2} + L_{\mu} L_{2, r} \frac{1 - \gamma^{T}}{1-\gamma} \\
        &\leq L_{\mu} L_{ p} L_{r} \gamma \frac{1 - \gamma^{T}}{(1 - \gamma)^2} + L_{\mu} L_{2, r} \frac{1 - \gamma^{T}}{1-\gamma}.
    \end{align*}
    Putting together the bounds on (i) and (ii):
    \begin{align*}
        &\sum_{t=0}^{T-1} \E_{\underline{\vepsilon}\sim \Phi_{\da}^T}[ \| \nabla_{\vtheta} \nabla_{\vepsilon_t} f(\underline{\vepsilon})\rvert_{\underline{\vepsilon} = \widetilde{\underline{\vepsilon}}}\|_2^2]^{1/2} \\
        &\leq \sum_{k=0}^{T-1} \left\| (i) \right\|_2 + \sum_{k=0}^{T-1} \left\| \text{(ii)} \right\|_2 \\
        &\leq L_{\mu} L_{ p}^2 R_{\max} \gamma \frac{1 - \gamma^{T}}{(1 - \gamma)^4} + (L_{r} L_{ p} + R_{\max} L_{2,  p}) L_{\mu} \frac{1 - \gamma^T}{(1 - \gamma)^2} \\
        &\qquad + L_{\mu} L_{ p} L_{r} \gamma \frac{1 - \gamma^{T}}{(1 - \gamma)^2} + L_{\mu} L_{2, r} \frac{1 - \gamma^{T}}{1-\gamma} \\
        &\leq L_{\mu}\left(\frac{L_{ p}^2 R_{\max} \gamma}{(1 - \gamma)^4} + \frac{(L_{r} L_{ p} + R_{\max} L_{2,  p} + L_{ p} L_{r} \gamma)}{(1 - \gamma)^2} + \frac{L_{2, r}}{1-\gamma}\right)(1 - \gamma^{T}),
    \end{align*}
    which concludes the proof.
\end{proof}

\pgInheritedWGD*
\begin{proof}
	This proof directly follows from the combination of Lemma~\ref{lem:Jd_Ja_mixed_weak_grad_dom} and Theorem~\ref{thr:assImpl},
    and we can proceed as in the proof of Theorem~\ref{thr:Jp_weak_gradient_domination}.
	Indeed, recalling that $L$ is
     \begin{align*}
         L = \frac{\gamma + \gamma^{1 + T} (T-1) - T \gamma^T }{(1-\gamma)^2} L_{\log p}  R_{\max} + \frac{1 - \gamma^T}{1 - \gamma} L_r,
     \end{align*}
    from Theorem~\ref{thr:assImpl},
    it follows that:
	\begin{align}
		\Ja(\vtheta) - L \asigma \sqrt{\da} \leq \Jd(\vtheta) \leq \Ja(\vtheta) + L \asigma \sqrt{\da}. \label{eq:Ja_weak_gradient_domination_eq_1}
	\end{align}
	Analogously to the proof of Theorem~\ref{thr:Jp_weak_gradient_domination}, it is useful to notice that by definition of $\Ja$, we have $\Jd^* \geq \Ja^*$.
    Thus, it holds that:
    \begin{align*}
        \Jd^* - \Jd(\vtheta) &\geq \Ja^* - \Jd(\vtheta) \\
        &\geq \Ja^* - \Ja(\vtheta) - L \asigma \sqrt{\da},
    \end{align*}
    where the last line follows from Line~(\ref{eq:Ja_weak_gradient_domination_eq_1}). We rename $L_A\coloneqq L$ in the statement.
\end{proof}

\begin{restatable}[Global convergence of PGPE - Inherited WGD]{thr}{pgpeSamCompWGDInherited} \label{thr:pgpe_sam_comp_wgdInherited}
     Consider the PGPE algorithm. Under Assumptions~\ref{ass:lipMdp}, \ref{ass:lipPol}, \ref{ass:smoothMdp},~\ref{ass:smoothPol}, \ref{ass:magic},  \ref{ass:J_wgdD}, with a suitable constant step size and {setting $\sigma_{\text{P}} = \frac{\epsilon}{(\dalpha L_2+4L_J)\sqrt{\dt}} = O(\epsilon (1-\gamma)^3 \dt^{-1/2})$}, to guarantee $
         \Jd^* - \E[\Jd(\vtheta_K)] \le \epsilon + \dbeta$ the sample complexity is at most:
          \begin{align}
         NK = \widetilde{O} \left( \frac{\dalpha^6\dt^2}{(1-\gamma)^{11}\epsilon^5}\right).
 \end{align}
 \end{restatable}

 \begin{proof}
     Simply apply Theorem~\ref{thr:gen_conv_new} and Theorem~\ref{thr:deployPB} to obtain a guarantee of $\Jd^* - \E[\Jd(\vtheta_K)] \le \epsilon/2 + \dbeta + (\dalpha L_2+L_J)\sqrt{\dt}\sigma_{\text{P}} + 3L_J\sqrt{\dt}\sigma_{\text{P}} = \epsilon + \dbeta + (\dalpha L_2+4L_J)\sqrt{\dt}\sigma_{\text{P}}$. Then, we set $\sigma_{\text{P}}$ to ensure that $(\dalpha L_2+4L_J)\sqrt{\dt}\sigma_{\text{P}} = \epsilon/2$.
 \end{proof}

 \begin{restatable}[Global convergence of GPOMDP - Inherited WGD]{thr}{gpomgpSamCompWGDInherited} \label{thr:gpomgp_sam_comp_wgdInherited}
     Consider the GPOMDP algorithm. Under Assumptions~\ref{ass:lipMdp}, \ref{ass:lipPol}, \ref{ass:smoothMdp},~\ref{ass:smoothPol}, \ref{ass:magic}, \ref{ass:J_wgdD},  with a suitable constant step size and {setting $\sigma_{\text{A}} = \frac{\epsilon}{(\dalpha \Psi+4L)\sqrt{\da}} = O(\epsilon (1-\gamma)^4 \da^{-1/2})$}, to guarantee $
         \Jd^* - \E[\Jd(\vtheta_K)] \le \epsilon + \dbeta$ the sample complexity is at most:
          \begin{align}
         NK = \widetilde{O} \left( \frac{\dalpha^6\da^2}{(1-\gamma)^{14}\epsilon^5}\right).
 \end{align}
 \end{restatable}

 \begin{proof}
     Simply apply Theorem~\ref{thr:gen_conv_new} and Theorem~\ref{thr:deployAB} to obtain a guarantee of $\Jd^* - \E[\Jd(\vtheta_K)] \le \epsilon/2 + \dbeta + (\dalpha \Psi+L)\sqrt{\dt}\sigma_{\text{A}} + 3L\sqrt{\dt}\sigma_{\text{A}} = \epsilon + \dbeta + (\dalpha \Psi+4L)\sqrt{\dt}\sigma_{\text{A}} $. Then, we set $\sigma_{\text{A}}$ to ensure that $(\dalpha \Psi+4L)\sqrt{\da}\sigma_{\text{A}} = \epsilon/2$.
 \end{proof}

\subsection{Proofs from Section~\ref{sec:fisher}}\label{app:fisher}
In this section, we focus on AB exploration with white-noise policies (Definition~\ref{defi:add}), and give the proofs that were omitted in Section~\ref{sec:fisher}.
We denote by $\upsilon_{\vtheta}(\cdot,\cdot)$ the state-action distribution induced by the (stochastic) policy $\pi_{\vtheta}$, and, with some abuse of notation, $\upsilon_{\vtheta}(\cdot)$ to denote the corresponding state distribution. We denote by $A^{\vtheta}:\Ss\times\As\to\mathbb{R}$ the advantage function of $\pi_{\vtheta}$~\citep[for the standard definitions, see][]{sutton2018reinforcement}.

We first have to give a formal characterization of $\epsilon_\mathrm{bias}$. Equivalent definitions appeared in~\citep{liu2020improved,ding2022global,yuan2022general}, but the concept dates back at least to~\citep{peters2005natural}.

\begin{defi}\label{ass:compatible_critic}
    Let $\ell(\bm{w};s,a,\vtheta)=\left(A^{\vtheta}(s,a)-(1-\gamma)\bm{w}^\top\nabla_{\vtheta}\log\pi_{\vtheta}(a|s)\right)^2$, and $\bm{w}^\star(\vtheta) = \arg\min_{\bm{w}\in\mathbb{R}^{d_\Theta}}\E_{s,a\sim \upsilon_{\vtheta}}[\ell(\bm{w};s,a,\vtheta)]$. We define $\epsilon_\mathrm{bias}$ as the smallest positive constant such that, for all $\vtheta\in\Theta$, $\E_{s,a\sim \upsilon_{\vtheta^\star}}[\ell(\bm{w}^\star(\vtheta);s,a,\vtheta)]\le\epsilon_\mathrm{bias}$, where $\vtheta^\star\in\arg\max \Ja(\vtheta)$.
\end{defi}

We begin by showing that white-noise policies are Fisher-non-degenerate, in the sense of~\citep{ding2022global}. First we need to introduce the concept of Fisher information matrix, that for stochastic policies is defined as~\citep{kakade2001natural}:
\begin{equation}
    F(\vtheta) \coloneqq \E_{s,a\sim\upsilon^{\vtheta}}[\nabla_{\vtheta}\log\pi_{\vtheta}(a|s)\nabla_{\vtheta}\log\pi_{\vtheta}(a|s)^\top].
\end{equation}

\begin{lemma}\label{lem:fisher_by_noise}
    Let $\pi_{\vtheta}$ be a white-noise policy (Definition~\ref{defi:add}).
    Under Assumption~\ref{ass:explore}, for all $\vtheta\in\Theta$, $F(\vtheta)\succeq \lambda_{F} I$, where
    $$
        \lambda_F \coloneqq \frac{\lambdaexp}{\sigma_A^2}.
    $$
\end{lemma}
\begin{proof}
Let $\Sigma=\E_{\vepsilon\sim\Phi_{d_{\mathcal{A}}}}[\vepsilon\vepsilon^\top]$ be the covariance matrix of the noise, which by definition has $\lambda_{\max}(\Sigma)\le\sigma_A^2$.
By a simple change of variable and Cramer-Rao's bound:
    \begin{align*}
        F(\vtheta) &= \E_{s,a\sim\upsilon^{\vtheta}}[\nabla_{\vtheta}\log\pi_{\vtheta}(a|s)\nabla_{\vtheta}\log\pi_{\vtheta}(a|s)^\top] \\
        &=\E_{s\sim\upsilon^{\vtheta}}\left[\nabla_{\vtheta}\mu_{\vtheta}(s)\E_{a\sim\pi_{\vtheta}(\cdot|s)}\left[\nabla_{\vepsilon}\log\phi(\widetilde{\vepsilon})|_{\widetilde{\vepsilon}=a-\mu_{\vtheta}(s)}\nabla_{\vepsilon}\log\phi(\widetilde{\vepsilon})^\top|_{\widetilde{\vepsilon}=a-\mu_{\vtheta}(s)}\right]\nabla_{\vtheta}\mu_{\vtheta}(s)^\top\right] \\
        &= \E_{s\sim\upsilon^{\vtheta}}\left[\nabla_{\vtheta}\mu_{\vtheta}(s)\E_{\vepsilon\sim\Phi_{d_\mathcal{A}}}\left[\nabla_{\vepsilon}\log\phi(\vepsilon)\nabla_{\vepsilon}\log\phi(\vepsilon)^\top\right]\nabla_{\vtheta}\mu_{\vtheta}(s)^\top\right] \\
        &\succeq \E_{s\sim\upsilon^{\vtheta}}\left[\nabla_{\vtheta}\mu_{\vtheta}(s)\Sigma^{-1}\nabla_{\vtheta}\mu_{\vtheta}(s)^\top\right] &&\text{(Cramer-Rao)} \\
        &\succeq \frac{1}{\lambda_{\max}(\Sigma)}\E_{s\sim\upsilon^{\vtheta}}\left[\nabla_{\vtheta}\mu_{\vtheta}(s)\nabla_{\vtheta}\mu_{\vtheta}(s)^\top\right] \\
        &\succeq \frac{\lambdaexp}{\sigma_A^2} I.
    \end{align*}
\end{proof}

We can then use Corollary 4.14 by~\citet{yuan2022general}, itself a refinement of Lemma 4.7 by~\citet{ding2022global}, to prove that $\Ja$ enjoys the WGD property.

\wdgByFisher*
\begin{proof}
    Corollary 4.14 by~\citet{yuan2022general} tells us that, under Assumption~\ref{ass:compatible_critic},
    $$
    \Ja^* - \Ja(\vtheta) \le \frac{\xi}{\lambda_F}\|\nabla_{\vtheta}\Ja(\vtheta)\| + \frac{\sqrt{\epsilon_\mathrm{bias}}}{1-\gamma},
    $$
    whenever $F(\vtheta)\succeq \lambda_F I$ and $\E_{a\sim\pi_{\vtheta}(\cdot|s)}[\|\nabla_{\vtheta}\log\pi_{\vtheta}(a|s)\|^2]\le\xi^2$ hold for all $\vtheta\in\Theta$ and $s\in\Ss$. By Lemma~\ref{lem:fisher_by_noise}, and the fact that $\xi=\sqrt{c d_{\As}}\sigma_A^{-1}$ is a valid choice under Assumption~\ref{ass:magic}, the previous display holds with
    $$
        \frac{\xi}{\lambda_F} \le \frac{\sqrt{c d_{\As}}\sigma_A^{-1}}{\lambdaexp\sigma_A^{-2}} = \frac{\sqrt{c d_{\As}}\sigma_A}{\lambdaexp},
    $$
    the proof is concluded by letting $C=\sqrt{c}$, where $c$ is the constant from Assumption~\ref{ass:magic}.
\end{proof}

Finally, we can use the WGD property just established, with its values of $\alpha$ and $\beta$, to prove special cases of Theorems~\ref{thr:gpomgp_sam_comp_wgd} and~\ref{thr:pg_sam_comp_wgdAdaptive}. The key difference with respect to the other sample complexity results presented in the paper is that the amount of noise $\sigma_A$ has an effect on the $\alpha$ parameter of the WGD property.

We first consider the case of a generic $\sigma_A$:

\begin{thr}\label{thr:gpomdp_fisher_generic}
     Consider the GPOMDP algorithm. Under Assumptions~\ref{ass:lipMdp}, \ref{ass:lipPol}, \ref{ass:smoothPol}, \ref{ass:magic}, \ref{ass:explore}, and~\ref{ass:compatible_critic}, with a suitable constant step size, to guarantee $
         \Jd^* - \E[\Jd(\vtheta_K)] \le \epsilon + \frac{\sqrt{\epsilon_\mathrm{bias}}}{1-\gamma} + 3L\sqrt{\da} \sigma_\text{A}
$, where $3L\sqrt{\da} \sigma_\text{A} = O(\sqrt{\da}\sigma_\text{A}(1-\gamma)^{-2})$ the sample complexity is at most:
     \begin{align}
         NK = \widetilde{O} \left( \frac{\da^4}{\lambdaexp^4(1-\gamma)^5\epsilon^3}\right).
 \end{align}
 Furthermore, under Assumption~\ref{ass:smoothMdp}, the same guarantee is obtained with a sample complexity of at most:
          \begin{align}
         NK = \widetilde{O} \left( \frac{\da^3\sigma_{\text{A}}^2}{\lambdaexp^4(1-\gamma)^6\epsilon^3}\right).
 \end{align}
\end{thr}
\begin{proof}
    By Theorem~\ref{thr:gpomgp_sam_comp_wgd} and Lemma~\ref{lem:wgd_by_fisher}.
\end{proof}

The first bound seem to have no dependence on $\sigma_A$. However, a complex dependence is hidden in $\lambdaexp^4$.
Also, it may seem that $\sigma_A\simeq 0$ is a good choice, especially for the second bound. However, $\lambdaexp$ can be very large (or infinite) for a (quasi-)deterministic policy.

If we instead set $\sigma_A$ as in Section~\ref{sec:convergence} in order to converge to a good deterministic policy (which, of course, completely ignores the complex dependencies of $\lambdaexp$ and $\epsilon_\mathrm{bias}$ on $\sigma_A$), we obtain the following:

\begin{thr} \label{thr:gpomdp_fisher_optimal}
    Consider the GPOMDP algorithm. Under Assumptions~\ref{ass:lipMdp}, \ref{ass:lipPol}, \ref{ass:smoothPol}, \ref{ass:magic}, \ref{ass:explore}, and~\ref{ass:compatible_critic} with a suitable constant step size and {setting $\sigma_{\text{A}} = \frac{\epsilon}{6L\sqrt{\da}} = O(\epsilon (1-\gamma)^2 \da^{-1/2})$}, to guarantee $\Jd^* - \E[\Jd(\vtheta_K)] \le \epsilon + \frac{\sqrt{\epsilon_\mathrm{bias}}}{1-\gamma}$ the sample complexity is at most:
     \begin{align}
         NK = \widetilde{O} \left( \frac{\da^6}{\lambdaexp^4(1-\gamma)^{13}\epsilon^3}\right).
    \end{align}
    Furthermore, under Assumption~\ref{ass:smoothMdp}, the same guarantee is obtained with a sample complexity of at most:
    \begin{align}\label{eq:broken}
         NK = \widetilde{O} \left( \frac{\da^4}{\lambdaexp^4(1-\gamma)^{10}\epsilon}\right).
    \end{align}
 \end{thr}
 \begin{proof}
     By Theorem~\ref{thr:pg_sam_comp_wgdAdaptive} and Lemma~\ref{lem:wgd_by_fisher}.
 \end{proof}

The \emph{apparently} better sample complexity w.r.t. Theorem~\ref{thr:pg_sam_comp_wgdAdaptive} is easily explained: using a small $\sigma$ makes the $\alpha$ parameter of WGD from Lemma~\ref{lem:wgd_by_fisher} smaller \emph{if we ignore the effect of $\lambdaexp$}, and smaller $\alpha$ yields faster convergence. However, Equation~\eqref{eq:broken} clearly shows that $\lambdaexp$ cannot be ignored. In particular, $\lambdaexp$ must be $O(\sigma_A^{1/4})$ not to violate the classic $\Omega(\epsilon^{-2})$ lower bound on the sample complexity~\citep{azar2013minimax}. This may be of independent interest.

\section{Assumptions' Implications}\label{app:ass_impl}

\begin{restatable}[$L$ and $L_J$ characterization]{lemma}{assImpl}\label{thr:assImpl}
    Assumption~\ref{ass:lipMdp} implies Assumption~\ref{ass:Jd_lip_ns} with:
   \begin{align}
       & L_t \le \frac{\gamma^{k+1}-\gamma^T}{1-\gamma} L_{ p}  R_{\max} + \gamma^k L_r, \\
       & L \le \frac{\gamma (1-\gamma^T) }{(1-\gamma)^2} L_{ p}  R_{\max} + \frac{1 - \gamma^T}{1 - \gamma} L_r \le  \frac{\gamma L_p R_{\max}}{(1-\gamma)^2}   + \frac{L_r}{1 - \gamma}  .
   \end{align}%
   Assumption~\ref{ass:lipMdp} and~\ref{ass:lipPol} imply Assumption~\ref{ass:Jd_lip} with $L_J \le  L L_{{\mu}}$.
\end{restatable}

\begin{proof}
    In AB exploration, we introduce the following convenient expression for the trajectory density function having fixed a sequence of noise $\underline{\vepsilon}$:
    \begin{align}
     p_{\text{D}}(\tau; \underline{\bm{\mu}} + \underline{\vepsilon}) = \rho_0(s_{\tau, 0}) \prod_{t=0}^{T-1}  p(s_{\tau, t+1} | s_{\tau, t}, \bm{\mu}_t(s_{\tau, t}) + \vepsilon_t).
    \end{align}
    Furthermore, we denote with $ p_{\text{D}}(\tau_{0:l}; \underline{\bm{\mu}} + \underline{\vepsilon})$ the density function of a trajectory prefix of length $l$:
    \begin{align}
       p_{\text{D}}(\tau_{0:l}; \underline{\bm{\mu}} + \underline{\vepsilon}) = \rho_0(s_{\tau, 0}) \prod_{t=0}^{l-1}  p(s_{\tau, t+1} | s_{\tau, t}, \bm{\mu}_t(s_{\tau, t}) + \vepsilon_t).  
    \end{align}
    Let us decompose $\underline{\bm{\mu}}' = \underline{\bm{\mu}}' + \underline{\vepsilon}$. We have:
    \begin{align*}
        \Jd(\underline{\bm{\mu}}') & = \int_{\tau}  p_{\text{D}}(\tau; \underline{\bm{\mu}} + \underline{\vepsilon})  \sum_{t=0}^{T-1} \gamma^t r(s_t, \bm{\mu}_t(s_t)+ \vepsilon_t ) \de \tau \\
        & = \underbrace{\sum_{t=0}^{T-1}  \int_{\tau_{0:t}}  p_{\text{D}}(\tau_{0:t}; \underline{\bm{\mu}} + \underline{\vepsilon})  \gamma^t r(s_t, \bm{\mu}_t(s_t)+ \vepsilon_t ) \de \tau_{0:t}}_{\eqqcolon f(\underline{\vepsilon})}.
    \end{align*}
    Note that given the definition of $f(\underline{\vepsilon})$, we have that $ f(\underline{\mathbf{0}_{\da}})= \Jd(\underline{\bm{\mu}})$.  Using Taylor expansion, we have for $\widetilde{\underline{\vepsilon}} = x{\underline{\vepsilon}}$, for some $x \in [0,1]$:
    \begin{align*}
        \Jd(\underline{\bm{\mu}}')  &= f(\underline{\vepsilon}) \\
        &= f(\underline{\mathbf{0}_{\da}}) + \underline{\vepsilon}^\top  \nabla_{\underline{\vepsilon}}  f(\underline{\vepsilon})\rvert_{\underline{\vepsilon} = \widetilde{\underline{\vepsilon}}}\\
        &= \Jd(\underline{\bm{\mu}}) + \sum_{t=0}^{T-1} \vepsilon_t^\top \nabla_{\vepsilon_t} f(\underline{\vepsilon})\rvert_{\underline{\vepsilon} = \widetilde{\underline{\vepsilon}}}\\
        &\le \Jd(\underline{\bm{\mu}})+ \sum_{t=0}^{T-1} \|\vepsilon_t\|_2   \|\nabla_{\vepsilon_t} f(\underline{\vepsilon})\rvert_{\underline{\vepsilon} = \widetilde{\underline{\vepsilon}}}\|_2.
    \end{align*}
    We want to find a bound for the $ \|\nabla_{\vepsilon_t} f(\vepsilon)\rvert_{\underline{\vepsilon} = \widetilde{\underline{\vepsilon}}}\|_2^2$ which is different for every $t$. This will result in the Lipschitz constant $L_t$. 
      We have for $k \in \dsb{0,T-1}$:
    \begin{align*}
        \left\| \nabla_{\underline{\vepsilon}_k} f(\underline{\vepsilon}) \right\|_2 &\leq \E_{\tau \sim p_{\text{D}}(\cdot; \underline{\bm{\mu}} + \underline{\vepsilon})} \left[  \sum_{t=0}^{T-1} \left\| \nabla_{{\vepsilon}_k} \log p_{\text{D}}(\tau_{0:t}; \underline{\bm{\mu}} + \underline{\vepsilon}) \right\|_2 \gamma^t |r(s_t, \bm{\mu}_{t}(s_t)+ {\vepsilon}_t )| +  \sum_{t=0}^{T-1} \gamma^t \|\nabla_{{\vepsilon}_k} r(s_t, \bm{\mu}_{t}(s_t)+ {\vepsilon}_t )\|_2 \right] \\
        &  \leq \E_{\tau \sim p_{\text{D}}(\cdot; \underline{\bm{\mu}} + \underline{\vepsilon})} \left[  \sum_{t=0}^{T-1}  \gamma^t |r(s_t, \bm{\mu}_{t}(\vs_t)+ {\vepsilon}_t ) |\sum_{l=0}^{t-1} \| \nabla_{{\vepsilon}_k}  p(s_{l+1}|\vs_l, \bm{\mu}_{l}(s_l)+{\vepsilon}_l)  \|_2  +  \sum_{t=0}^{T-1} \gamma^t \|\nabla_{{\vepsilon}_k} r(s_t, \bm{\mu}_{t}(s_t)+ {\vepsilon}_t )\|_2 \right] \\
        &  = \E_{\tau \sim p_{\text{D}}(\cdot; \underline{\bm{\mu}} + \underline{\vepsilon})} \left[  \sum_{t=k+1}^{T-1}  \gamma^t |r(s_t, \bm{\mu}_{t}(s_t)+ {\vepsilon}_t ) | \| \nabla_{{\vepsilon}_k}  p(s_{k+1}|s_k, \bm{\mu}_{k}(s_k)+{\vepsilon}_k)  \|_2  +  \gamma^k \|\nabla_{{\vepsilon}_k} r(s_k,\bm{\mu}_{k}(s_k)+ {\vepsilon}_k )\|_2 \right] \\
        &\leq \sum_{t=k+1}^{T-1} \gamma^t  R_{\max}  L_{ p} + \gamma^k L_r   \\
        & = \frac{\gamma^{k+1}-\gamma^T}{1-\gamma} L_{ p}  R_{\max} + \gamma^k L_r \eqqcolon L_k.
    \end{align*}
    Thus, we have:
    \begin{align*}
        \sum_{k=0}^{T-1}L_t &= \sum_{k=0}^{T-1} \frac{\gamma^{k+1}-\gamma^T}{1-\gamma} L_{ p}  R_{\max} + \gamma^k L_r \\ 
        &= \frac{\gamma + \gamma^{1 + T} (T-1) - T \gamma^T }{(1-\gamma)^2} L_{ p}  R_{\max} + \frac{1 - \gamma^T}{1 - \gamma} L_r \\
        & \le \frac{\gamma(1-\gamma^T)}{(1-\gamma)^2} L_p R_{\max} + \frac{(1-\gamma^T)}{1-\gamma} L_r \eqqcolon L.
    \end{align*}
    
 For the PB exploration, we consider the trajectory density function:
    \begin{align}
        p_{\text{A}}(\tau;\vtheta+\vepsilon) =  \rho_0(s_{\tau, 0}) \prod_{t=0}^{T-1}  p(s_{\tau, t+1} | s_{\tau, t}, \bmu[\vtheta+\vepsilon](s_{\tau, t})),
    \end{align}
    and the corresponding version for a trajectory prefix:
    \begin{align}
         p_{\text{D}}(\tau_{0:l}; \vtheta+\vepsilon) = \rho_0(s_{\tau, 0}) \prod_{t=0}^{l-1}  p(s_{\tau, t+1} | s_{\tau, t}, \bmu[\vtheta+\vepsilon](s_{\tau, t}) ).  
    \end{align}
    With such a notation, we can write the $\vtheta' = \vtheta+\vepsilon$ index as follows:

\begin{align*}
\Jd(\vtheta') = \int_{\tau} p_{\text{D}}(\tau;\vtheta+\vepsilon) \sum_{t=0}^{T-1} \gamma^t r(s_t, \bm{\mu}_{\vtheta+ \vepsilon}(s_t) ) \de \tau = \underbrace{\int_{\tau}  \sum_{t=0}^{T-1} p_{\text{D}}(\tau_{0:t};\vtheta + \vepsilon) \gamma^t r(s_t, \bm{\mu}_{\vtheta + \vepsilon}(s_t))}_{\eqqcolon g(\vepsilon)} .
\end{align*}
We recall that $ g(\mathbf{0}_{\dt}) = \Jd(\vtheta)$.
By using Taylor expansion, where $\widetilde{\vepsilon} = x\vepsilon$ for some $x \in [0,1]$:
\begin{align}
    \Jd(\vtheta') & = g(\vepsilon) \\
    & = g(\mathbf{0}_{\dt}) +  \vepsilon^\top\nabla_{\vepsilon}g(\vepsilon)\rvert_{\vepsilon=\widetilde{\vepsilon}}\\
    & \le  \Jd(\vtheta) + \|\vepsilon\|_2  \|\nabla_{\vepsilon}g(\vepsilon)\rvert_{\vepsilon=\widetilde{\vepsilon}}\|_2.
\end{align}
We now bound the norm of the gradient:
\begin{align*}
\| \nabla_{\vepsilon} g(\vepsilon) \|_2 &\leq \E_{\tau \sim p_{\text{A}}(\cdot;\bm{\mu}_{\vtheta + \vepsilon})} \left[  \sum_{t=0}^{T-1} \| \nabla_{\vepsilon} \log p(\tau_{0:t};{{\vtheta + \vepsilon}})\|_2 \gamma^t |r(s_t, \bm{\mu}_{\vtheta+ \vepsilon}(s_t) )| +  \sum_{t=0}^{T-1} \gamma^t \|\nabla_{\vepsilon} r(s_t, \bm{\mu}_{\bm{\theta}+ \vepsilon}(s_t) )\|_2 \right] \\
& \leq \E_{\tau \sim p_{\text{A}}(\cdot;\bm{\mu}_{\vtheta + \vepsilon})} \Bigg[  \sum_{t=0}^{T-1}  \gamma^t |r(s_t, \bm{\mu}_{\vtheta+ \vepsilon}(s_t) ) |\sum_{l=0}^{t-1} \| \nabla_{\ba} \log p(s_{l+1}|s_l,\ba) \rvert_{\ba = \bm{\mu}_{\vtheta+\epsilon}(s_t)} \|_2 \| \nabla_{\vepsilon} \bm{\mu}_{\vtheta + \vepsilon}(s_l) \|_2 \\
& \qquad  +  \sum_{t=0}^{T-1} \gamma^t \|\nabla_{\ba} r(s_t, \ba ) \rvert_{\ba=\bm{\mu}_{\vtheta+ \vepsilon}(s_t)}\|_2  \| \nabla_{\vepsilon} \bm{\mu}_{\bm{\theta} + \vepsilon}(s_t) \|_2 \Bigg] \\
&\leq \sum_{t=0}^{T-1}  \gamma^t R_{\max} t L_{ p} L_{{\mu}} + \frac{1 - \gamma^T}{1 - \gamma}  L_r L_{{\mu}} \\
		&  = \frac{\gamma + \gamma^{1 + T} (T-1) - T \gamma^T }{(1-\gamma)^2} L_{ p} L_{{\mu}} R_{\max} + \frac{1 - \gamma^T}{1 - \gamma} L_rL_{{\mu}} = L L_{{\mu}}.
\end{align*}
\end{proof}

\begin{ass}[Smooth $\Jd$ w.r.t. parameter $\vtheta$]\label{ass:smoothJ}  $\Jd$ is $L_2$-LS \wrt parameter $\vtheta$, i.e., for every $\vtheta,\vtheta' \in \Theta$, we have:
\begin{align}
    \| \nabla_{\vtheta}\Jd(\vtheta')-\nabla_{\vtheta}\Jd(\vtheta)\|_2 \le L_2 \|\vtheta' - \vtheta \|_2.
\end{align}
\end{ass}

\begin{lemma}[$L_2$ Characterization] \label{lem:L_2_characterization}
    Assumptions~\ref{ass:lipMdp}, \ref{ass:lipPol}, \ref{ass:smoothMdp}, and \ref{ass:smoothPol} imply Assumption~\ref{ass:smoothJ}  with
    \begin{align*}
        L_2 \leq \frac{\gamma (1 + \gamma) L_{\mu}^2 L_{p}^2 R_{\max}}{(1-\gamma)^3} + \frac{\gamma(2 L_{\mu}^2 L_{p} L_{r} + L_{2, \mu} L_{2, p} R_{\max})}{(1-\gamma)^2} + \frac{L_{2, \mu} L_{2, r}}{1 - \gamma}.
    \end{align*}
\end{lemma}
\begin{proof}
    It suffices to find a bound to the quantity $\left\| \nabla_{\vtheta}^2 \Jd(\vtheta) \right\|_2$, for a generic $\vtheta \in \Theta$.
    Notice that in the following we use the notation $\tau_{0:l}$ to refer to a trajectory of length $l$.
    Recalling that:
    \begin{align*}
        \Jd(\vtheta) = \E_{\tau \sim p_{\text{D}}(\cdot | \vtheta)} \left[ \sum_{t=0}^{T-1} \gamma^t r(s_t, \mu_{\vtheta}(s_t)) \right],
    \end{align*}   
    we have what follows:
    \begin{align*}
        \nabla_{\vtheta}^2 \Jd(\vtheta) &= \nabla_{\vtheta}^2 \E_{\tau \sim p_{\text{D}}(\cdot | \vtheta)} \left[ \sum_{t=0}^{T-1} \gamma^t r(s_t, \mu_{\vtheta}(s_t)) \right] \\
        &= \nabla_{\vtheta}^2 \int_{\tau} p_{\text{D}}(\tau, \vtheta) \sum_{t=0}^{T-1} \gamma^t r(s_t, \mu_{\vtheta}(s_t)) \de \tau \\
        &= \sum_{t=0}^{T-1} \nabla_{\vtheta}^2 \int_{\tau_{0:t}} p_{\text{D}}(\tau_{0:t}, \vtheta) \gamma^t r(s_t, \mu_{\vtheta}(s_t)) \de \tau_{0:t} \\
        &= \sum_{t=0}^{T-1} \nabla_{\vtheta} \int_{\tau_{0:t}} p_{\text{D}}(\tau_{0:t}, \vtheta) \left( \nabla_{\vtheta} \log p_{\text{D}}(\tau_{0:t}, \vtheta) \gamma^t r(s_t, \mu_{\vtheta}(s_t)) + \gamma^t \nabla_{\vtheta}r(s_t, \mu_{\vtheta}(s_t)) \right) \de \tau_{0:t} \\
        &= \sum_{t=0}^{T-1} \E_{\tau_{0:t} \sim p_{\text{D}}(\cdot | \vtheta)} \Bigg[ \nabla_{\vtheta} \log p_{\text{D}}(\tau_{0:t}, \vtheta) \left( \nabla_{\vtheta} \log p_{\text{D}}(\tau_{0:t}, \vtheta) \gamma^t r(s_t, \mu_{\vtheta}(s_t)) + \gamma^t \nabla_{\vtheta}r(s_t, \mu_{\vtheta}(s_t)) \right) \\
        &\qquad + \nabla_{\vtheta}^2 \log p_{\text{D}}(\tau_{0:t}, \vtheta) \gamma^t r(s_t, \mu_{\vtheta}(s_t)) + \nabla_{\vtheta} \log p_{\text{D}}(\tau_{0:t}, \vtheta) \gamma^t \nabla_{\vtheta} r(s_t, \mu_{\vtheta}(s_t)) + \gamma^t \nabla_{\vtheta}^2 r(s_t, \mu_{\vtheta}(s_t)) \Bigg].
    \end{align*}
    Now that we have characterized $\nabla_{\vtheta}^2 \Jd(\vtheta)$, we can consider its norm by applying the assumptions in the statement, obtaining the following result:
    \begin{align*}
        &\left\| \nabla_{\vtheta}^2 \Jd(\vtheta) \right\| \\
        &\leq \sum_{t=0}^{T-1} \E_{\tau_{0:t} \sim p_{\text{D}}(\cdot | \vtheta)} \Bigg[ \left\| \nabla_{\vtheta} \log p_{\text{D}}(\tau_{0:t}, \vtheta) \right\|_2 \left( \left\| \nabla_{\vtheta} \log p_{\text{D}}(\tau_{0:t}, \vtheta) \right\|_2 \gamma^t \left| r(s_t, \mu_{\vtheta}(s_t)) \right| + \gamma^t \left\| \nabla_{\vtheta}r(s_t, \mu_{\vtheta}(s_t)) \right\|_2 \right) \\
        &\qquad + \left\| \nabla_{\vtheta}^2 \log p_{\text{D}}(\tau_{0:t}, \vtheta) \right\|_2 \gamma^t \left| r(s_t, \mu_{\vtheta}(s_t)) \right| + \left\| \nabla_{\vtheta} \log p_{\text{D}}(\tau_{0:t}, \vtheta) \right\|_2 \gamma^t \left\| \nabla_{\vtheta} r(s_t, \mu_{\vtheta}(s_t)) \right\|_2 + \gamma^t \left\| \nabla_{\vtheta}^2 r(s_t, \mu_{\vtheta}(s_t)) \right\|_2 \Bigg] \\
        &\leq \sum_{t=0}^{T-1} L_{\mu}^2 L_{p}^2 R_{\max} t^2 \gamma^t + (2 L_{\mu}^2 L_{p} L_{r} + L_{2, \mu} L_{2, p} R_{\max}) t \gamma^t + L_{2, \mu} L_{2, r} \gamma^t \\
        &\leq L_{\mu}^2 L_{p}^2 R_{\max} \gamma \frac{1 + \gamma - T^2 \gamma^{T-1} + (2(T-1)^2 +2(T-1) - 1)\gamma^T - (T-1)^2 \gamma^{T+1}}{(1-\gamma)^3} \\
        &\qquad + (2 L_{\mu}^2 L_{p} L_{r} + L_{2, \mu} L_{2, p} R_{\max}) \gamma \frac{1- T \gamma^{T-1} + (T-1) \gamma^T}{(1-\gamma)^2} + L_{2, \mu} L_{2, r} \frac{1-\gamma^T}{1 - \gamma} \\
        &\leq L_{\mu}^2 L_{p}^2 R_{\max} \gamma \frac{1 + \gamma - \gamma^T}{(1-\gamma)^3} + (2 L_{\mu}^2 L_{p} L_{r} + L_{2, \mu} L_{2, p} R_{\max}) \gamma \frac{1- \gamma^T}{(1-\gamma)^2} + L_{2, \mu} L_{2, r} \frac{1-\gamma^T}{1 - \gamma} \\
        &\leq \frac{\gamma (1 + \gamma) L_{\mu}^2 L_{p}^2 R_{\max}}{(1-\gamma)^3} + \frac{\gamma(2 L_{\mu}^2 L_{p} L_{r} + L_{2, \mu} L_{2, p} R_{\max})}{(1-\gamma)^2} + \frac{L_{2, \mu} L_{2, r}}{1 - \gamma}.
    \end{align*}
\end{proof}

\begin{lemma}\label{lemma:boundsMagic}
    Let $\pi_{\vtheta}$ be a white noise-based policy. Under Assumption~\ref{ass:lipPol},~\ref{ass:smoothPol}, and ~\ref{ass:magic} it holds that for every $s \in \mathcal{S}$:
     \begin{enumerate}[leftmargin=*, noitemsep, label=($\roman*$), topsep=-2pt]
        \item  $\E_{\va \sim \pi_{\vtheta}(\va|\vs)} [\| \nabla_{\vtheta} \log \pi_{\vtheta}(\va|\vs)\|_2^2] \le   c\da \sigma^{-2}_{\text{A}} L_\mu^2$;
        \item  $\E_{\va \sim \pi_{\vtheta}(\va|\vs)} [\| \nabla_{\vtheta}^2 \log \pi_{\vtheta}(\va|\vs)\|_2] \le  c \sigma_{\text{A}}^{-2} L_\mu^2 +c \sqrt{\da} \sigma_{\text{A}}^{-1} L_{2,\mu} $.
    \end{enumerate}
\end{lemma}

\begin{proof}
    Since  $\pi_{\vtheta}$ is a white noise-based policy, we have that $ \pi_{\vtheta}(\va|\vs) = \phi(\va - \bm{\mu}_{\vtheta}(\vs))$. Consequently, we have:
    \begin{align}
        & \nabla_{\vtheta} \log \pi_{\vtheta}(\va|\vs)  = \nabla_{\vtheta} \log\phi(\va - \bm{\mu}_{\vtheta}(\vs)) =  - \nabla_{\vtheta} \bm{\mu}_{\vtheta}(\vs) \nabla_{\vepsilon} \log\phi(\vepsilon)\rvert_{\vepsilon = \va - \bm{\mu}_{\vtheta}(\vs)} , \\
        & \nabla_{\vtheta}^2 \log \pi_{\vtheta}(\va|\vs) = \nabla_{\vtheta}^2 \log\phi(\va - \bm{\mu}_{\vtheta}(\vs)) = \nabla_{\vtheta} \bm{\mu}_{\vtheta} \nabla_{\vepsilon}^2 \log\phi(\vepsilon)\rvert_{\vepsilon = \va - \bm{\mu}_{\vtheta}(\vs)} \nabla_{\vtheta} \bm{\mu}_{\vtheta}^\top -\nabla_{\vtheta}^2 \bm{\mu}_{\vtheta}\nabla_{\vepsilon} \log\phi(\vepsilon)\rvert_{\vepsilon = \va - \bm{\mu}_{\vtheta}(\vs)}  .
    \end{align}
    Thus, recalling that $\va - \bm{\mu}_{\vtheta}(\vs) \sim  \Phi_{\da}$ and using the Lipschitzinity and smoothness of $\bm{\mu}_{\vtheta}$, we have:
    \begin{align}
         \E_{\va \sim \pi_{\vtheta}(\va|\vs)} [\| \nabla_{\vtheta} \log \pi_{\vtheta}(\va|\vs)\|_2^2]  & = \E_{\va \sim \pi_{\vtheta}(\va|\vs)} [\| \nabla_{\vtheta} \bm{\mu}_{\vtheta}(\vs)\nabla_{\vepsilon} \log\phi(\vepsilon)\rvert_{\vepsilon = \va - \bm{\mu}_{\vtheta}(\vs)}  \|_2^2] \\
        & \le L_\mu^2 \E_{\vepsilon \sim \Phi_{\dt}} [\|-\nabla_{\vepsilon} \log\phi(\vepsilon)\|_2^2] \le c\dt \sigma^2_{\text{A}}L_\mu^2,\\
         \E_{\va \sim \pi_{\vtheta}(\va|\vs)} [\| \nabla_{\vtheta}^2\log \pi_{\vtheta}(\va|\vs)\|_2]  &  = \E_{\va \sim \pi_{\vtheta}(\va|\vs)} [\|\nabla_{\vtheta} \bm{\mu}_{\vtheta} \nabla_{\vepsilon}^2 \log\phi(\vepsilon)\rvert_{\vepsilon = \va - \bm{\mu}_{\vtheta}(\vs)} \nabla_{\vtheta} \bm{\mu}_{\vtheta}^\top -\nabla_{\vtheta}^2 \bm{\mu}_{\vtheta}\nabla_{\vepsilon} \log\phi(\vepsilon)\rvert_{\vepsilon = \va - \bm{\mu}_{\vtheta}(\vs)}\|_2]\\
         & \le L_\mu^2  \E_{\vepsilon\sim \Phi_{\da}} [\|\nabla_{\vepsilon}^2 \log\phi(\vepsilon)\|_2] + L_{2,\mu} \E_{\vepsilon\sim \Phi_{\da}} [\|\nabla_{\vepsilon} \log\phi(\vepsilon)\|_2]   \\      
         & \le c \sigma_{\text{A}}^{-2} L_\mu^2 +c \sqrt{\da} \sigma_{\text{A}}^{-1} L_{2,\mu} .
    \end{align}
\end{proof}

\begin{lemma}\label{lemma:boundsMagic2}
    Let $\nu_{\vtheta}$ be a white noise-based hyperpolicy. Under Assumption~\ref{ass:magic}, it holds that:
     \begin{enumerate}[leftmargin=*, noitemsep, label=($\roman*$), topsep=-2pt]
        \item  $\E_{\vtheta' \sim \nu_{\vtheta}} [\| \nabla_{\vtheta} \log \nu_{\vtheta}(\vtheta')\|_2^2] \le c \dt \sigma^{-2}_{\text{P}}$;
        \item  $\E_{\vtheta' \sim \nu_{\vtheta}} [\| \nabla_{\vtheta}^2 \log \nu_{\vtheta}(\vtheta')\|_2] \le c  \sigma^{-2}_{\text{P}}$.
    \end{enumerate}
\end{lemma}

\begin{proof}
    Since  $\nu_{\vtheta}$ is a white noise-based hyperpolicy, we have that $ \nu_{\vtheta}(\vtheta') = \phi(\vtheta'-\vtheta)$. Consequently, we have:
    \begin{align}
        & \nabla_{\vtheta} \log \nu_{\vtheta}(\vtheta') = \nabla_{\vtheta} \log\phi(\vtheta'-\vtheta) =  -\nabla_{\vepsilon} \log\phi(\vepsilon)\rvert_{\vepsilon = \vtheta'-\vtheta}, \\
        & \nabla_{\vtheta}^2 \log \nu_{\vtheta}(\vtheta') = \nabla_{\vtheta}^2 \log\phi(\vtheta'-\vtheta) =  \nabla_{\vepsilon}^2 \log\phi(\vepsilon)\rvert_{\vepsilon = \vtheta'-\vtheta}.
    \end{align}
    Thus, recalling that $\vtheta' - \vtheta \sim  \Phi_{\dt}$
    \begin{align}
        &\E_{\vtheta' \sim \nu_{\vtheta}} [\| \nabla_{\vtheta} \log \nu_{\vtheta}(\vtheta')\|_2^2]  = \E_{\vtheta' \sim \nu_{\vtheta}} [\|  \nabla_{\vepsilon} \log\phi(\vepsilon)\rvert_{\vepsilon = \vtheta'-\vtheta}\|_2^2] = \E_{\vepsilon \sim \Phi_{\dt}} [\|  \nabla_{\vepsilon} \log\phi(\vepsilon)\|_2^2] \le c\dt \sigma^2_{\text{P}},\\
        & \E_{\vtheta' \sim \nu_{\vtheta}} [\| \nabla_{\vtheta}^2 \log \nu_{\vtheta}(\vtheta')\|_2]  = \E_{\vtheta' \sim \nu_{\vtheta}} [\|  \nabla_{\vepsilon}^2 \log\phi(\vepsilon)\rvert_{\vepsilon = \vtheta'-\vtheta}\|_2] = \E_{\vepsilon \sim \Phi_{\dt}} [\|  \nabla_{\vepsilon}^2 \log\phi(\vepsilon)\|_2] \le c \sigma^2_{\text{P}}.
    \end{align}
\end{proof}

\section{General Convergence Analysis under Weak Gradient Domination} \label{apx:general_convergence}
In this section, we provide the theoretical guarantees on the convergence to the global optimum of a generic stochastic first-order optimization algorithm $\mathfrak{A}$ (\eg policy gradient employing either AB or PB exploration).
Let $\vtheta$ be the parameter vector optimized by $\mathfrak{A}$, and let $\Theta = \mathbb{R}^{\dt}$ be the parameter space.
The objective function that $\mathfrak{A}$ aims at optimizing is $J: \Theta \xrightarrow{} \mathbb{R}$, which is a generic function taking as argument a parameter vector $\vtheta \in \Theta$ and mapping it into a real value.
Examples of objective functions of this kind are $\Jd$, $\Ja$, or $\Jp$, which are all defined in Section~\ref{sec:preliminaries}.
The algorithm $\mathfrak{A}$ is run for $K$ iterations and it updates directly the parameter vector $\vtheta \in \Theta$.
At the $k$-th iteration, the update is:
\begin{align*}
    \vtheta_{k + 1} \xleftarrow{} \vtheta_{k} + \zeta_{k} \widehat{\nabla}_{\vtheta} J(\vtheta_k),
\end{align*}
where $\zeta_k$ is the step size, $\vtheta_{k}$ is the parameter configuration at the $k$-th iteration, and $\widehat{\nabla}_{\vtheta} J(\vtheta_{k})$ is an \emph{unbiased} estimate of $\nabla_{\vtheta} J(\vtheta_{k})$ computed from a batch $\mathcal{D}_k$ of $N$ \emph{samples}.
In the following, we refer to $N$ as \textit{batch size}.
Examples of unbiased gradient estimators are the ones employed by GPOMDP and PGPE, which can be found in Section~\ref{sec:preliminaries}. For GPOMDP, samples are trajectories; for PGPE, parameter-trajectory pairs.
In what follows, we refer to the optimal parameter configuration as $\vtheta^* \in \argmax_{\vtheta \in \Theta} J(\vtheta)$.
For the sake of simplicity, we will shorten $J(\vtheta^*)$ as $J^*$.
Given an optimality threshold $\delta \geq 0$, we are interested in assessing the \emph{last-iterate} convergence guarantees:
\begin{align*}
    J^* - \E \left[ J(\vtheta_{K}) \right] \leq \delta,
\end{align*}
where the expectation is taken over the stochasticity of the learning process.

\begin{thr}\label{thr:gen_conv_new}
    
Under Assumptions~\ref{ass:J_wgd}, \ref{ass:J_gen_conv}, and \ref{ass:J_var_gen_conv}, running the Algorithm $\mathfrak{A}$ for $K>0$ iterations with a batch size of $N>0$ trajectories in each iteration with the constant learning rate $\zeta$ fulfilling:
	\begin{align*}
		\zeta \le \min \left\{ \frac{1}{\Lhgen}, \frac{1}{\mu \max\{0,J^* - J(\vtheta_0) - \betagen\}}, \left(\frac{N}{\Lhgen\Vgen\mu}\right)^{1/3} \right\}
  \end{align*}
	where $\mu = \frac{1}{\alphagen^2}$. Then, it holds that:
	\begin{align*}
		J^* - \E[J(\vtheta_K)] \le \betagen +\left(1- \frac{1}{2} \sqrt{\frac{\mu \zeta^3 \Lhgen \Vgen}{N}}\right)^{K} \max\left\{0,J^* - J(\vtheta_0) - \betagen \right\} + \sqrt{\frac{\Lhgen \Vgen \zeta}{\mu N}} .
	\end{align*}
	In particular, for sufficiently small $\epsilon>0$, setting $\zeta = \frac{\epsilon^2\mu N}{4 \Lhgen \Vgen}$, the following total number of samples is sufficient to ensure that $J(\vtheta^*) - \E[J(\vtheta_K)] \le \betagen + \epsilon$:
 \begin{align}
     KN \ge \frac{16 \Lhgen \Vgen}{\epsilon^3 \mu^2 } \log \frac{\max\left\{0,J^* - J(\vtheta_0) - \betagen \right\}}{\epsilon}.
 \end{align}

\end{thr}

\begin{proof}
     Before starting the proof, we need a preliminary result that immediately follows from Assumption~\ref{ass:J_wgd}, by rearranging:
 
	\begin{align}\label{eq:genconv_eq_1}
		\frac{1}{\alphagen} \max\left\{0, J(\vtheta^*) - \betagen - J(\vtheta)\right\} \leq \| \nabla_{\vtheta} J(\vtheta) \|_2,
	\end{align}
    and we will use the notation $\widetilde{J}(\vtheta^*) \coloneqq J(\vtheta^*)-\betagen$ and $\mu = \alphagen^{-2}$. Note that
     $\widetilde{J}(\vtheta^*)-J(\vtheta)$ can be negative.
	Considering a $k \in \dsb{K}$, it follows that:
	\begin{align*}
		\widetilde{J}(\vtheta^*) - J(\vtheta_{k+1}) 
		&=  \widetilde{J}(\vtheta^*) - J(\vtheta_{k}) - ( J(\vtheta_{k+1}) - J(\vtheta_{k})) \\
		&\leq \widetilde{J}(\vtheta^*) - J(\vtheta_{k}) - \left<  \vtheta_{k+1} - \vtheta_{k}, \nabla_{\vtheta}J(\vtheta_{k}) \right> + \frac{\Lhgen}{2} \|  \vtheta_{k+1} - \vtheta_{k} \|_2^2 \\
		&\leq \widetilde{J}(\vtheta^*) - J(\vtheta_{k}) - \zeta_k \left< \widehat{\nabla}_{\vtheta}J(\vtheta_{k})  , \nabla_{\vtheta}J(\vtheta_{k}) \right> + \frac{\Lhgen}{2} \zeta_k^2 \|  \widehat{\nabla}_{\vtheta}J(\vtheta_{k}) \|_2^2,
	\end{align*}
    where the first inequality follows by applying the Taylor expansion with Lagrange remainder and exploiting Assumption~\ref{ass:J_gen_conv}, and the last inequality follows from the fact that the parameter update is $\vtheta_{k+1} \leftarrow \vtheta_{k} + \zeta_k \widehat{\nabla}_{\vtheta} J(\vtheta_{k})$. 

    In the following, we use the shorthand notation $\E_k[\cdot]$ to denote the conditional expectation w.r.t. the history up to the $k$-th iteration \emph{not} included. More formally, let $\mathcal{F}_k=\sigma\left(\vtheta_0,\mathcal{D}_0,\mathcal{D}_1,\dots,\mathcal{D}_k\right)$ be the $\sigma$-algebra encoding all the stochasticity up to iteration $k$ included. Note that all the stochasticity comes from the samples (except from the initial parameter $\vtheta_0$, which may be randomly initialized), and that $\vtheta_{k}$ is $\mathcal{F}_{k-1}$-measurable, that is, deterministically determined by the realization of the samples collected in the first $k-1$ iterations. Then, $\E_k[\cdot] \coloneqq \E[\cdot|\mathcal{F}_{k-1}]$. We will make use of the basic facts $\E[\cdot]=\E[\E_k[\cdot]]$ and $\E_k[X] = X$ for $\mathcal{F}_{k-1}$-measurable $X$. The variance of $\widehat{\nabla} J(\vtheta_k)$ must be always understood as conditional on $\mathcal{F}_{k-1}$.
	Now, 
    for any $k \in \dsb{K}$:
    \begin{align*}
		\E_k \left[ 	\widetilde{J}(\vtheta^*) - J(\vtheta_{k+1}) \right] &\leq \E_k \left[ \widetilde{J}(\vtheta^*) - J(\vtheta_{k}) - \zeta_k \left< \widehat{\nabla}_{\vtheta}J(\vtheta_{k}) ,\nabla_{\vtheta}J(\vtheta_{k}) \right> + \frac{\Lhgen}{2} \zeta_k^2 \|  \widehat{\nabla}_{\vtheta}J(\vtheta_{k}) \|_2^2 \right] \\
		&\leq \widetilde{J}(\vtheta^*) - J(\vtheta_{k}) - \zeta_k \| \nabla_{\vtheta} J(\vtheta_{k}) \|_2^2+ \frac{\Lhgen}{2} \zeta_k^2 \E_k \left[\|  \widehat{\nabla}_{\vtheta}J(\vtheta_{k}) \|_2^2\right] \\
		&\leq \widetilde{J}(\vtheta^*) - J(\vtheta_{k}) - \zeta_k \left(1 -  \frac{\Lhgen}{2} \zeta_k \right)  \| \nabla_{\vtheta} J(\vtheta_{k}) \|_2^2 + \frac{\Lhgen}{2} \zeta_k^2 \Var\left[\widehat{\nabla}_{\vtheta}J(\vtheta_{k}) \right] \\
		&\leq \widetilde{J}(\vtheta^*) - J(\vtheta_{k}) - \zeta_k \left( 1 - \frac{\Lhgen}{2} \zeta_k \right)  \| \nabla_{\vtheta} J(\vtheta_{k}) \|_2^2 + \frac{\Lhgen \Vgen}{2N} \zeta_k^2,
	\end{align*}
	where the third inequality follows from the fact that $\widehat{\nabla}_{\vtheta} J(\vtheta)$ is an unbiased estimator and from the definition of $\Var[\widehat{\nabla} J(\vtheta)]$, and the last inequality is by Assumption~\ref{ass:J_var_gen_conv}.
	Now, selecting a step size $\zeta_k \leq \frac{1}{\Lhgen}$, we have that $1-\frac{\Lhgen}{2}\zeta_k \ge \frac{1}{2}$, we can use the bound derived in Equation~(\ref{eq:genconv_eq_1}):
	\begin{align*}
			\E_k &\left[ \widetilde{J}(\vtheta^*) - J(\vtheta_{k+1}) \right]  \le \widetilde{J}(\vtheta^*) - J(\vtheta_{k}) - \frac{\mu \zeta_k}{2}\max\left\{0, \widetilde{J}(\vtheta^*) - J(\vtheta_{k}) \right\}^2  + \frac{\Lhgen \Vgen}{2N} \zeta_k^2.
	\end{align*}
	The next step is to consider the total expectation over both the terms of the inequality and observe that
 \begin{align*}
     \E\left[ \max\left\{0, \widetilde{J}(\vtheta^*) - J(\vtheta_{k}) \right\}^2\right] \ge \E\left[ \max\left\{0, \widetilde{J}(\vtheta^*) - J(\vtheta_{k}) \right\}\right]^2 \ge  \max\left\{0, \E\left[ \widetilde{J}(\vtheta^*) - J(\vtheta_{k}) \right] \right\}^2,
 \end{align*}
 having applied Jensen's inequality twice, being both the square and the $\max$ convex functions. In particular, we define $r_k \coloneqq \E \left[J(\vtheta^*) - J(\vtheta_{k})\right]$.
	We can then rewrite the previous inequality as follows:
	\begin{align*}
		r_{k+1} \leq r_k - \frac{\mu \zeta_k}{2}\max\{0,r_k\}^2 + \frac{\Lhgen \Vgen}{2N} \zeta_k^2.
	\end{align*}
To study the recurrence, we define the helper sequence:
\begin{align}
    \begin{cases}
        \rho_{0}= r_0 \\
        \rho_{k+1} =  \rho_k - \frac{\mu \zeta_k}{2}\max\{0,\rho_k\}^2 + \frac{\Lhgen \Vgen}{2N} \zeta_k^2 & \text{if } k \ge 0
    \end{cases}.
\end{align}
We now show that under a suitable condition on the step size $\zeta_k$, the sequence $\rho_k$ upper bounds the sequence $r_k$.

\begin{lemma}\label{lem:aux_1}
    If $\zeta_k \le \frac{1}{\mu\rho_k}$ for every $k \ge 0$, then, $r_k \le \rho_k$ for every $k \ge 0$.
\end{lemma}

\begin{proof}[Proof of Lemma~\ref{lem:aux_1}]
    By induction on $k$. For $k=0$, the statement holds since $\rho_0 = r_0$. Suppose the statement holds for every $j \le k$, we prove that it holds for $k+1$:
    \begin{align}
         \rho_{k+1} & =  \rho_k - \frac{\mu \zeta_k}{2}\max\{0,\rho_k\}^2 + \frac{\Lhgen \Vgen}{2N} \zeta_k^2 \\
         & \ge r_k - \frac{\mu \zeta_k}{2}\max\{0,r_k\}^2 + \frac{\Lhgen \Vgen}{2N} \zeta_k^2 \\
         & \ge r_{k+1}.
    \end{align}
    where the first inequality holds by the inductive hypothesis and by observing that the function $f(x) = x -  \frac{\mu \zeta_k}{2}\max\{0,x\}^2$ is non-decreasing in $x$ when $\zeta_k \le 1/(\mu x)$. Indeed, if $x < 0$, then $f(x) = x$, which is non-decreasing; if $x \ge 0$, we have $f(x) = x -  \frac{\mu \zeta_k}{2}x^2$, that can be proved to be non-decreasing in the interval $[0,1/(\mu\zeta_k)]$ simply by studying the sign of the derivative.
    The requirement $\zeta_k \le 1/(\mu \rho_k)$ ensures that $\rho_k$ falls in the non-decreasing region, and so does $r_k$ by the inductive hypothesis. 
\end{proof}
Thus, from now on, we study the properties of the sequence $\rho_{k}$ and enforce the learning rate to be constant, $\zeta_k \coloneqq \zeta$ for every $k \ge 0$. Let us note that, if $\rho_k$ is convergent, than it converges to the fixed-point $\overline{\rho}$ computed as follows:
\begin{align}
    \overline{\rho} =  \overline{\rho} - \frac{\mu \zeta}{2}\max\{0,\overline{\rho}\}^2 + \frac{\Lhgen \Vgen}{2N} \zeta^2 \implies \overline{\rho}  = \sqrt{\frac{\Lhgen \Vgen \zeta}{\mu N}},
\end{align}
having retained the positive solution of the second-order equation only, since the negative one never attains the maximum $\max\{0,\overline{\rho}\}$.
Let us now study the monotonicity properties of the sequence $\rho_k$. 

\begin{lemma}\label{lem:aux_2}
    The following statements hold:
    \begin{itemize}
        \item If $r_0 > \overline{\rho}$ and $\zeta \le \frac{1}{\mu r_0}$, then for every $k \ge 0$ it holds that: $\overline{\rho} \le \rho_{k+1} \le \rho_k$.
        \item If $r_0 < \overline{\rho}$ and $\zeta \le \frac{1}{\mu \overline{\rho}}$, then for every $k \ge 0$ it holds that: $\overline{\rho} \ge \rho_{k+1} \ge \rho_k$.
    \end{itemize}
\end{lemma}

Before proving the lemma, let us comment on it. We have stated that if we initialize the sequence with $\rho_0=r_0$ above the fixed-point $\overline{\rho}$, the sequence is non-increasing and remains in the interval $[\overline{\rho},r_0]$. Symmetrically, if we initialize $\rho_0=r_0$ (possibly negative) below the fixed-point $\overline{\rho}$, the sequence is non-decreasing and remains in the interval $[r_0,\overline{\rho}]$. These properties hold under specific conditions on the learning rate.

\begin{proof}[Proof of Lemma~\ref{lem:aux_2}]
    We first prove the first statement, by induction on $k$. The inductive hypothesis is ``$\rho_{k+1}\le\rho_k$ \emph{and} $\rho_{k+1}\ge\overline{\rho}$ ". For $k=0$, for the first inequality, we have:
    \begin{align}
        \rho_1 = \rho_0 - \frac{\zeta \mu}{2} \rho_0^2 + \frac{\Lhgen \Vgen}{2N} \zeta^2 \le \rho_0 - \frac{\zeta \mu}{2} \overline{\rho}^2 + \frac{\Lhgen \Vgen}{2N} \zeta^2 =\rho_0,
    \end{align}
    having exploited the fact that $\rho_0 > \overline{\rho}>0$ and the definition of $\overline{\rho}$. For the second inequality, we have:
    \begin{align}
        \rho_1 = \rho_0 - \frac{\zeta \mu}{2} \rho_0^2 + \frac{\Lhgen \Vgen}{2N} \zeta^2 \ge \overline{\rho} - \frac{\zeta \mu}{2}  \overline{\rho}^2 + \frac{\Lhgen \Vgen}{2N} \zeta^2  = \overline{\rho},
    \end{align}
    recalling that the function $ x - \frac{\zeta \mu}{2} x^2$ is non-decreasing in $x$ for $x\le\rho_0$ since $\zeta \le 1/(\mu \rho_0)$, and by definition of $\overline{\rho}$. Suppose now that the statement holds for every $j < k$. First of all, we observe that, under this inductive hypothesis, $\rho_k \le \rho_0$ and, consequently, the condition $\zeta \le 1/(\mu\rho_0)$ entails $\zeta \le 1/(\mu\rho_k)$. Thus, for the first inequality, we have:
    \begin{align}
         \rho_{k+1} = \rho_k - \frac{\zeta \mu}{2} \rho_k^2 + \frac{\Lhgen \Vgen}{2N} \zeta^2 \le \rho_k - \frac{\zeta \mu}{2} \overline{\rho}^2 + \frac{\Lhgen \Vgen}{2N} \zeta^2 =\rho_k,
    \end{align}
    having used the inductive hypothesis and the definition of $\overline{\rho}$. For the second inequality, we have:
    \begin{align}
        \rho_{k+1} = \rho_k - \frac{\zeta \mu}{2} \rho_k^2 + \frac{\Lhgen \Vgen}{2N} \zeta^2 \ge \overline{\rho} - \frac{\zeta \mu}{2}  \overline{\rho}^2 + \frac{\Lhgen \Vgen}{2N} \zeta^2  = \overline{\rho},
    \end{align}
    having used the inductive hypothesis and recalled that the function $ x - \frac{\zeta \mu}{2} x^2$ is non-decreasing in $x$ for $x\le\rho_k$ since $\zeta \le 1/(\mu\rho_k)$.

    For the second statement, we observe that if $\rho_0=r_0 < 0$, we have:
    \begin{align}
        \rho_k = \rho_0 + k \frac{\Lhgen \Vgen}{2N} \zeta^2,
    \end{align}
    for all $k \le k^*$, where $k^*$ is the minimum $k$ in which $\rho_0 + k \frac{\Lhgen \Vgen}{2N} \zeta^2 \ge 0$. From that point on,  we can proceed in an analogous way as for the first statement, simply switching the signs of the inequalities and recalling that the largest value of $\rho_k$ is bounded by $\overline{\rho}$ in this case.
\end{proof}

We now focus on the first case of the previous lemma in which $r_0 > \overline{\rho}$, as the second one, as we shall see later, is irrelevant for the convergence rate. We now want to show that the sequence $\rho_k$ actually converges to $\overline{\rho}$ and characterize its convergence rate. To this end, we introduce a new auxiliary sequence:
\begin{align}
     \begin{cases}
        \eta_{0}= \rho_0 \\
        \eta_{k+1} =   \left(1- \frac{\mu \zeta\overline{\rho}}{2}\right) \eta_k + \frac{\Lhgen \Vgen}{2N} \zeta^2 & \text{if } k \ge 0
    \end{cases}.
\end{align}

We show that the sequence $\eta_{k}$ upper bounds $\rho_k$ when $\rho_0 = r_0 \ge \overline{\rho}$.
\begin{lemma}
    If $r_0> \overline{\rho}$ and $\zeta \le \frac{1}{\mu r_0}$, then, for every $k \ge 0$, it holds that $\eta_k \ge \rho_k$.
\end{lemma}

\begin{proof}
    By induction on $k$. For $k=0$, we have $\eta_0=\rho_0$, so, the statement holds. Suppose the statement holds for every $j \le k$, we prove it for $k+1$:
    \begin{align}
        \eta_{k+1} & = \left(1- \frac{\mu \zeta\overline{\rho}}{2}\right) \eta_k + \frac{\Lhgen \Vgen}{2N} \zeta^2 \\
        & \ge  \left(1- \frac{\mu \zeta{\rho_k}}{2}\right) \eta_k + \frac{\Lhgen \Vgen}{2N} \zeta^2 \\ 
        & \ge \left(1- \frac{\mu \zeta{\rho_k}}{2}\right) \rho_k + \frac{\Lhgen \Vgen}{2N} \zeta^2 \\
        & = \rho_k - \frac{\zeta \mu}{2} \max\{0,\rho_k\}^2 + \frac{\Lhgen \Vgen}{2N} \zeta^2 = \rho_{k+1}.
    \end{align}
    having exploited that $\rho_k \ge \overline{\rho}$ (by Lemma~\ref{lem:aux_2}) in the second line; using the inductive hypothesis in the third line, exploiting the fact that $1- \frac{\mu\zeta\rho_k}{2} \ge 0$ whenever $\zeta \le 2/(\mu\rho_k)$, which is entailed by the requirement $\zeta \le 1/(\mu\rho_0)$; and by recalling that $\rho_k > 0$ since $\overline{\rho}>0$ in the last line.
\end{proof}

Thus, we conclude by studying the convergence rate of the sequence $\eta_{k}$. This can be easily obtained by unrolling the recursion:
\begin{align}
    \eta_{k+1} & =  \left(1- \frac{\mu \zeta\overline{\rho}}{2}\right)^{k+1} \eta_0 +  \frac{\Lhgen \Vgen\zeta^2}{2N} \sum_{j=0}^k  \left(1- \frac{\mu \zeta\overline{\rho}}{2}\right)^{j} \\
    & \le \left(1- \frac{\mu \zeta\overline{\rho}}{2}\right)^{k+1} \eta_0+\frac{\Lhgen \Vgen\zeta^2}{2N} \sum_{j=0}^{+\infty}  \left(1- \frac{\mu \zeta\overline{\rho}}{2}\right)^{j} \\
    & = \left(1- \frac{\mu \zeta\overline{\rho}}{2}\right)^{k+1} \eta_0+\frac{\Lhgen \Vgen\zeta}{N\mu\overline{\rho}}\\
    & = \left(1- \frac{1}{2} \sqrt{\frac{\mu \zeta^3 \Lhgen \Vgen}{N}}\right)^{k+1} \eta_0 + \sqrt{\frac{\Lhgen \Vgen \zeta}{\mu N}}.
\end{align}
  Putting all the conditions on the step size $\zeta$ together, we must set:
    \begin{align}
        \zeta = \min \left\{ \frac{1}{\Lhgen}, \frac{1}{\mu \max\{0,r_0\}}, \left(\frac{N}{\Lhgen\Vgen\mu}\right)^{1/3} \right\}.
    \end{align}
  Thus, we have:
    \begin{align}
        J(\vtheta^*) - \E[J(\vtheta_K)] \le \betagen +\left(1- \frac{1}{2} \sqrt{\frac{\mu \zeta^3 \Lhgen \Vgen}{N}}\right)^{K} \max\left\{0,J(\vtheta^*) - J(\vtheta_0) - \betagen \right\} + \sqrt{\frac{\Lhgen \Vgen \zeta}{\mu N}} .
    \end{align}
  We derive the number of iterations (setting $K \leftarrow k+1$):
  \begin{align}
      \left(1- \frac{1}{2} \sqrt{\frac{\mu \zeta^3 \Lhgen \Vgen}{N}}\right)^{K} \eta_0 \le \frac{\epsilon}{2} \implies K \le  \frac{\log\frac{2\eta_0}{\epsilon}}{\log \frac{1}{1- \frac{1}{2} \sqrt{\frac{\mu \zeta^3 \Lhgen \Vgen}{N}}}} \le   \sqrt{\frac{4N }{\mu \zeta^3 \Lhgen \Vgen}}\log\frac{2\eta_0}{\epsilon},\label{eq:K_bound}
  \end{align}
  having exploited the inequality $\log \frac{1}{1-x} \ge x$.  Furthermore, let us observe that:
    \begin{align}
        \overline{\rho} = \sqrt{\frac{\Lhgen \Vgen \zeta}{\mu N}} \le \frac{\epsilon}{2} \implies \zeta \le \frac{\epsilon^2\mu N}{4 \Lhgen \Vgen}.
    \end{align}
    Thus, recalling that $\rho_0=\eta_0=r_0$, we have that: (i) when $r_0 < \overline{\rho}$, we have that $r_k \le \rho_k \le \overline{\rho} \le \epsilon/2$; (ii) when $r_0 \ge \overline{\rho}$, we have $r_k \le \rho_k \le \eta_k \le \epsilon$. 
    Thus, for sufficiently small $\epsilon$, we plug $\zeta = \frac{\epsilon^2\mu N}{4 \Lhgen \Vgen}$ in Equation~\eqref{eq:K_bound} to obtain the following upper bound on the sample complexity:
    \begin{align}
        KN \le \frac{16 \Lhgen \Vgen}{\epsilon^3 \mu^2 } \log \frac{\max\left\{0,J(\vtheta^*) - J(\vtheta_0) - \betagen \right\}}{\epsilon},
    \end{align}
    which guarantees $J(\vtheta^*) - \E[J(\vtheta_K)] \le \betagen + \epsilon$.
\end{proof}

\globalConvGeneral*
\begin{proof}
    Directly follows from the second statement of Theorem~\ref{thr:gen_conv_new}.
\end{proof}

\clearpage
\section{Specifications of the Algorithms} \label{apx:algo}

\textbf{PGPE.}~~In this section we report the algorithm PGPE as it is reported in its original paper~\cite{SEHNKE2010551}.
In particular, we show the pseudo-code (Algorithm~\ref{alg:pgpe}) of its original basic version, that is also the one we analyzed throughout this work, even if several variants are available.

\RestyleAlgo{ruled}
\begin{algorithm}[h!]
    \caption{PGPE.}\label{alg:pgpe}
    \SetKwInOut{Input}{Input}
    \small
    \Input{Number of iterations $K$, batch size $N$, initial parameter vector $\vtheta_{0}$, environment $\mathcal{M}$, deterministic policy $\mu_{\vtheta}$, hyperpolicy $\nu_{\vtheta}$, step size schedule $(\zeta_k)_{k=0}^{K-1}$, exploration parameter $\psigma$.}

    Initialize $\vtheta \xleftarrow[]{} \vtheta_{0}$
    
    \For{$i \in \dsb{K}$}{
        Set the hyperpolicy parameters: $\nu_{\vtheta}$
        
        \For{$l \in \dsb{N}$}{
            Sample a parameter configuration $\vrho_{l} \sim \nu_{\vtheta}$ according to the exploration parameter $\psigma$

            Collect a trajectory $\tau_{l}$ by acting in $\mathcal{M}$ with $\mu_{\vrho_{l}}$

            Compute the cumulative discounted reward $R(\tau_l)$
        }
        Compute the gradient estimator: $\widehat{\nabla}_{\vtheta} \Jp(\vtheta) \xleftarrow[]{} \frac{1}{N} \sum_{j=0}^{N-1} \nabla_{\vtheta} \log \nu_{\vtheta}(\vrho_j) R(\tau_j)$

        Update the hyperpolicy parameter vector: $\vtheta \xleftarrow[]{} \vtheta + \zeta_i \widehat{\nabla}_{\vtheta} \Jp(\vtheta)$
    }
    Return $\vtheta$.
\end{algorithm}

Notice that, the original version of PGPE by \citet{SEHNKE2010551} considers to collect $M$ trajectories for each parameter configuration $\vrho$ sampled from the hyperpolicy $\nu_{\vtheta}$.
In the pseudo-code (as well as in the paper) we consider $M=1$ (\ie we collect a single trajectory) in order to make GPOMDP and PGPE testing the same number of trajectories in each iteration, given an equal batch size $N$. 
In the original paper also other variants of PGPE are considered, that we have not considered in or work.
For instance, the one with symmetric sampling, or the one employing a baseline while sampling.
Moreover, it would be possible to learn a proper exploration amount $\psigma$ while learning the hyperpolicy parameters, however we decided to keep $\psigma$ fixed, for reasons remarked in Appendix~\ref{apx:mapping}.

\textbf{GPOMDP.}~~As done for PGPE, here we report the algorithm GPOMDP in its original version \citep{baxter2001infinite, peters2006policy}.
We show the pseudo-code (Algorithm~\ref{alg:pg}) of such original basic version, that is also the one we analyzed throughout this work.


\RestyleAlgo{ruled}
\begin{algorithm}[h!]
    \caption{GPOMDP.}\label{alg:pg}
    \SetKwInOut{Input}{Input}
    \small
    \Input{Number of iterations $K$, batch size $N$, initial parameter vector $\vtheta_{0}$, environment $\mathcal{M}$, stochastic policy $\pi_{\vtheta}$ (with exploration parameter $\asigma$), step size schedule $(\zeta_k)_{k=0}^{K-1}$, , horizon $T$, discount factor $\gamma$.}

    Initialize $\vtheta \xleftarrow[]{} \vtheta_{0}$
    
    \For{$i \in \dsb{K}$}{
        Set the stochastic policy parameters: $\pi_{\vtheta}$
        
        \For{$l \in \dsb{N}$}{
            Initialize trajectory $\tau_l$ as an empty tuple
            
            \For{$t \in \dsb{T}$}{
                Observe state $s_t$
                
                Play action $\va_t \sim \pi_{\vtheta}(\cdot | s_t)$

                Observe reward $r_t$

                Add to $\tau_l$ the tuple $(s_t, a_t, r_t)$
            }
        }
        Compute the gradient estimator: $\widehat{\nabla}_{\vtheta} \Ja(\vtheta) \xleftarrow[]{} \frac{1}{N} \sum_{i=1}^{N} \sum_{t=0}^{T-1} \left( \sum_{k=0}^{t} \nabla_{\vtheta} \log \pi_{\vtheta} (\va_{\tau_i, k} | \vs_{\tau_i, k}) \right) \gamma^t r(\vs_{\tau_i, t}, \va_{\tau_i, t})$

        Update the policy parameter vector: $\vtheta \xleftarrow[]{} \vtheta + \zeta_i \widehat{\nabla}_{\vtheta} \Ja(\vtheta)$
    }
    Return $\vtheta$.
\end{algorithm}

In the original paper, it is available a variant of GPOMDP which employs baselines while sampling, but in our work we do not consider this approach, as for PGPE.
Also in this case, we decided to employ a fixed value for $\asigma$, even if it would be possible to adapt it at runtime (Appendix~\ref{apx:mapping}).

\clearpage
\section{Additional Experimental Results} \label{apx:experiments}
In this section, we present additional experimental results for what concerns the comparison of GPOMDP and PGPE, and the sensitivity analysis on the exploration parameters, respectively $\asigma$ and $\psigma$.

\subsection{Learning Curves of the Variance Study of Section~\ref{sec:experiments}.} \label{apx:experiments:curves}

\textbf{Setting.}~~We show the results gained by learning in three environments of increasing complexity taken from the MuJoCo~\citep{todorov2012mujoco} suite: \emph{Swimmer-v4}, \emph{Hopper-v4}, and \emph{HalfCheetah-v4}. 
Details on the environmental parameters are shown in Table~\ref{tab:exp_setting}.
In order to facilitate the exploration, thus highlighting the results of the sensitivity study on the exploration parameters, we added an action clipping to the environments.\footnote{When the policy draws an action the environment performs a clip of the action \emph{before} the reward is computed.}
The target \emph{deterministic policy} $\bm{\mu}_{\vtheta}$ is linear in the state, while the \emph{hyperpolicy} $\nu_{\vtheta}$ employed by PGPE is Gaussian with a parameterized mean, and the \emph{stochastic policy} $\pi_{\vtheta}$ employed by GPOMDP is Gaussian with a mean linear in the state.
Both PGPE and GPOMDP were run for $K=2000$ iterations, generating $N=100$ trajectories per iteration.
We conducted a sensitivity analysis on the exploration parameters, using $\{ 0.01, 0.1, 1, 10, 100 \}$ as values for $\psigma^2$ and $\asigma^2$. 
We employed Adam~\citep{kingma2014adam} to set the step size with initial values $0.1$ for PGPE and $0.01$ for GPOMDP. 
The latter does not support a larger step size due to the higher variance of the employed estimator \wrt the one used by PGPE.

\begin{table}[h!]
    \centering
    \small
    \renewcommand{\arraystretch}{1.5}
    \begin{tabular}{|c||c|c|c|c|c|}
        \hline
        \rowcolor{gray!20}
         Environment & $T$ & $\gamma$ & $d_{\mathcal{S}}$ & $\da$ & $\dt$ \\
         \hline
         \hline
         Swimmer & $200$ & $1$ & $8$ & $2$ & $16$ \\
         \hline
         Hopper & $100$ & $1$ & $11$ & $3$ & $33$ \\
         \hline
         HalfCheetah & $100$ & $1$ & $17$ & $6$ & $102$ \\
        \hline        
    \end{tabular}
    \caption{Parameters of the environments.}
    \label{tab:exp_setting}
\end{table}

Here we show the learning curves of $\Jp$ and $\Ja$ (and the associated empirical $\Jd$) obtained in the same setting of Section~\ref{sec:experiments}, which is also summarized in Table~\ref{tab:exp_setting}.
In particular, Figures~\ref{fig:pgpe_hc_vs_curves} and~\ref{fig:pg_hc_vs_curves} show the learning curves associated with the \emph{HalfCheetah-v4} environment, Figures~\ref{fig:pgpe_hop_vs_curves} and~\ref{fig:pg_hop_vs_curves} show the ones for the \emph{Hopper-v4} environment, while Figures~\ref{fig:pgpe_swimmer_vs_curves} and~\ref{fig:pg_swimmer_vs_curves} show the ones for the \emph{Swimmer-v4} environment.
In all the environments, it is possible to notice that, for increasing values of the exploration parameters $\psigma$ and $\asigma$, the learning curves $\Jp$ and $\Ja$ (optimized respectively by PGPE and GPOMDP) differ increasingly with the associated empirical deterministic one $\Jd$ (reported in right-hand side column in the plots).
This is due to the fact that small values of $\psigma$ and $\asigma$ lead to a lower amount of exploration. 
Poorly exploratory $\nu_{\vtheta}$ and $\pi_{\vtheta}$ make the algorithms test actions that are very similar to the ones that target deterministic policy $\mu_{\vtheta}$ would suggest.
Conversely, large values of $\psigma$ and $\asigma$ lead to a higher amount of exploration, thus $\Jp$ and $\Ja$ tend to show a higher offset \wrt to the associated empirical $\Jd.$

\textbf{HalfCheetah.}~~In Figures~\ref{fig:pgpe_hc_vs_curves} and~\ref{fig:pg_hc_vs_curves}, it is possible to see the learning curves of $\Jp$ and $\Ja$ (and the associated empirical $\Jd$) seen by PGPE and GPOMDP while learning on \emph{HalfCheetah-v4}.
Note that, in this case, the optimal value for $\psigma^2$ is $1$, while the one for $\asigma^2$ is $10$.
With $T=100$, PGPE seems to struggle a bit more in finding a good deterministic policy \wrt GPOMDP.
This can be explained by the fact that the parameter dimensionality $\dt$ is the highest throughout the three presented environments. 

\textbf{Hopper.}~~In Figures~\ref{fig:pgpe_hop_vs_curves} and~\ref{fig:pg_hop_vs_curves}, it is possible to see the learning curves of $\Jp$ and $\Ja$ (and the associated empirical $\Jd$) seen by PGPE and GPOMDP while learning on \emph{Swimmer-v4}.
Also in this case, the optimal value for $\psigma^2$ is $1$, while the one for $\asigma^2$ is $10$.
As for \emph{HalfCheetah}, with $T=100$, PGPE seems to struggle a bit more in finding a good deterministic policy \wrt GPOMDP, even if this is the intermediate difficulty environment for what concerns the parameter dimensionality $\dt$.

\textbf{Swimmer.}~~In Figures~\ref{fig:pgpe_swimmer_vs_curves} and~\ref{fig:pg_swimmer_vs_curves}, it is possible to see the learning curves of $\Jp$ and $\Ja$ (and the associated empirical $\Jd$) seen by PGPE and GPOMDP while learning on \emph{Swimmer-v4}.
Note that, in this case, the optimal value for $\psigma^2$ is $10$, while the one for $\asigma^2$ is $1$.
Here we employed an horizon $T=200$.
Indeed, as also commented in Section~\ref{sec:experiments}, GPOMDP struggles more than PGPE in finding a good deterministic policy.

\begin{figure}[H]
    \centering
    \resizebox{\linewidth}{!}{\includegraphics{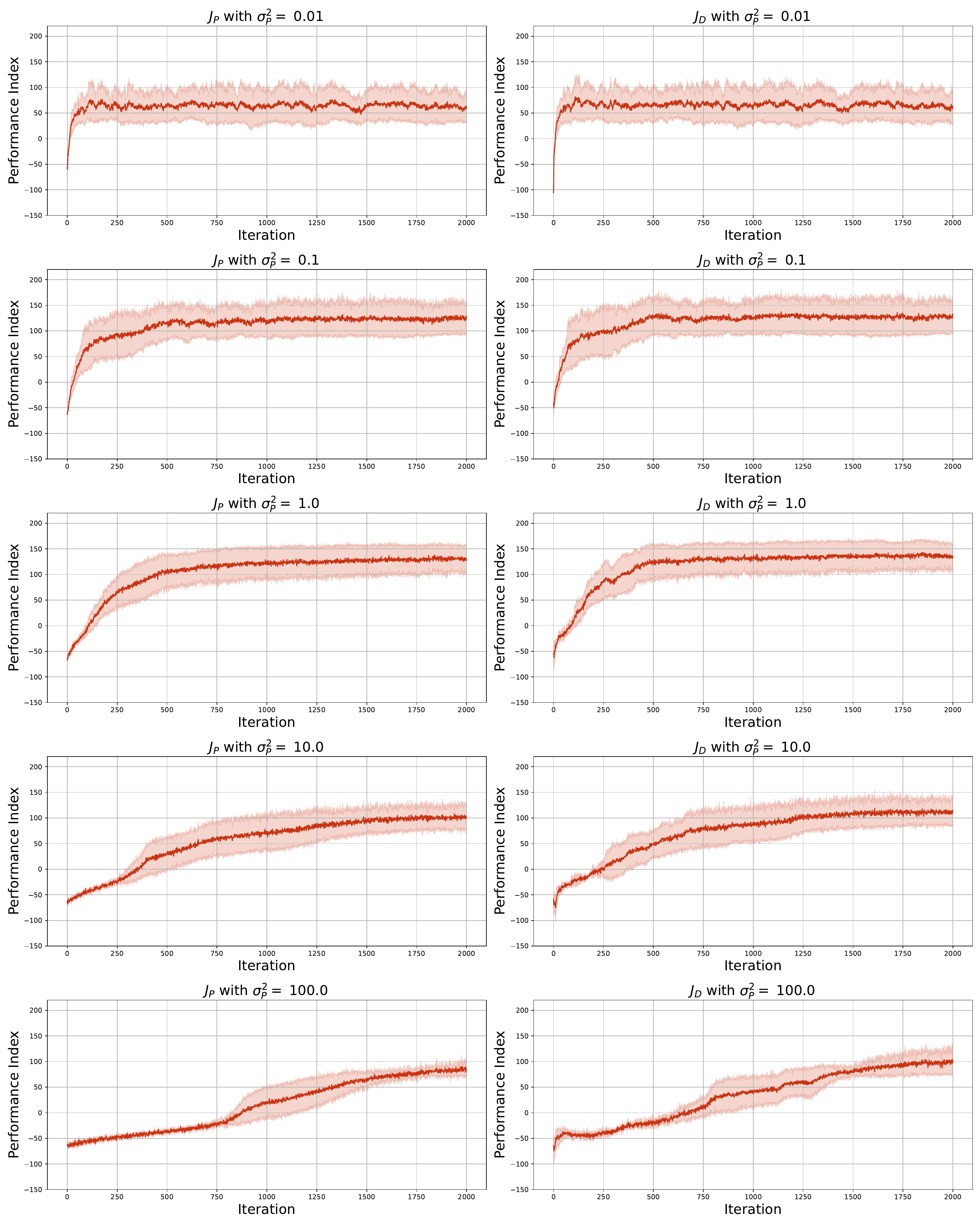}}
    \caption{$\Jp$ and $\Jd$ learning curves (5 runs, mean $\pm 95\%$ C.I.) for PGPE on \emph{Half Cheetah-v4}.}
    \label{fig:pgpe_hc_vs_curves}
\end{figure}

\begin{figure}[H]
    \centering
    \resizebox{\linewidth}{!}{\includegraphics{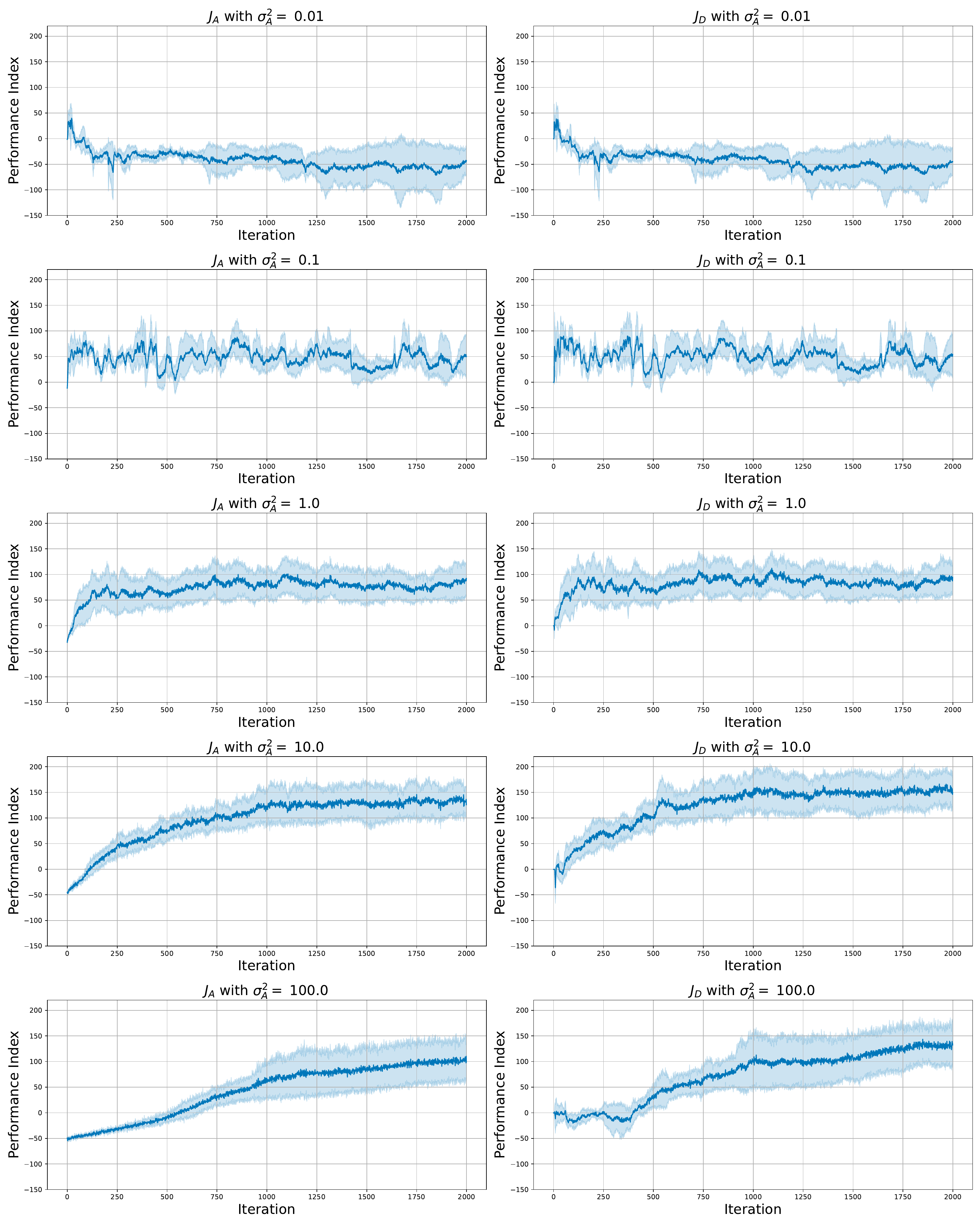}}
    \caption{$\Ja$ and $\Jd$ learning curves (5 runs, mean $\pm 95\%$ C.I.) for GPOMDP on \emph{Half Cheetah-v4}.}
    \label{fig:pg_hc_vs_curves}
\end{figure}

\begin{figure}[H]
    \centering
    \resizebox{\linewidth}{!}{\includegraphics{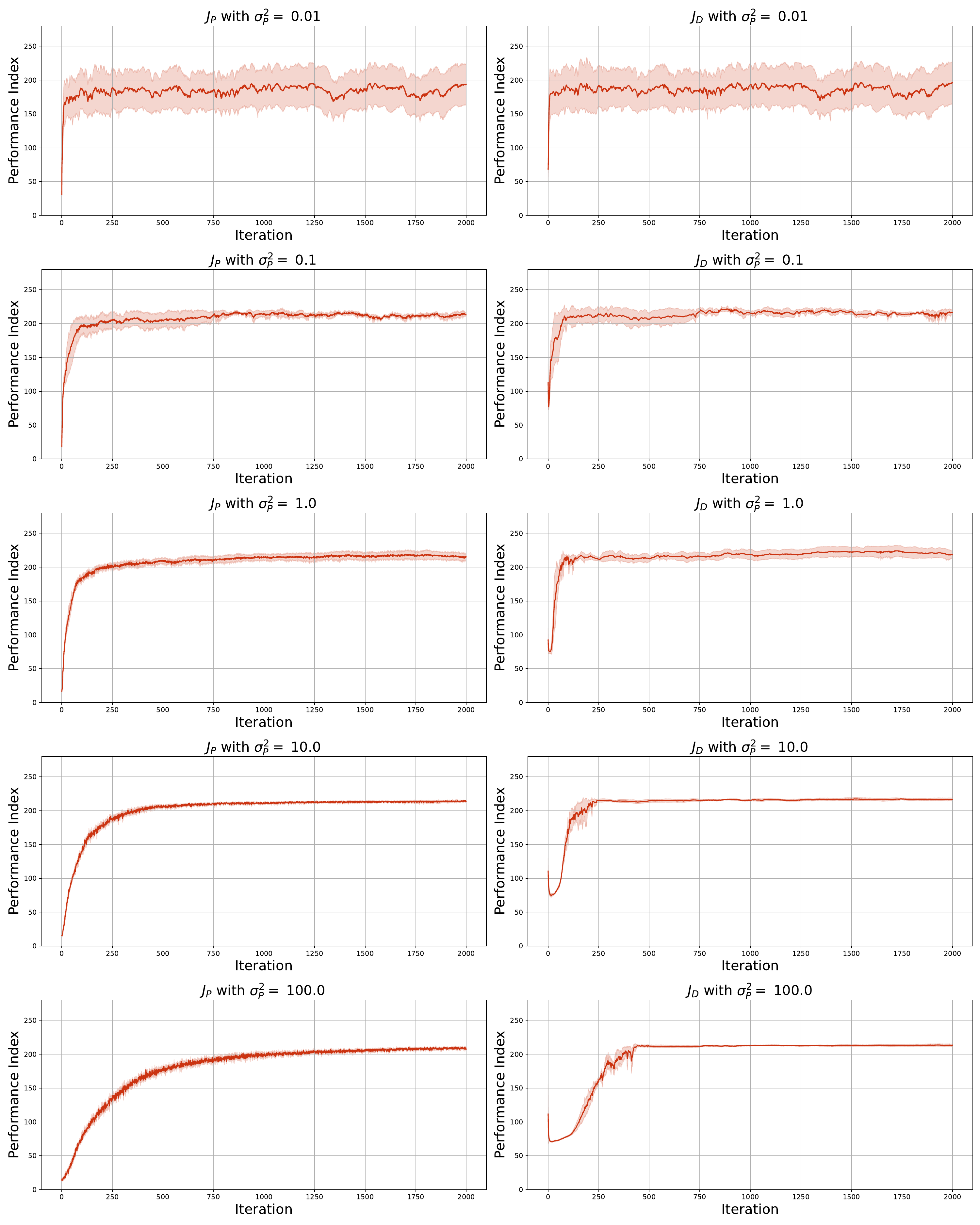}}
    \caption{$\Jp$ and $\Jd$ learning curves (5 runs, mean $\pm 95\%$ C.I.) for PGPE on \emph{Hopper-v4}.}
    \label{fig:pgpe_hop_vs_curves}
\end{figure}

\begin{figure}[H]
    \centering
    \resizebox{\linewidth}{!}{\includegraphics{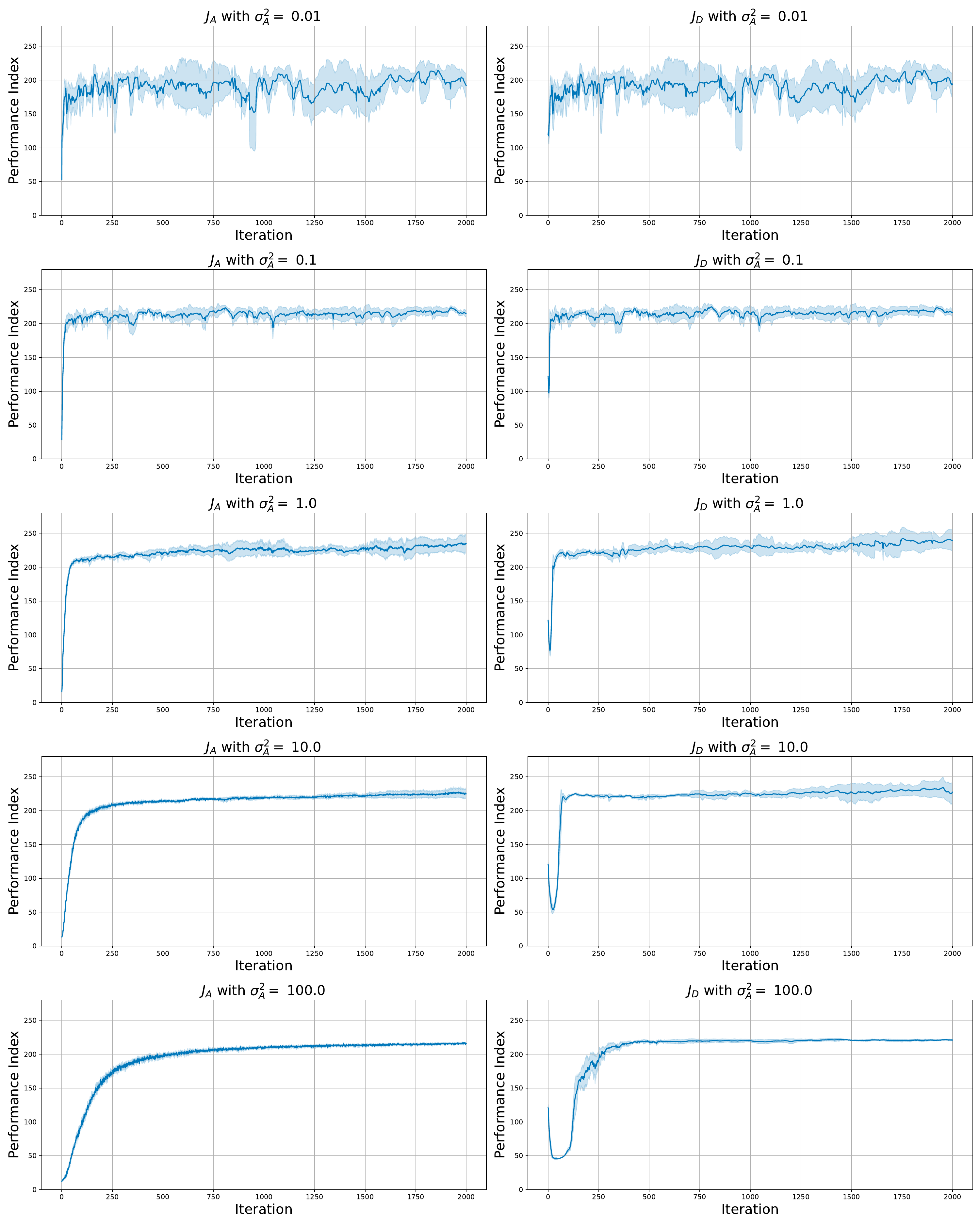}}
    \caption{$\Ja$ and $\Jd$ learning curves (5 runs, mean $\pm 95\%$ C.I.) for GPOMDP on \emph{Hopper-v4}.}
    \label{fig:pg_hop_vs_curves}
\end{figure}

\begin{figure}[H]
    \centering
    \resizebox{\linewidth}{!}{\includegraphics{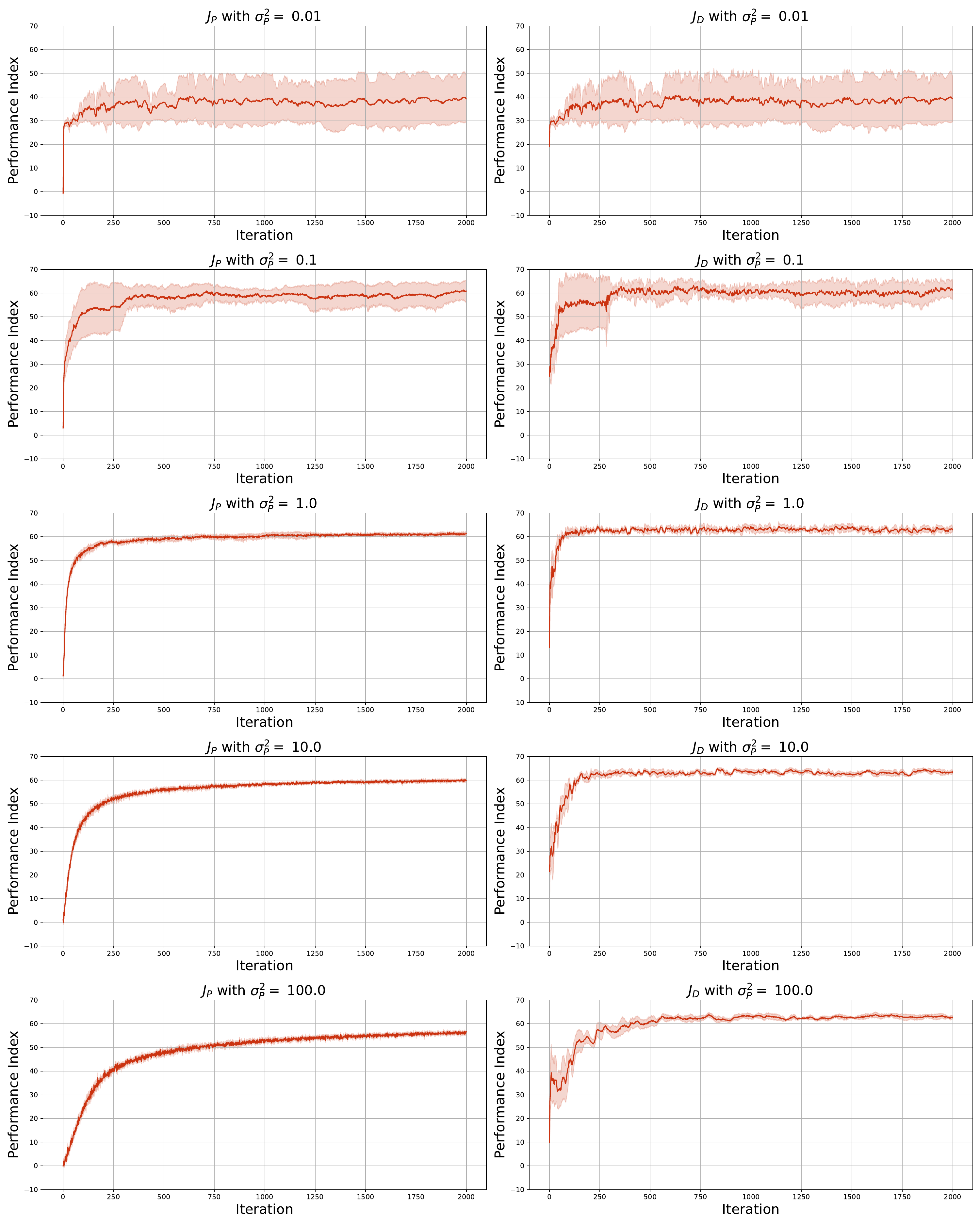}}
    \caption{$\Jp$ and $\Jd$ learning curves (5 runs, mean $\pm 95\%$ C.I.) for PGPE on \emph{Swimmer-v4}.}
    \label{fig:pgpe_swimmer_vs_curves}
\end{figure}

\begin{figure}[H]
    \centering
    \resizebox{\linewidth}{!}{\includegraphics{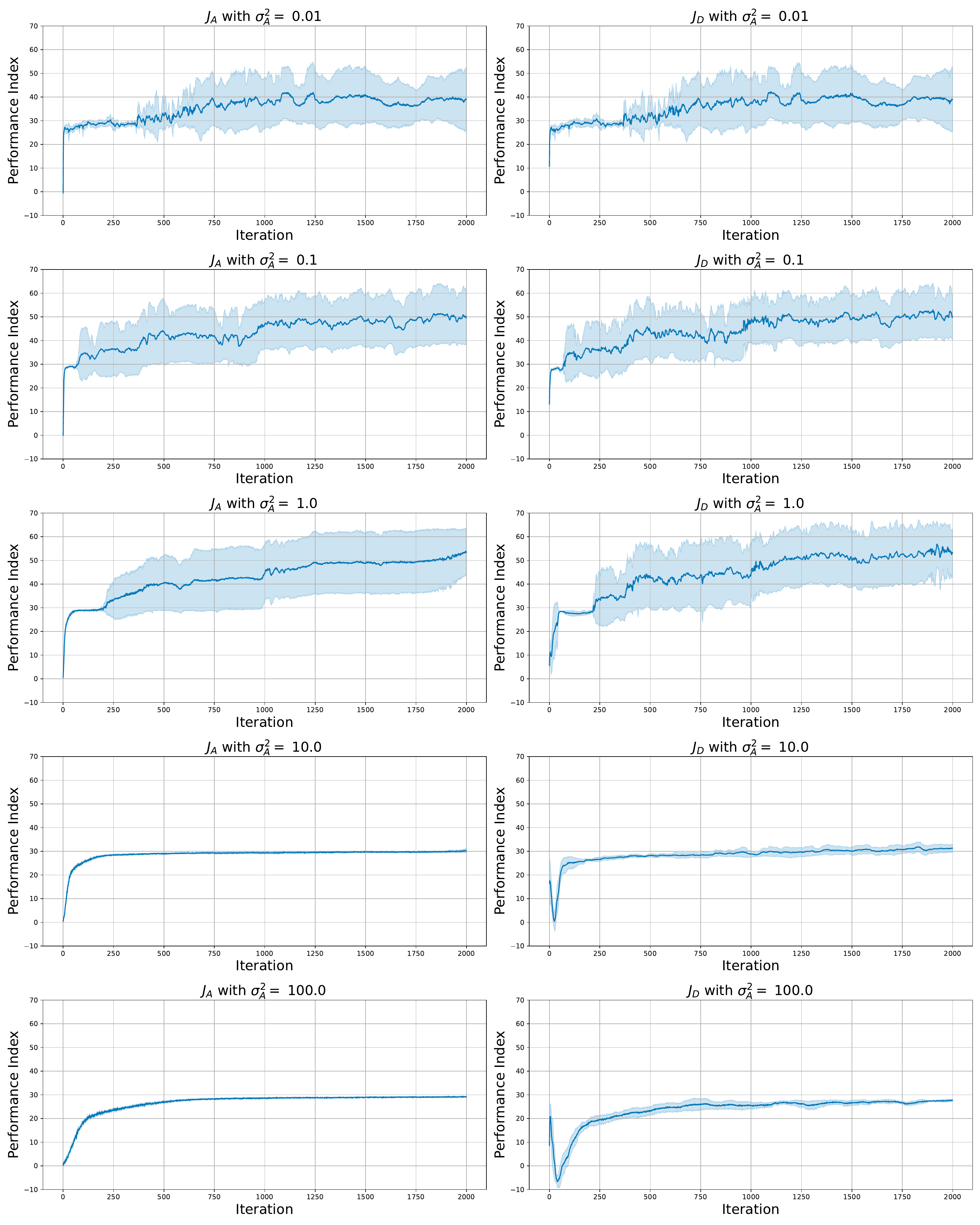}}
    \caption{$\Ja$ and $\Jd$ learning curves (5 runs, mean $\pm 95\%$ C.I.) for GPOMDP on \emph{Swimmer-v4}.}
    \label{fig:pg_swimmer_vs_curves}
\end{figure}

\subsection{GPOMDP vs. PGPE: the case of LQR}
 In order to show more clearly the discussed trade-offs, here we present a numerical validation conduced on the Linear Quadratic Regulator~\citep[LQR,][]{kucera1992optimal} environment, much smaller \wrt the ones offered by the MuJoCo suite.

\paragraph{Brief description of the environment}
Here we summarize the considered version of the LQR environment.
Considering $x_t$ and $u_t$ as the state and action at time $t$, respectively, the state evolution is computed as: $x_{t+1} = A x_t + B u_t$. 
The reward at time $t$ is computed as: $r_{t} = - x_{t}^{\top} Q x_{t} - u_{t}^{\top} R u_{t}$. 
The initial state of the environment is randomly sampled from the interval $[-3, 3]$ using a uniform distribution. 

\paragraph{Setting}
Our objective is to control the LQR environment via a deterministic linear policy.
For the presented results, we considered a number of iterations of $K=3000$ for both PGPE and GPOMDP, with a batch size of $N=100$ trajectories for each iteration, and a learning rate schedule governed by Adam, with initial step sizes of $0.01$. We conducted 3 runs for each experiment, and the plots depict the mean $\pm 95\%$ confidence interval.
Moreover, we considered a bi-dimensional LQR environment (\ie $d_{\mathcal{S}} = 2$ and $d_{\mathcal{A}} = 2$), with unlimited state and action spaces (\ie state and action ranges are $(-\infty, +\infty)$).
Furthermore, the characteristic matrices of the LQR environment were selected as:
\begin{align}
    A = B = \left[\begin{matrix} 0.9 & 0 \\ 0 & 0.9\end{matrix}\right], \qquad Q = \left[\begin{matrix} 0.9 & 0 \\ 0 & 0.1\end{matrix}\right], \qquad \text{and} \qquad  R = \left[\begin{matrix} 0.1 & 0 \\ 0 & 0.9\end{matrix}\right].
\end{align}

\paragraph{Sensitivity \wrt $\sigma_{\dagger}^2$}
Here we present a similar study to the one that has been discussed in the main paper.
We tested both PGPE and GPOMDP on the previously described LQR with $T=50$ and with the exploration amounts varying in $\sigma_{\dagger}^2 \in \{10^{-5}, 10^{-4}, 10^{-3}, 10^{-2}, 10^{-1}\}$.
As can be noticed in Figure~\ref{fig:lqr_sigma}, there are values for the exploration amounts $\psigma$ and $\asigma$ leading to higher performance values for the deployed deterministic policy. 
In particular, PGPE deploys its best version of $\mu_{\vtheta}$ when setting $\psigma^2 = 10^{-3}$, while the same happens with GPOMDP setting $\asigma = 10^{-4}$.

\begin{figure}[t!]
    \centering
    \resizebox{\linewidth}{!}{\includegraphics{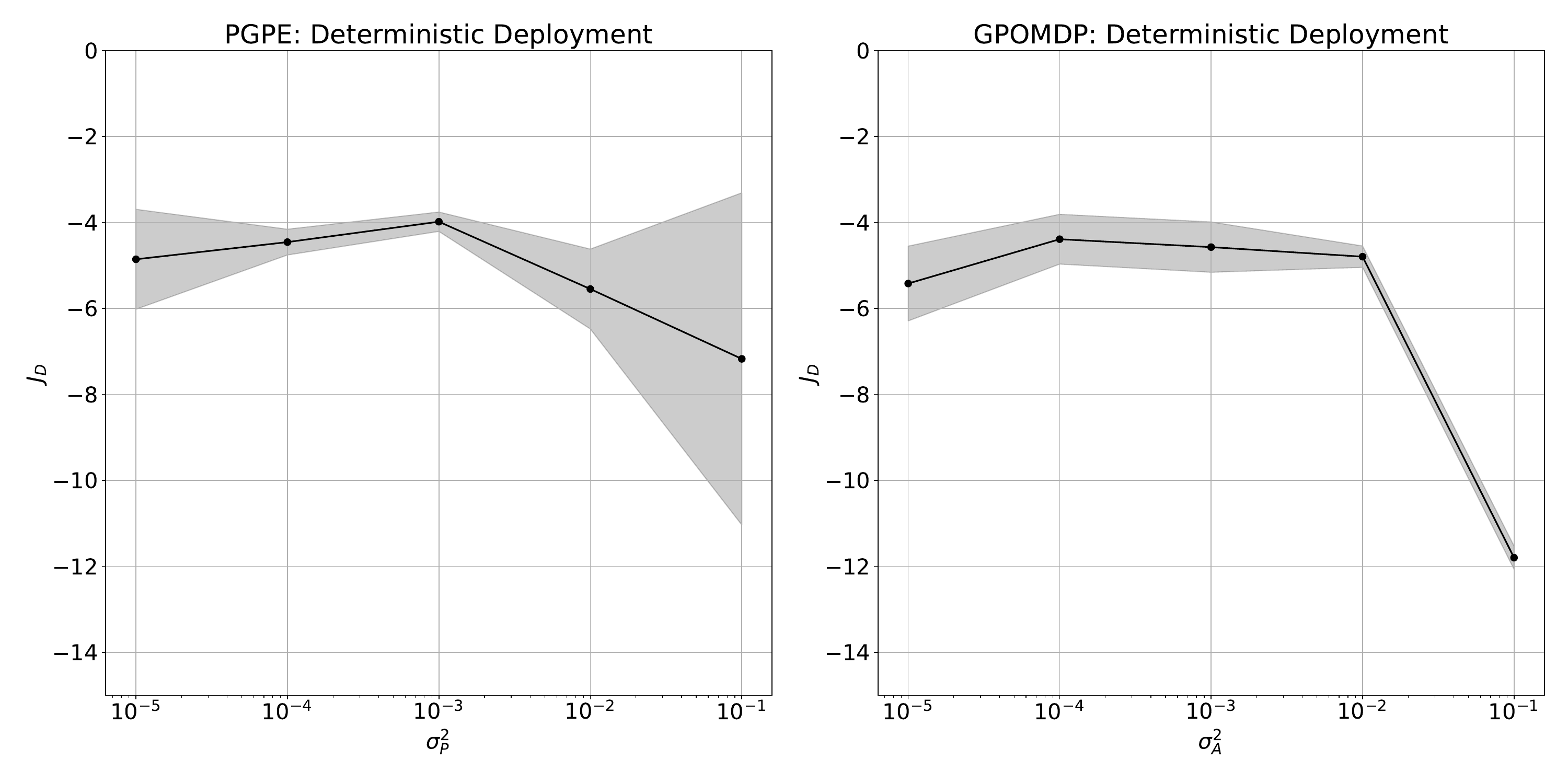}}
    \caption{$\Jd(\vtheta_K)$ associated with GPOMDP and PGPE in LQR ($T=50$) employing linear (hyper)policies and varying exploration amounts (5 runs, mean $\pm 95\%$ C.I.).}
    \label{fig:lqr_sigma}
\end{figure}

\paragraph{Increasing $T$}
Here we present a study on the horizon length $T$, for which we tested values $T \in \{ 50, 100, 200 \}$. We display in Figure~\ref{fig:lqr_h} the resulting $\Jd$ associated with the learning processes of PGPE and GPOMDP. For each of the algorithms, we employed the best values of $\sigma$ obtained from the previous experiment.
Additionally, we depicted as a dashed line the estimated performance of the optimal policy for LQR, showing that both algorithms manage to achieve performance close to optimal on average.
As can be observed, GPOMDP struggles more than PGPE in converging to the globally optimal deterministic policy.
In particular, the performance of the deterministic policy associated with PGPE appears not to change with increasing values of $T$.

\begin{figure}[t!]
    \centering
    \resizebox{0.85\linewidth}{!}{\includegraphics{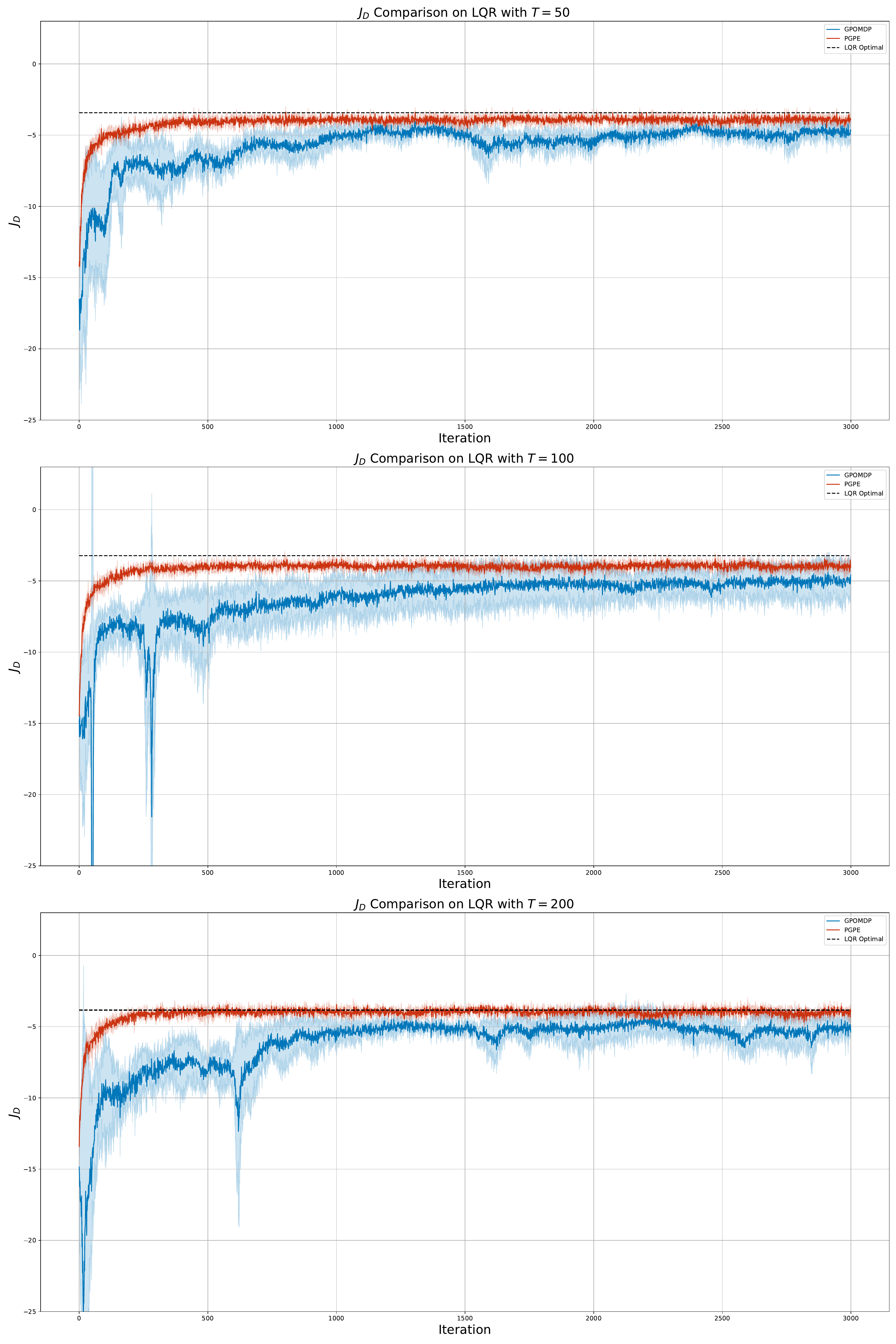}}
    \caption{$\Jd$ associated with GPOMDP and PGPE in LQR employing linear (hyper)policies and $T\in\{50,100,200\}$ (5 runs, mean $\pm 95\%$ C.I.).}
    \label{fig:lqr_h}
\end{figure}

\subsection{GPOMDP vs. PGPE: the case of \emph{Swimmer-v4}.} \label{apx:experiments:tradeoff}
In this section, we conduce experiments to highlight the trade-off between parameter dimensionality $\dt$ and trajectory length $T$.
As emerges from the theoretical results shown in the main paper, GPOMDP should struggle in finding a good deterministic policy $\mu_{\vtheta}$ from large values of $T$, while PGPE should struggle in the same task form large values of $\dt$.
Notice that this behavior was already visible in the variance study conducted in Section~\ref{sec:experiments} and Appendix~\ref{apx:experiments:curves}, where we added action clipping to environments.
To better illustrate the trade-off at hand, we removed the action clipping to conduce the following experimental results, restoring the original version of the MuJoCo environments. 
Indeed, we remark that action clipping was introduced to facilitate the exploration, highlighting the outcomes of the variance study.

\textbf{Setting.}~~We consider two different target deterministic policies $\mu_{\vtheta}$: 
\begin{itemize}
    \item \emph{linear}: PGPE and GPOMDP are run for $K=2000$, with $N=100$, $\dt = 16$ (parameters initialized to $0$);
    \item \emph{neural network} (two dense 
    hidden layers with $32$ neurons and with hyperbolic tangent activation functions): PGPE and GPOMDP are run for $K=2000$, with $N=100$, $\dt = 1344$ (parameters initial values sampled by $\mathcal{N}(0, 1)$).
\end{itemize}
For the learning rate schedule, we employed Adam with the same step sizes $0.1$ for PGPE and $0.01$ for GPOMDP (the reason is the same explained in Section~\ref{sec:experiments}).
For all the experiments we fixed both $\psigma$ and $\asigma$ to $1$.

\textbf{Increasing $T$.}~~Here we show the results of learning on \emph{Swimmer-v4} with $T \in \{100, 200\}$ (and $\gamma = 1$).
The target deterministic policy in this case is the linear one, thus $\dt = 16$.
Figures~\ref{fig:swimmer_linear_100} and~\ref{fig:swimmer_linear_200} show the learning curves of $\Jp$ and $\Ja$, with their associated empirical $\Jd$.
For $T=100$, PGPE and GPOMDP reach deterministic policies exhibiting similar values of $\Jd(\vtheta_{K})$.
For $T=200$, instead, the algorithms reach deterministic policies showing an offset in the values of $\Jd(\vtheta_{K})$ in favor of PGPE.
As suggested by the theoretical results shown in the paper, the fact that GPOMDP struggles in reaching a good deterministic policy can be explained by the doubling of the horizon value.

\textbf{Increasing $\dt$.}~~Here we show the results of learning on \emph{Swimmer-v4} with $T \in \{100, 200\}$ (and $\gamma = 1$), with two different target deterministic policies: the linear one ($\dt = 16$) and the neural network one ($\dt=1344$).

Figures~\ref{fig:pgpe_swimmer_nn_100} and ~\ref{fig:pgpe_swimmer_nn_200} show the learning curves of $\Jp$, with their associated empirical $\Jd$, for both the target policies, when learning with trajectories respectively of length $100$ and $200$.
For both the values of the horizon, it is possible to notice that with a smaller value of $\dt$ PGPE manages to find a better deterministic policy.
Indeed, the found linear and neural network deterministic policies show an offset in $\Jd(\vtheta_{K})$ in favor of the linear one.
As suggested by the theoretical results shown in the paper, the fact that PGPE struggles in reaching a good deterministic policy can be explained by the heavily increased parameter dimensionality $\dt$.

Figures~\ref{fig:pg_swimmer_nn_100} and ~\ref{fig:pg_swimmer_nn_200} show the learning curves of $\Ja$, with their associated empirical $\Jd$, for both the target policies, when learning with trajectories respectively of length $100$ and $200$.
From Figure~\ref{fig:pg_swimmer_nn_100}, even with the target neural network policy, for $T=100$ GPOMDP is able however to find a deterministic policy with similar performances to the one found when the target deterministic policy is the linear one.
Switching to $T=200$ (Figure~\ref{fig:pg_swimmer_nn_200}), it is possible to notice a severe offset between the learning curves of the empirical $\Jd$ associated to $\Ja$, in favor of the case in which the target policy is the linear one.
As done for the analysis on the increasing $T$, this can be explained by the fact that the horizon has been doubled, which is in line with the theoretical results shown throughout this work.

\begin{figure}[t!]
    \centering
    \subfloat[PGPE.]{
        \resizebox{\linewidth}{!}{\includegraphics{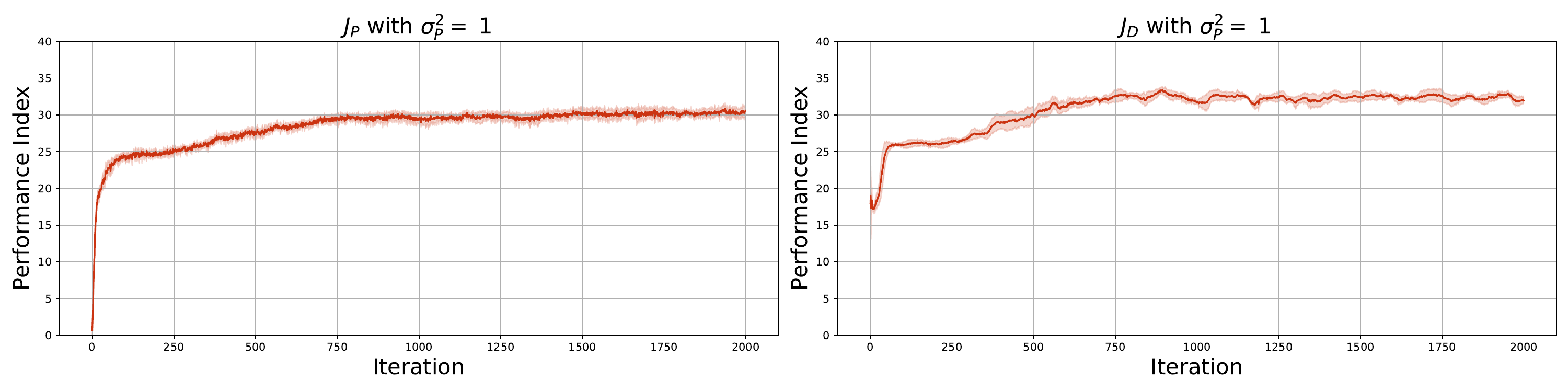}}
    }
    \hfill
    \subfloat[GPOMDP.]{
        \resizebox{\linewidth}{!}{\includegraphics{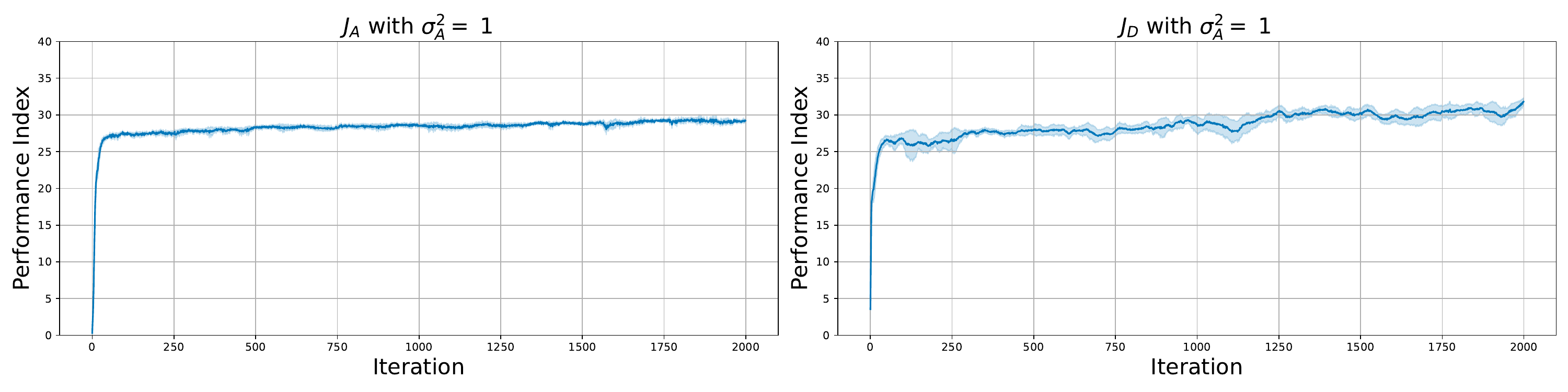}}
    }
    \caption{PGPE and GPOMDP on \emph{Swimmer-v4} with linear policy and $T=100$ (5 runs, mean $\pm 95\%$ C.I.).}
    \label{fig:swimmer_linear_100}
\end{figure}

\begin{figure}[t!]
    \centering
    \subfloat[PGPE.]{
        \resizebox{\linewidth}{!}{\includegraphics{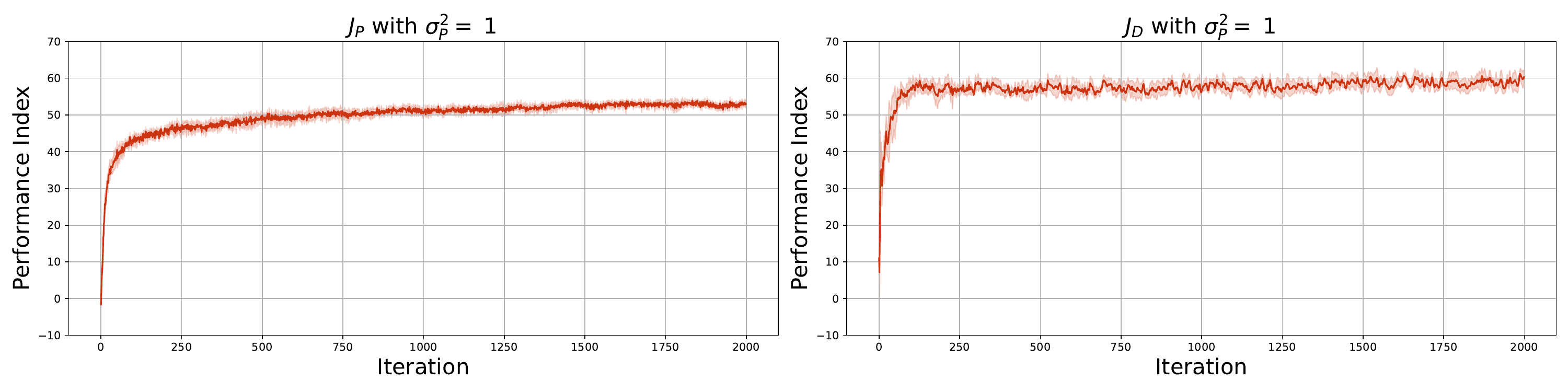}}
    }
    \hfill
    \subfloat[GPOMDP.]{
        \resizebox{\linewidth}{!}{\includegraphics{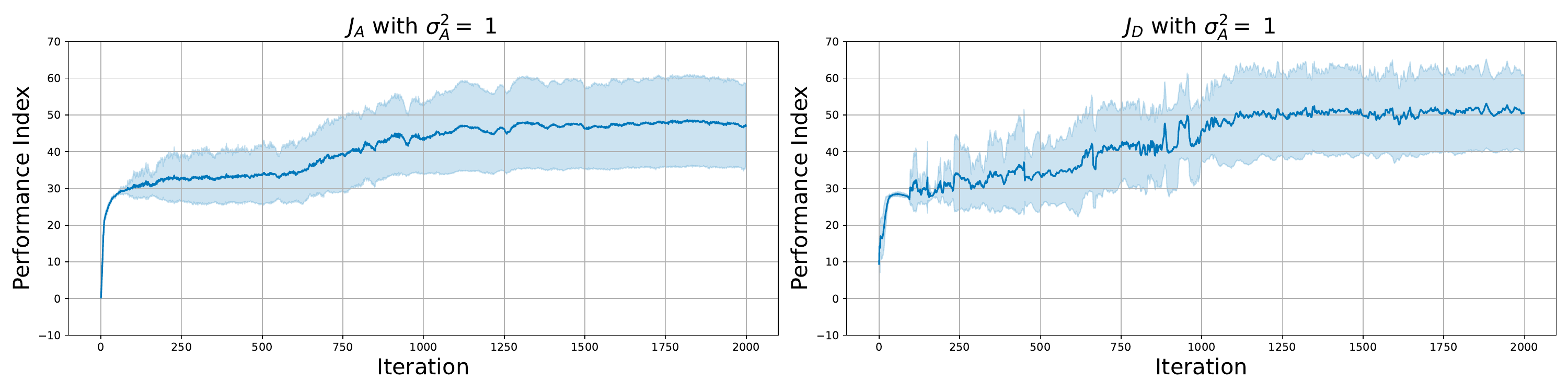}}
    }
    \caption{PGPE and GPOMDP on \emph{Swimmer-v4} with linear policy and $T=200$ (5 runs, mean $\pm 95\%$ C.I.).}
    \label{fig:swimmer_linear_200}
\end{figure}

\begin{figure}[t!]
    \centering
    \subfloat[Linear.]{
        \resizebox{\linewidth}{!}{\includegraphics{img/pgpe_swimmer_linear_curves_100.pdf}}
    }
    \hfill
    \subfloat[Neural Network.]{
        \resizebox{\linewidth}{!}{\includegraphics{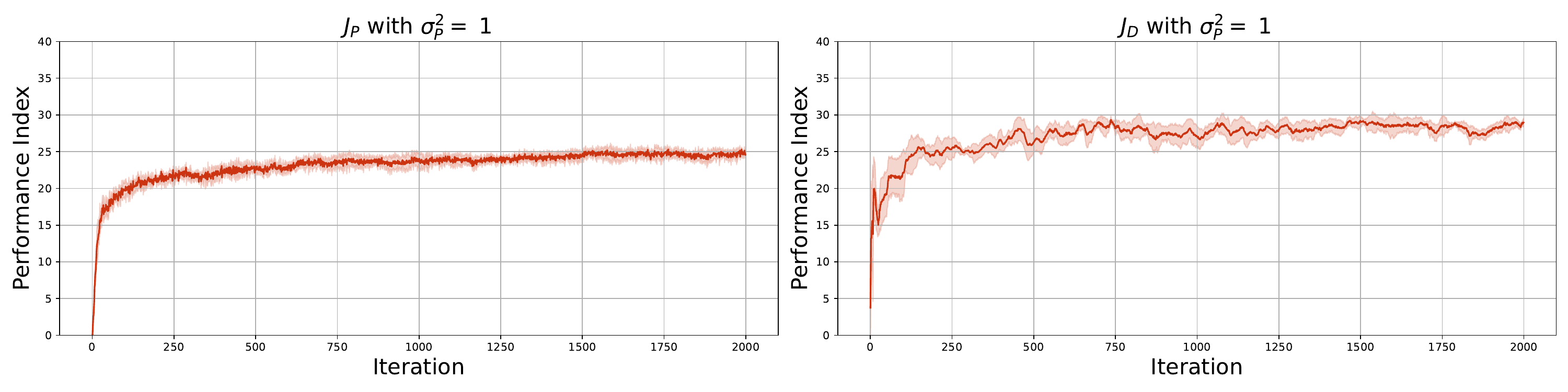}}
    }
    \caption{PGPE on \emph{Swimmer-v4} with linear and neural network policies, and $T=100$ (5 runs, mean $\pm 95\%$ C.I.).}
    \label{fig:pgpe_swimmer_nn_100}
\end{figure}

\begin{figure}[t!]
    \centering
    \subfloat[Linear.]{
        \resizebox{\linewidth}{!}{\includegraphics{img/pg_swimmer_linear_curves_100.pdf}}
    }
    \hfill
    \subfloat[Neural Network.]{
        \resizebox{\linewidth}{!}{\includegraphics{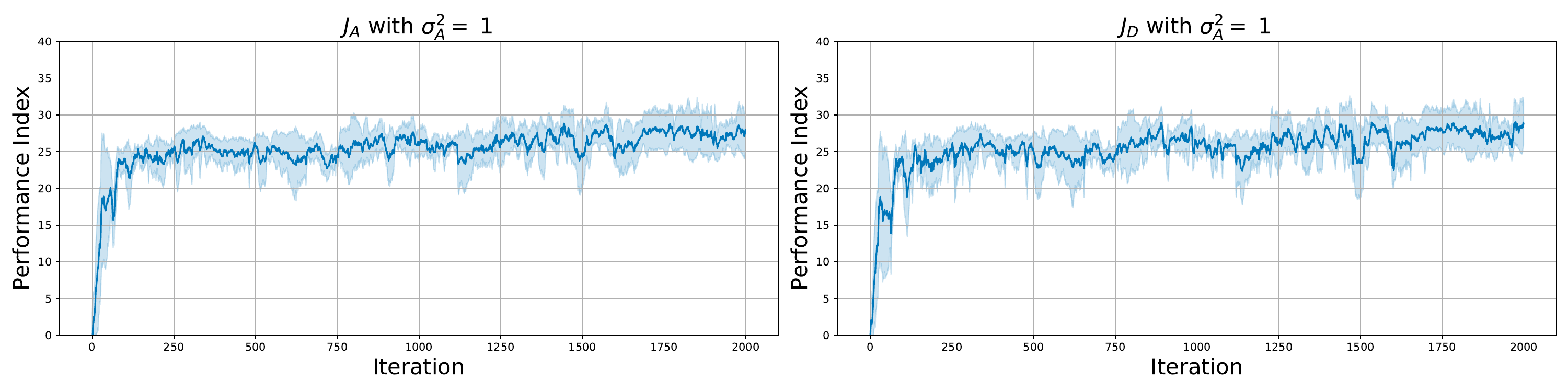}}
    }
    \caption{GPOMDP on \emph{Swimmer-v4} with linear and neural network policies, and $T=100$ (5 runs, mean $\pm 95\%$ C.I.).}
    \label{fig:pg_swimmer_nn_100}
\end{figure}

\begin{figure}[t!]
    \centering
    \subfloat[Linear.]{
        \resizebox{\linewidth}{!}{\includegraphics{img/pgpe_swimmer_linear_curves_200.pdf}}
    }
    \hfill
    \subfloat[Neural Network.]{
        \resizebox{\linewidth}{!}{\includegraphics{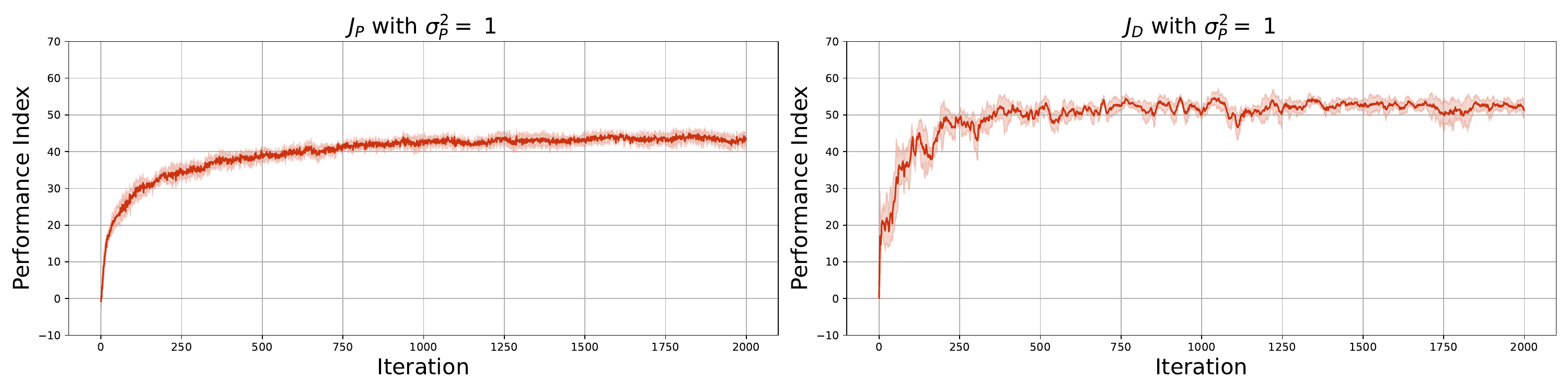}}
    }
    \caption{PGPE on \emph{Swimmer-v4} with linear and neural network policies, and $T=200$ (5 runs, mean $\pm 95\%$ C.I.).}
    \label{fig:pgpe_swimmer_nn_200}
\end{figure}

\begin{figure}[t!]
    \centering
    \subfloat[Linear.]{
        \resizebox{\linewidth}{!}{\includegraphics{img/pg_swimmer_linear_curves_200.pdf}}
    }
    \hfill
    \subfloat[Neural Network.]{
        \resizebox{\linewidth}{!}{\includegraphics{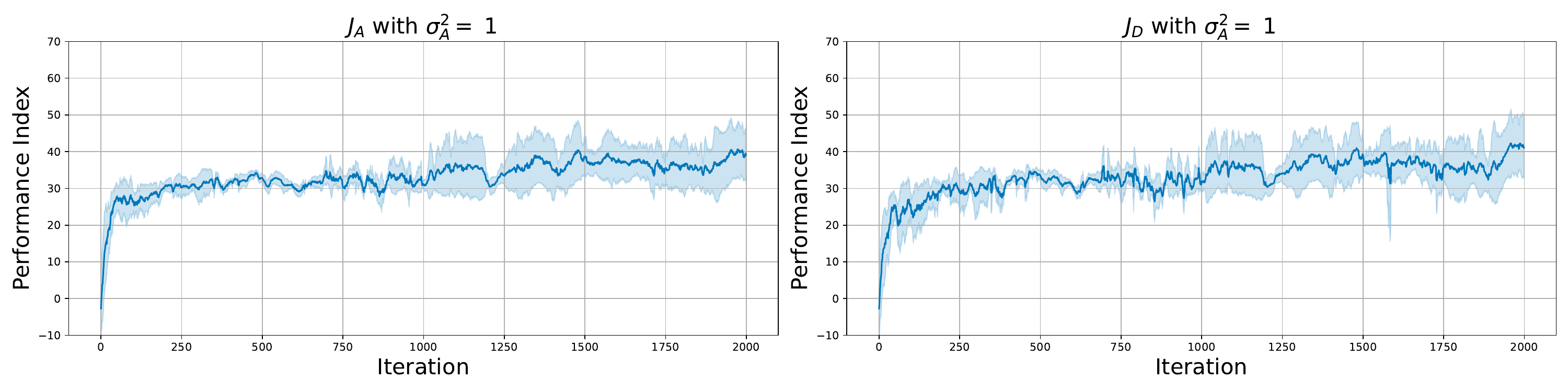}}
    }
    \caption{GPOMDP on \emph{Swimmer-v4} with linear and neural network policies, and $T=200$ (5 runs, mean $\pm 95\%$ C.I.).}
    \label{fig:pg_swimmer_nn_200}
\end{figure}

\end{document}